\documentclass[]{article} 

\usepackage[numbers]{natbib}

\usepackage{graphicx}
\usepackage{amsthm, amssymb, amsmath}
\usepackage{mathtools}
\usepackage{color}
\usepackage{tikz}
\usetikzlibrary{arrows}
\usepackage{url}
\usepackage{bbold}%
\usepackage{cases}
\usepackage{wrapfig}
\usepackage{comment}
\usepackage{caption}

\numberwithin{equation}{section}
\usepackage{hyperref}
\hypersetup{
  colorlinks = true,
  urlcolor = blue, 
  linkcolor = blue,
  citecolor = blue}

\definecolor{ggreen}{rgb}{0.0, 0.5, 0.3}
\definecolor{rred}{rgb}{0.65, 0.2, 0.2}
\definecolor{bblue}{rgb}{0.0, 0.0, 1}

\newtheorem{theorem}{Theorem}
\newtheorem{lemma}{Lemma}

\newtheorem{corollary}{Corollary}

\newcommand{\cS}{\mathcal{S}} 
 
\newcommand{\Sp}{\mathrm{Sp}} 		
\renewcommand{\P}{P}	

\newcommand{\E}{\mathbb{E}}

\newcommand{\sfgrad}{\mathsf{grad}}

\newcommand{\bx}{\boldsymbol{x}}

\newcommand{\bu}{\boldsymbol{u}}

\newcommand{\bbeta}{\boldsymbol{\beta}}
\newcommand{\balpha}{\boldsymbol{\alpha}}

\newcommand{\bQ}{\boldsymbol{Q}}
\newcommand{\bP}{\boldsymbol{P}}
\newcommand{\bB}{\boldsymbol{B}}
\newcommand{\bI}{\boldsymbol{I}}
\newcommand{\bV}{\boldsymbol{V}}
\newcommand{\bA}{\boldsymbol{A}}
\newcommand{\bW}{\boldsymbol{W}}
\newcommand{\bY}{\boldsymbol{Y}}

\newcommand{\bG}{\boldsymbol{G}}
\newcommand{\bzero}{\boldsymbol{0}}
\newcommand{\bDelta}{\boldsymbol{\Delta}}
\newcommand{\bZ}{\boldsymbol{Z}}
\newcommand{\bSigma}{\boldsymbol{\Sigma}}

\newcommand{\cX}{\mathcal{X}}
\newcommand{\cP}{\mathcal{P}}
\newcommand{\cF}{\mathcal{F}}
\newcommand{\cH}{\mathcal{H}}

\newcommand{\cB}{\mathcal{B}}
\newcommand{\cA}{\mathcal{A}}

\usepackage{bm}

\newcommand{\bmP}{\bm{P}}

\newcommand{\N}{\mathbb N}

\newcommand{\R}{\mathbb R}

\newcommand{\argmin}{\mathop{\mathrm{argmin}}}

\renewcommand{\phi}{\varphi}

\DeclareMathOperator{\Tr}{Tr}

\DeclareMathOperator{\Exp}{\mathsf{Exp}}

\newcommand{\bX}{\boldsymbol X}
\newcommand{\Xin}{\mathcal X_{\mathrm{in}}}
\newcommand{\Xout}{\mathcal X_{\mathrm{out}}}
\newcommand{\bXin}{\boldsymbol{X}_{\mathrm{in}}}
\newcommand{\bXout}{\boldsymbol{X}_{\mathrm{out}}}
\newcommand{\cSREAP}{\mathcal S_{\mathsf{REAP}}}
\newcommand{\cAREAP}{\mathcal A_{\mathsf{REAP}}}
\newcommand{\cPREAP}{\mathcal P_{\mathsf{REAP}}}

\usepackage[algo2e,ruled,,linesnumbered,lined,boxed,commentsnumbered,ruled,vlined]{algorithm2e}

\allowdisplaybreaks

\title{Stochastic and Private Nonconvex Outlier-Robust PCA}

\author{%
 {Tyler Maunu} \and
 {Chenyu Yu} \and {Gilad Lerman}
}

\date{}

\begin{document}

\maketitle

\begin{abstract}
    We develop theoretically guaranteed stochastic methods for outlier-robust PCA. Outlier-robust PCA seeks an underlying low-dimensional linear subspace from a dataset that is corrupted with outliers.  We are able to show that our methods, which involve stochastic geodesic gradient descent over the Grassmannian manifold, converge and recover an underlying subspace in various regimes through the development of a novel convergence analysis. The main application of this method is an effective differentially private algorithm for outlier-robust PCA that uses a Gaussian noise mechanism within the stochastic gradient method. Our results emphasize the advantages of the nonconvex methods over another convex approach to solving this problem in the differentially private setting. Experiments on synthetic and stylized data verify these results.
\end{abstract}

\section{Introduction}

\emph{Outlier-robust PCA} (ORPCA) involves the problem of robustly estimating an underlying linear subspace from data in the presence of large amounts of corrupted data. While many solutions have been proposed for this problem, some particularly effective methods involve nonconvex energy minimization~\citep{maunu2019well}. However, these methods require generic conditions on the full dataset and it is not clear how they behave in the presence of \emph{stochastic gradients}, since they typically require good initialization and control over where the iterates lie.

This work develops a deeper understanding of how nonconvex methods for OR-PCA interact with stochastic gradients. 
Past studies have mainly looked at recovery limits of such methods~\citep{lerman2018overview}, both in terms of percentages of corrupted data as well as their associated statistics. In the current work, we show that it is possible to extend the results to the stochastic setting while maintaining robustness guarantees. 

As an important application, we show that specific choices of stochastic gradients lead to \emph{differential privacy}. Private algorithms provide an important way to gain insight from sensitive data. As a framework, differential privacy has harkened in a new era in the study of privacy and its interaction with data science and machine learning~\citep{dwork2006calibrating,dwork2008differential}.
To make an algorithm differentially private, one typically incorporates some sort of noise mechanism. This noise mechanism is applied to either the data itself or within the algorithm to limit the influence any single point can have on the output. In this paper, we focus on differentially private gradient descent algorithms, which use noisy gradients at each iteration to achieve differential privacy. 
 
While differential privacy may be simple to include within an algorithm, it is less straightforward to guarantee how accurate the algorithm will be. Recently, there has been work on empirical risk minimization by differentially private methods, which show that it is possible to achieve fast estimation and optimization rates with differentially private algorithms~\citep{bassily2014private,talwar2014private,bassily2019private}. Such results typically focus on the convex setting, but some recent work has studied such algorithms in the nonconvex setting as well~\citep{wang2019differentially}. While these results are quite general, they do not capture the intricacies of the analysis of robust methods. That is, especially in the setting of ORPCA, robust methods are concerned with \emph{recovery results}, where under various conditions on a corrupted dataset, an algorithm can still recover some underlying structure. Especially in nonconvex recovery problems, it is not clear how the stochastic nature of the private algorithms interacts with existing recovery guarantees. Due to our generic guarantees for stochastic gradient methods, we are able to guarantee recovery for a differentially private method.



\subsection{Background}

Suppose that we observe a dataset $\cX = \{\bx_1, \dots, \bx_N\}\subset \R^D$.  The classical problem of principal component analysis (PCA) seeks the $r$ directions of maximum variance within this dataset, where $r$ is a parameter chosen by the user. Equivalently, one can also try to find a linear subspace that spans these directions.
It is therefore convenient to encode PCA as a problem over the Grassmannian manifold of $r$-dimensional linear subspaces in $\R^D$, $G(D, r)$. Throughout the paper, we also consider the optimization over orthogonal bases for $L \in G(D,r)$: each element of $G(D,r)$ can be spanned by the columns of a semiorthogonal matrix in $O(D,r):=\{\bV \in \mathbb{R}^{D \times r} : \bV^\top \bV = \bI_r\}$.

In this language, PCA solves the geometric problem
\begin{equation}\tag{\sf{PCA}}\label{eq:PCA}
    \min_{\bV \in O(D,r)} \frac{1}{N} \sum_{i=1}^N \| (\bI - \bV \bV^\top) \bx_i\|^2,
\end{equation}
where $ \bV \bV^\top$ is the orthogonal projection matrix onto $\Sp(\bV)$. PCA thus finds the subspace which minimizes the sum of squared distances between points and the subspace. 

PCA is not outlier-robust due to the use of squared error. A typical way to robustify 
it 
is to remove the square, which results in the following formulation which we refer to as Grassmannian Least Absolute Deviations (GLAD). :
\begin{align}\tag{\sf GLAD}\label{eq:GLAD}
    &\min_{\bV \in O(D,r)} \frac{1}{N} \sum_{i=1}^N \| (\bI - \bV \bV^\top) \bx_i\| =: F(\bV;\cX).
\end{align}
Many methods have been proposed to solve this nonconvex and nonsmooth problem and they are overviewed in \citet{lerman2018overview}. 

To give a high-level overview of our results, we will briefly discuss the two areas that it straddles. First, the primary result of our analysis is guarantees for a nonconvex, stochastic method for ORPCA. Typically, ORPCA algorithms assume an inlier-outlier model, $\mathcal X = \Xin \cup \Xout$, where $\Xin$ lie on a low-dimensional subspace $L_\star$, and the outliers $\Xout$ are corrupted to not lie on this subspace. The goal is to recover $L_\star$, or $\bV_\star \in O(D,r)$ such that $\Sp(\bV_\star) = L_\star$. For simplicity, we assume that the data is centered, so that we search for a linear subspace. 
Throughout the paper, we also make the simplifying assumption that $\cX \subset S^{D-1}$, where $S^{D-1}$ is the sphere in $\R^D$, so that the function $F$ is 1-Lipschitz. This can be achieved by first normalizing all points to the sphere, which has robustifying characteristics to adversarial outliers~\citep{maunu2019robust}.

Second, the important application of our results involves differential privacy \citep{Dwork2013privacybook}. A randomized algorithm $\mathcal{A}$, which takes in an input $x$ and gives back a random output, is $(\epsilon,\delta)$-differentially private if, for all $\mathcal{S} \subseteq \text{Range}(\mathcal{A})$ and for all datasets $x,y$ that only differ in at most one data point, 
$
    P[\mathcal{A}(x) \in \mathcal{S}] \leq e^{\epsilon} P[\mathcal{A}(y) \in \mathcal{S}] + \delta.
$    
Two common ways to make a first-order algorithm private include adding noise to data or adding noise to the gradients. In this work, since we study stochastic gradient methods for outlier-robust PCA, the recovery guarantees we prove naturally extend to the private setting.

\subsection{Contributions}

We derive the following results for ORPCA 
with large $N$:
\begin{enumerate}
    \item We present stochastic versions of the geodesic gradient descent (GGD) algorithm, which results in the Noisy GGD  (NGGD), Stochastic GGD (SGGD), and Noisy Stochastic GGD (NSGGD) methods. We give theorems guaranteeing linear convergence and subspace recovery by these three methods. Our results are the first nonconvex convergence guarantees for stochastic gradient descent in the least absolute deviations framework. 
    \item  With specifically chosen noise parameters, we demonstrate that these methods are differentially private, and we refer to the resulting algorithms as dp-GGD and dp-SGGD, respectively. We compare these private algorithms to convex methods for differentially private outlier-robust PCA based on the REAPER problem, (dp-REAP). 
    In this setting, we extend past results on differentially private convex empirical risk minimization to give subspace recovery guarantees for the dp-REAP algorithms under generic conditions.
    \item By comparing our theoretical results for the differentially private methods, we demonstrate a distinct advantage in the differentially private setting for dp-(S)GGD over dp-REAP. The nonconvex dp-(S)GGD algorithm converges at a linear rate while the convex dp-REAP methods converge at a sublinear rate, meaning that one can obtain a much more accurate approximation to the underlying subspace in less iterations. In terms of best approximations while still maintaining privacy, we achieve approximation errors that are $O(N^{-1})$ for the convex methods and errors on the order of $O(2^{-N^{\tau}})$ for the nonconvex methods, where $\tau$ is some constant in $(0, 2)$ that depends on the statistics of the dataset.
    \item Experiments on synthetic and stylized data emphasize the theoretical results of this paper. In particular, they demonstrate the advantage in terms of speed and accuracy for the nonconvex methods, and in particular demonstrate distinct advantages for the dp-SGGD method.
\end{enumerate}

\subsection{Review of Directly Related Work}

For a comprehensive review of the many 
methods used for ORPCA, we direct the reader to~\citet{lerman2018overview}. Perhaps one of the most popular frameworks for ORPCA uses least absolute deviations. Originating with the study of robust orthogonal regression in \citet{osborne_watson85,spath_watson1987}, it was considered for ORPCA in \citet{ding2006r}. More recent studies by \citet{zhang2014novel,lerman2015robust,lerman2018fast,maunu2019well} have demonstrated the considerable advantages of this program. This problem is distinct from what is called Robust PCA (RPCA), which considers sparse corruptions~\citep{chandrasekaran2011rank, candes2011robust}.

The nonconvex method we propose is based on optimization on the Grassmannian manifold~\citep{Edelman98thegeometry}. Manifold optimization has recently been of great interest for the machine learning community~\citep{zhang2016first}. 

Differential privacy has become the preeminent way of protecting sensitive data~\citep{Dwork2013privacybook}. There has been a recent surge of work examining how differential privacy affects the accuracy of various methods~\citep{bassily2014private,bassily2019private}. Some recent work has been devoted to considering differentially private methods for PCA~\citep{chaudhuri2013near, hardt2014noisy, jiang2016wishart}.

\subsection{Notation}


We let $\sigma_j(\cdot)$ denote the $j$th singular value of a matrix.
For measuring subspace approximation, we use a distance metric on the Grassmannian. A typical metric is $d(L_1, L_2) = \sqrt{\sum_{j=1}^r \theta_j^2}$, where $\theta_j$ are the principal angles between $L_1$ and $L_2$. For our later analysis of the nonconvex method, for $\bV, \bV' \in O(D,r)$, which are bases for two elements of $G(D,r)$, it is more convenient to work with the squared metric $d_r^2(\bV, \bV') = 1 - \sigma_r(\bV^\top \bV')$, which for subspaces that are close together is on the order of 1/2 times the largest principal angle squared between $\Sp(\bV)$ and $\Sp(\bV')$ (specifically, it is $1-\cos(\theta_1)$).
We denote $B_{d_r^2}(\bV, \rho)$ to be the ball of radius $\rho$ with respect to $d_r^2$. 

\section{Stochastic Algorithms to Minimize~\ref{eq:GLAD}}

In this paper, we propose to use stochastic gradient descent to directly minimize~\eqref{eq:GLAD}. This extends the existing framework for ORPCA studied by~\citet{maunu2019well}, where the authors proposed to use vanilla geodesic gradient descent (GGD).  Section~\ref{subsec:ggd} reviews the GGD method used to minimize \eqref{eq:GLAD}. Then, Section~\ref{subsec:sggd} discusses modifications of this method to include noisy and minibatch gradients, which result in stochastic GGD methods.

\subsection{Geodesic Gradient Descent}
\label{subsec:ggd}

One can directly optimize \eqref{eq:GLAD} over the Grassmannian manifold using geometric methods. Past algorithms that accomplish this with some theoretical guarantees (despite the nonconvex setting) include IRLS~\citep{lerman2018fast} and GGD~\citep{maunu2019well}. 
On top of frequently being more accurate than their convex counterparts, these methods are also
faster than convex methods, since nonconvex methods work with a $D \times r$ optimization variable rather than the typical $D \times D$ variable.

We briefly review GGD. 
Since $G(D,r)$ forms a Riemannian manifold, the Riemannian gradient of the energy function in~\eqref{eq:GLAD} is 
\begin{equation}\label{eq:grad}
    \nabla F(\bV; \cX) = \frac{1}{|\cX|} \bQ_{\bV} \sum_{\bx \in \cX} \frac{\bx \bx^\top \bV}{\|\bQ_{\bV} \bx\|},
\end{equation}
where $\bV \in O(D,r)$ is a matrix whose columns span $L$,  $\bQ_{\bV} = \bI - \bV \bV^\top$ projects the gradient to the tangent space of $G(D, r)$ and $|\cX|$ denotes the number of points in the set $\cX$. Geodesic gradient descent (GGD) then takes the form $\bV_{k+1} = \Exp_{\bV_k} (- \eta_k \nabla F(\bV_k;\cX))$ (where $\Exp$ is the exponential map). For a complete discussion of this iteration and associated concepts related to the geometry of $G(D,r)$~\citep{Edelman98thegeometry,maunu2019well}.

\subsection{Stochastic Geodesic Gradient Descent Methods}
\label{subsec:sggd}

In terms of optimization, the main innovation in this work is to consider stochastic gradient methods for~\eqref{eq:GLAD}. One specific stochastic gradient one may consider is the addition of Gaussian noise, which enhances privacy. To go beyond this setting, we also consider stochasticity due to \emph{minibatching}. While the addition of stochastic gradients is a small modification of the original GGD method,  
it is entirely nontrivial to extend convergence and recovery analysis to the stochastic setting (see Section \ref{sec:theory}).

We first describe a version of GGD which uses noisy gradients. Let $\bB_k \in \R^{D \times r}$ whose entries are i.i.d.$\sim \mathcal{N}(0, \sigma^2)$. The noisy GGD (NGGD) iteration is given by
\begin{equation}\label{eq:nggd}\tag{\sf{NGGD}}
    {\bV}_{k+1} = \cP_{O(D,r)}({\bV}_k - \eta_k (\nabla F({\bV}_k;\cX) + \bB_k)),
\end{equation}
where $\cP_{O(D, r)}$ is the projection operator that solves $\cP_{O(D, r)} (\bA) = \argmin_{\bV \in O(D, r)} \|\bV - \bA \|^2$. This is an example of the orthogonal Procrustes problem~\citep{gower2004procrustes}, and it can be solved via the SVD or polar decomposition~\citep{fan1955some}. This iteration is referred to as Noisy Geodesic Gradient Descent (NGGD).

We can also use stochastic estimates of $\nabla F({\bV}_k;\cX)$ to add further ``noise" to the gradient. We call such a method noisy stochastic geodesic gradient descent, NSGGD, which is defined by the iteration
\begin{equation}\label{eq:sggd}\tag{\textsf{NSGGD}}
    \tilde{\bV}_{k+1} = \cP_{O(D,r)}(\tilde{\bV}_k - \eta_k (\bG_k + \bB_k)).
\end{equation}
Here, $\bG_k$ is an estimate of the gradient at $\tilde{\bV}_{k}$. When using minibatch stochastic gradients, we let $\bG_k = \mathsf{grad}F(\bV_k; \cB_k)$, where $\cB_k \subset \cX$. We refer to the method with minibatch stochastic gradients and zero noise as SGGD.

These methods have many potential applications. First, the minibatch SGGD method allows for less per-iteration complexity than that of GGD, where SGGD has complexity $O(|\cB|Dd)$ per iteration and GGD has complexity of $O(NDd)$ per iteration. Furthermore, the addition of noise allows for the potential development of Langevin-like algorithms on the Grassmannian. Finally, as we discuss later, when the noise has sufficiently large variance, we can show that the resulting method is differentially private.

\section{Theory}
\label{sec:theory}

In the following sections we present our theoretical results for NGGD, SGGD, and NSGGD. In particular, we prove convergence \emph{and} subspace recovery results for these methods.

First, in Section~\ref{subsec:init}, we recall a result from~\citet{maunu2019well}, which shows that PCA gives a good initial approximation to the underlying subspace with high probability. 
After this, Section~\ref{subsec:nggd} gives an iteration complexity and approximation result for NGGD. Then, Section~\ref{subsec:nsggd} gives an iteration complexity and approximation result for SGGD as well as a convergence and recovery theorem for NSGGD. The proofs of convergence for these differentially private methods require nontrivial extensions of the past proofs of convergence for GGD seen in~\citet{maunu2019well}. After this, we finish in Section~\ref{subsec:linconv} by showing how one can extend these approximation guarantees to achieve linear convergence of the NGGD, SGGD, and NSGGD algorithms with a geometrically diminishing step size scheme. For brevity, all proofs are left to the Appendix.

The results in these sections represent the main theoretical innovation of this work. Similar to the analysis of the deterministic GGD method, the strategy is to prove, under a general condition called \emph{stability}, 1) good initialization by some means, and 2) convergence of the nonconvex stochastic gradient method.

\subsection{Initialization by PCA}
\label{subsec:init}

Our nonconvex methods require initialization in a sufficiently small neighborhood of the true subspace spanned by $\bV_\star$. To accomplish this, we initialize NGGD, SGGD, and NSGGD using a PCA subspace. Later, in the case of differentially private methods, we show that one can also initialize with differentially private PCA. The main result for initialization follows.

\begin{theorem}[\cite{maunu2019well}]
	If
	\begin{equation}
    \cS^{\mathsf{PCA}}_{\gamma}(\cX):= 2\sin(\arccos(\gamma)) \, \lambda_r(\bXin \bXin^\top) - \|\bXout\|_2^2>0,
	\end{equation} 
	then $d_r^2(L_{PCA}, L_\star) < \gamma$. 
	\label{thm:pcagen}
\end{theorem}

\subsection{Noisy GGD}
\label{subsec:nggd}

Towards a complete theory for private, nonconvex robust subspace recovery, we first prove an iteration complexity and approximation result for NGGD. 
Following the analysis in~\citet{maunu2019well}, the goal is to show that the sequence $\sigma_r(\bV_\star^{\top} \bV_k)$ forms a sequence that rapidly increases with $k$.

For the convergence of GGD in~\citet{maunu2019well}, the key idea is the development of the stability statistic, which is defined as
\begin{align*}
    \cS_\gamma(\cX) &= \gamma \lambda_r \left(\frac{1}{N} \sum_{\cX_{\mathrm{in}}} \frac{\bx \bx^\top}{\|\bx\|} \right) - \max_{\bV \in O(D,d)} \sigma_1 \left( \nabla F(\bV; \cX_{\mathrm{out}}) ) \right) \\
    &= \gamma \cP(\Xin) - \cA(\Xout).
\end{align*}
Note that our parametrization of this statistic is slightly different from that of~\citet{maunu2019well}, where we use $\gamma$ instead of $\cos(\gamma)$. Under \emph{stability}, or the assumption that $\cS_\gamma(\cX) > 0$,~\citet{maunu2019well} prove local convergence of GGD given initialization in $B(\bV_\star, \gamma)$. In the following theorem, we prove convergence of NGGD when $\cS_\gamma(\cX) > 0$ as long as $\bV_0 \in B_{d_r^2}(\bV_\star, \gamma/2)$ -- notice that the noisy method requires some extra wiggle room. 
\begin{theorem}\label{thm:nggd}
	Assume that $\cS_{\gamma}(\cX) > 0$, NGGD is initialized at $\bV_0 \in B_{d_r^2}(\bV_\star, \gamma/2)$ with a constant step size {$\eta_k = s = c_1 a / T^{\nu}$}, $0.5 < \nu < 1$, and is run for $T$ iterations, where {
	\begin{align*}
	    O(N^2\epsilon^2) > T > \mathcal{F}_1(a/d_r^2(\bV_\star, \bV_0), \lambda),
	\end{align*}
	for $\cF_1$ 
	defined in~\eqref{eq:nggdTcond1}.}
	Then NGGD yields a final iterate $\bV_T \in {B(\bV_\star, a)}$ with probability at least $1-{2} \lambda$. 
\end{theorem}

By Theorem~\ref{thm:pcagen}, PCA initialization achieves the proper initialization with high probability when the 
  condition $\cS^{\mathsf{PCA}}_\gamma(\cX)>0$ holds. Theorem~\ref{thm:nggd} states that, effectively, as long as the number of iterations is $>\mathcal{F}_1(c)$, the NGGD final iterate lies in $B_{d_r^2}(\bV_\star, c \gamma)$ with high probability.
We will show in Section~\ref{subsec:linconv} how one can turn this into a linear convergence result.


\subsection{Noisy Stochastic GGD}
\label{subsec:nsggd}

This section is mainly inspired by the analysis in~\citet{zhou2020convergence}.
We first present a novel analysis of minibatch SGGD for solving~\eqref{eq:GLAD}. 

To this end, we assume minibatches $\cX^k$, $1 \leq k \leq N_B$, of size $B$ that are drawn from $\cX$ with replacement. We can separate each minibatch into inlier and outlier components $\Xin^k$ and $\Xout^k$, respectively. 
Much in the same way that one can analyze GGD and NGGD, we analyze SGGD through the use of stability statistics. For SGGD, the main difference is now each minibatch has an associated stability statistic, $\cS_\gamma (\cX^k) = \gamma\mathcal{P}(\Xin^k) - \mathcal{A}(\Xout^k)$.
As there are $N^B$ subsets $\cX^k$, we get a range of stability statistics, some of which are positive and some of which are negative. Now, instead of assuming that $\cS_\gamma(\cX) > 0$, we assume that for a minibatch selected uniformly at random from $\cX$ with replacement,
\begin{equation}
	\E \cS_\gamma(\cX^k) = \cS_{\gamma,\E} > 0. 
\end{equation}
\begin{theorem}\label{thm:sggd}
	Assume that $\cS_{\gamma,\E} > 0$, SGGD is initialized at $\bV_0 \in B_{d_r^2}(\bV_\star, \gamma/2)$ with a constant step size {$\eta_k = s = c_1 a / T^{\nu}$}, $0.5 < \nu < 1$, and is run for $T$ iterations, where {
	$$O(N^2\epsilon^2) > T > \cF_2(a/d_r^2(\bV_\star, \bV_0), \lambda),$$ for  $\cF_2$ defined in \eqref{eq:sggdTcond1}.}
	Then SGGD yields a final iterate $\bV_T \in { B(\bV_\star, a)}$ with probability at least $1-{2} \lambda$.
\end{theorem}

To additionally prove convergence of NSGGD, we must also control the noise throughout the iterations. This result essentially combines Theorems~\ref{thm:nggd} and~\ref{thm:sggd}.
As before, this theorem states that in a  number of iterations $>\mathcal{F}_3(c)$, the NSGGD final iterate lies in $B_{d_r^2}(\bV_\star, c \gamma)$ with high probability.
\begin{theorem}\label{thm:nsggd}
	Assume that $\cS_{\gamma,\E} > 0$, NSGGD is initialized at $\bV_0 \in B_{d_r^2}(\bV_\star, \gamma/2)$ with a constant step size {$\eta_k = s = c_1 a / T^{\nu}$}, $0.5 < \nu < 1$, and is run for $T$ iterations, where {
	$$O(N^2\epsilon^2) > T > \mathcal{F}_3(a/d_r^2(\bV_\star, \bV_0), \lambda),$$ for  $\cF_3$ defined in the Appendix.}
	Then NSGGD yields a final iterate $\bV_T \in {B_{d_r^2}(\bV_\star, a)}$ with probability at least $1-{4} \lambda$.
\end{theorem}

\subsection{Linear Convergence Analysis}
\label{subsec:linconv}

As we commented in the previous sections, in a constant number of iterations, $\cF_1(c)$ for NGGD, $\cF_2(c)$ for SGGD, and $\cF_3(c)$ for NSGGD, the stochastic GGD algorithms converge to $B_{d_r^2}(\bV_\star, c \gamma)$. Setting $c = 1/2$, these methods have yielded final estimates twice as close to $\bV_\star$ in a constant number of iterates.

Using this fact, the following theorem guarantees linear convergence of the stochastic GGD algorithms. To accomplish this, we use a geometrically diminishing step size. That is, we run the algorithm with a constant step size $s$ for a sufficient number of iterations. Then, the algorithm is \emph{restarted} with a constant step size $cs$ for some fraction $c \in (0,1)$. This restarting procedure is then repeated $R$ times. This is similar to the strategy used in~\citet{maunu2019well} to prove linear convergence of GGD.
\begin{theorem}\label{thm:linconv}
    Suppose that one of the stochastic GGD algorithms is run for $R$ restarts and $\cS > 0$, where $\cS = \cS_\gamma(\cX)$ for NGGD and $\cS = \cS_{\gamma, \E}(\cX)$ for SGGD and NSGGD.
    Suppose that in the first run of the algorithm (out of all the restarts), the step size is {$s = c_1 a T_1^{-\nu}$}, $0.5 < \nu < 1$ and the number of iterations is $T_1 = O(\cF_j(a/d_r^2(\bV_\star, \bV_0))) $, where $j=1$ for NGGD, $j=2$ for SGGD, and $j=3$ for $NSGGD$. Suppose further that the step size for the $l$th restart is $s/2^{l-1}$ for $T_l = \cF_j(1/2)$ iterations, and $\sum_{l=1}^R T_l = O(N^2 \epsilon^2)$. Then, with probability at least {$1-2R\lambda$ (or $1-4R\lambda$ for NSGGD)}, the output of the $R$th restart, $\hat{\bV}$, satisfies $\hat{\bV} \in B_{d_r^2}(\bV_\star, {a/2^R})$.
\end{theorem}

We see that this theorem guarantees an approximation that decreases at an exponential rate over the number of restarts.

\section{Application: Differential Privacy}

In both the NGGD and NSGGD methods, if the noise variance is sufficiently large, then the methods become {differentially private}. We guarantee the privacy of these methods in the following theorem.
\begin{theorem}[Differential Privacy of NGGD and NSGGD]\label{thm:priv}
 There exists a constant $c$ such that for any $\varepsilon < c  T$, if {$\sigma^2 \geq \frac{c T \log^2(1/\delta)}{\varepsilon^2 N^2}$}, then NGGD run for $T$ iterations is $(\epsilon, \delta)$ differentially private.
On the other hand, if the batch size is $B$, there exist constants $c_1$ and $c_2$ such that for any $\varepsilon < c_1 q^2 T$, if $\sigma^2 \geq c_2\frac{(B/N)^2 T \log(1/\delta)}{\varepsilon^2 N^2},$ then NSGGD run for $T$ iterations is $(\epsilon, \delta)$ differentially private.
\end{theorem}
The proof of differential privacy for such stochastic first-order methods is standard and follows~\citet{bassily2014private,talwar2014private}. With the noise variances as specified in Theorem~\ref{thm:priv}, we refer to the NGGD algorithm as dp-GGD and to NSGGD as dp-SGGD. When writing statements that apply to either dp-GGD or dp-SGGD, we will refer to dp-(S)GGD.

In the following sections, we examine the implications of our recovery results in the differentially private setting. First, Section~\ref{subsec:dpinit} discusses how to initialize NGGD and NSGGD in a private way. Then,  Section~\ref{subsec:dpggd} explains how the results for NGGD and NSGGD translate to the differentially private setting.
Lastly, in Section~\ref{subsec:reap} we present convex differentially private methods based on REAPER~\citep{lerman2015robust}, which gives an important baseline for subspace recovery based on differentially private convex empirical risk minimization. In particular, we extend convergence results for convex empirical risk minimization to the case of the REAPER algorithm and show that these have implications for subspace recovery.

\subsection{PCA Initialization}
\label{subsec:dpinit}

Throughout the paper, we refer to dp-PCA as the output of the differentially private PCA method of~\citet{jiang2016wishart}.
Combining the previous result in Theorem~\ref{thm:pcagen} with a result of~\citet{jiang2016wishart}, we obtain the following theorem.
\begin{theorem}\label{thm:pca}
    If $N$ is sufficiently large and $\cS^{\mathsf{PCA}}_{\gamma/2}(\cX)>0$, we have that the output of dp-PCA, $\bV_{dp-PCA}$, lies in $B_{d_r^2}(\bV_\star, \gamma/2)$ with high probability.
\end{theorem}



\subsection{Approximation for dp-(S)GGD}
\label{subsec:dpggd}

Notice that, in order for the conditions of Theorem~\ref{thm:priv} to be satisfied, we need the total number of iterations to be bounded as $T = O(N^2 \epsilon^2)$.
To get a sense of the number of restarts we can take, we note that this implies $\sum_{l=1}^R T_l = O( N^2 \epsilon^2)$. 
If we take $T_l = (\epsilon N)^\alpha$ for $0 < \alpha < 2$ for all $l=1,\dots, R$, the conditions of Theorem~\ref{thm:linconv} are satisfied once $N$ is sufficiently large. 
Therefore, we can take $R = O((\epsilon N)^{2-\alpha})$, which yields a dp-(S)GGD estimator with accuracy on the order of $1/2^{(\epsilon N)^{2-\alpha}}$, which decreases exponentially in $N$. Taking this all together, we have the following corollary of Theorems~\ref{thm:linconv} and~\ref{thm:priv}.
\begin{corollary}\label{cor:dpggdapprox}
    Running the dp-(S)GGD algorithm as in Theorem~\ref{thm:linconv} for the maximum number of restarts $R$ while still maintaining privacy in Theorem~\ref{thm:priv} yields a final iterate $\hat{\bV}$ such that $d_r^2(\hat{\bV},\bV_\star) = O(\gamma / 2^{(\epsilon N)^{2-\alpha}})$.
\end{corollary}

\subsection{Differentially Private REAPER Algorithms}
\label{subsec:reap}

One could also attempt to relax~\eqref{eq:GLAD} and solve a surrogate convex problem instead. A popular relaxation for this task is the REAPER relaxation of~\citet{lerman2015robust}. In this section, we present a simple differentially private version of this method. We can directly apply existing empirical risk minimization results to this problem~\citep{bassily2014private,bassily2019private} to yield subspace recovery guarantees. This will give us a baseline that demonstrates the superiority of the nonconvex method.

The REAPER program~\citep{lerman2015robust} solves~\eqref{eq:GLAD} by relaxing the nonconvex constraints that $\bP_L$ is an orthoprojection:
\begin{align}\tag{\sf REAP}\label{eq:REAPER}
    &\min_{\bP \in H} \frac{1}{N} \sum_{i=1}^N \| (\bI - \bP) \bx_i\|,\ \cH := \{\bP : \bzero \preceq \bP \preceq \bI, \ \Tr(\bP) = r \}.
\end{align}
This is a convex program, and so~\eqref{eq:REAPER} can be solved by an array of standard convex optimization algorithms. Since $\cX \subset S^{D-1}$, 
$G(\bP;\cX) = \frac{1}{N} \sum_{i=1}^N \| (\bI - \bP) \bx_i\|$ is 1-Lipschitz. Since the objective in~\eqref{eq:REAPER} is not smooth, one must use subgradient based methods~\citep{clarke1990optimization}.  We use the following subgradient of REAPER:
\begin{equation}\label{eq:reaper_grad}
    \nabla G(\bP;\mathcal{X}) = -\sum_{\substack{\bx\in \mathcal{X} \\ \|\bx-\bP \bx\|> 0}}\frac{ (\bI-\bP)\bx \bx^T + \bx \bx^T(\bI-\bP)}{2 \|\bx-\bP\bx\|} .
\end{equation}

While~\citet{lerman2015robust} proposes to solve this problem using an iteratively reweighted least squares method, we instead opt to study first-order methods. The first method we consider is gradient descent, and the second is a mirror descent. To make these methods differentially private, we again use the Gaussian mechanism and add noise to the gradient. 
Since past work has demonstrated advantages for considering stochastic first-order methods when making convex algorithms private~\citep{Abadi2016dl_with_privacy,bassily2019private}, we also give stochastic versions of each algorithm. 
These convex optimization methods for the REAPER problem are differentially private by the previous arguments of~\citet{bassily2014private,talwar2014private}.

Since our primary focus is on the nonconvex method, and some nonprivate versions of the convex methods were previously explored by~\citet{goes2014robust}, we leave the exact formulation of these methods to the Appendix. In the Appendix, we outline 4 differentially private algorithm for solving this REAPER program: Differentially Private Gradient Descent (dp-GD-REAP), Differentially Private Stochastic Gradient Descent (dp-SGD-REAP), Differentially Private Mirror Descent (dp-MD-REAP), and Differentially Private Stochastic Mirror Descent (dp-SMD-REAP).


Previous work on optimization with differential privacy has focused on differentially private \emph{empirical risk minimization}~\citep{bassily2014private,talwar2014private,bassily2019private}. In this general set-up, one wishes to minimize the empirical surrogate for the population loss.
In the non-stochastic setting, we can use the main theorem of~\citet{talwar2014private} for both dp-GD-REAP and dp-MD-REAP. Indeed, if one uses the mirror map $\Psi(\cdot) = \|\cdot\|^2/2$, then the algorithm just becomes dp-GD-REAP, whereas if one uses the negative von Neumann entropy, it yields dp-MD-REAP. 
The following theorem gives our main approximation result for the nonstochastic and stochastic dp-REAP algorithms. In both cases, we show that the approximation error for these private methods is on $O(1/N)$, rather than exponential like the dp-(S)GGD algorithms.
In contrast to~\citet{bassily2019private}, this theorem does not resort to smoothing the cost function and instead uses the optimization rate for subgradient descent.  

\begin{theorem}\label{thm:talwar}
	 Let $\mathsf{D}$ be the diameter of the constraint set $\cH$. Then, if dp-GD-REAP or dp-MD-REAP is run for $T = O(\epsilon^2 N^2 )$ and yields the estimator $\bar{\bP}$, we have
\begin{equation}
    \E G(\bar{\bP};\cX) - \min_{\bP} G(\bP;\cX) \lesssim \frac{\mathsf D \log(N/\delta)}{\epsilon N},
\end{equation}
where the expectation is taken over the randomness of the algorithm. On the other hand, for dp-SGD-REAP and dp-SMD-REAP, if the noise variance is $\sigma^2 = c_2\frac{B^2 T \log(1/\delta)}{\varepsilon^2 N^2}$, then
    \begin{align*}
    \E G(\bar{\bP};\cX) - \min_{\bP} G(\bP;\cX)  &\leq \sqrt{1 + c_2 B^2 \log(1/\delta)}\mathsf{D}\frac{1}{\epsilon N},
    \end{align*}
  where the expectation is taken over the randomness of the algorithm and $\mathsf{D}$ is the diameter of the constraint set with respect to the appropriate geometry.
\end{theorem}
The proof for the nonstochastic methods is just Theorem 3.2 of~\citet{talwar2014private}, and the proof for the  stochastic methods is given in Appendix~\ref{app:sgdsmd}. 



\subsubsection{Implications for Subspace Recovery}

We show that the approximate minimization guaranteed by Theorem~\ref{thm:talwar} yields in a generic setting approximate subspace recovery, or for REAPER, approximate recovery of $\bP_\star = \bP_{L_\star}$.

We recall the following following \emph{permeance}, 
\emph{alignment}, and \emph{stability}
statistics from~\citet{lerman2015robust}:
\begin{align}
    \cPREAP &= \inf_{\bu \in L_\star \cap S^{D-1}} \frac{1}{N} \sum_{\bx \in \Xin} |\bu^\top \bP_\star \bx |, \\
    \cAREAP &= \frac{1}{N} \| \bXout \|  \| \widetilde{ \bQ_\star \bXout}\|, \\
    \cSREAP &= \frac{\mathcal P}{4 \sqrt{d}} - \mathcal A.
\end{align}
Here, $\bQ_\star = \bI - \bP_\star$ and the operator $\widetilde{\cdot}$ normalizes the columns of $\bQ_\star \bXout$ to the unit sphere.
The permeance measures how well spread the inliers are on the underlying subspace, the alignment measures how aligned the outliers are orthogonal to $L_\star$, and the stability is a tradeoff between these two terms. The result in Theorem 2.1 of~\citet{lerman2015robust} states that if $\cSREAP > 0$, then $ \| \hat \bP - \bP_{\star} \|_{*} = 0$, where $\| \cdot \|_*$ is the nuclear or Schatten 1-norm. In other words, the REAPER program exactly recovers $L_\star$ once $\cSREAP > 0$.

The following Theorem states the approximation result for the REAPER algorithms of Section~\ref{subsec:reap}. In particular, it states that as $N$ increases and the stability is bounded below, the distance between the REAPER subspace and the true subspace goes to zero at a rate of $1/N$. 
\begin{theorem}
\label{thm:1overN_factor}
Suppose that $\cSREAP > 0$. If dp-MD-REAP or dp-GD-REAP are run on the REAPER problem for $T = O(\epsilon^2 N^2)$ iterations, and if $\hat L$ is the principal subspace of the output of one of these algorithms, then
\begin{equation}
    \E d^2(\hat L, L_\star) \lesssim \frac{\pi}{2 \cS} \frac{\log(N/\delta)}{\epsilon N}. 
\end{equation}
On the other hand, if $\hat L$ is the principal subspace of the output of dp-SMD-REAP or dp-SGD-REAP,
\begin{equation}
    \E d^2(\hat L, L_\star) \lesssim \frac{\pi}{2 \cS} \frac{\sqrt{1 + c_2 B^2 \log^2(1/\delta)}}{\epsilon N}. 
\end{equation}
The GD, MD, SGD, and SMD algorithms differ in their respective constants.


\end{theorem}

\subsection{Comparing Nonconvex and Convex Results}

Notice that Theorem~\ref{thm:1overN_factor} and Corollary~\ref{cor:dpggdapprox} use different distance metrics for $G(D,r)$. It turns out that up to a factor of $r$, these are equivalent: for an $r$-dimensional subspace, $d^2(L_1, L_2) \leq r \theta_1^2$, where $\theta_1$ is the maximum principal angle between $L_1$ and $L_2$. On the other hand, $1 - \sigma_r(\bV_1^\top \bV_2) = 1 - \cos(\theta_1) = O(\theta_1^2)$ for small $\theta_1$. Up to a constant factor of $r$, the metrics $d_r(\cdot, \cdot)$ and $d(\cdot, \cdot)$ are of the same order. Thus, comparing the results of Theorem~\ref{thm:1overN_factor} to Corollary~\ref{cor:dpggdapprox}, we see that the nonconvex methods have a distinct advantage in the private setting. That is, the convex algorithm only achieves an approximation error of $O(N^{-1})$ while the nonconvex methods achieve approximation errors that are $O(2^{-N^{\tau}})$

\subsection{Stability and Privacy}

We finish with a short discussion of the interaction between robustness and privacy. Consider the stability result of the GGD algorithm, which states that if $\cS_\gamma(\cX) > 0$, then GGD locally recovers the underlying subspace $L_\star$. Notice that the robustness of the method itself can yield privacy. This is stated in the following theorem.
\begin{theorem}
 Let $\cX_{-i}$ be the dataset $\cX$ with the $i$th datapoint removed. Suppose that 
 \begin{align}
     \cS_\gamma(\cX_{-i}) > 0,\  \cS^{\mathsf{PCA}}_\gamma(\cX_{-i}) > 0, \ i=1, \dots, N.
 \end{align}
 Then, GGD with PCA initialization is differentially private. 
\end{theorem}
While the condition of this theorem is hard to verify, it says that for certain inlier-outlier datasets, one doesn't even need to add noise to the GGD algorithm, since \emph{it is already private}. An in depth study of privacy is left to future work, as the focus of this work is on the convergence of stochastic GGD methods.

\section{Differential Privacy Experiments}

We performed experiments in order to demonstrate some of the predictions of the substantial theory that was developed. The settings of our experiments focus on differential privacy, but we emphasize that the results are more general and similar experiments show the benefit of NGGD, SGGD, and NSGGD in practice. Additional experiments are in the appendix.

\subsection{Synthetic Experiments}

We present two synthetic experiments in this section. The first tests the convergence properties of the proposed algorithms for a setting with fixed parameters. The second tests the methods over a range of sample sizes and dimensions to look at their effect on subspace recovery.
More comprehensive experiments that demonstrate dependencies on other parameters are in the supplemental material. All experiments were implemented on a PC with Intel i7-9700 CPUs and 16GB RAM. Below, error refers to the distance between an iterate and the underlying subspace, $d^2(\Sp(\bV_k), \Sp(\bV_\star))$.

We tested the 6 proposed algorithms: dp-(S)GD-REAP, dp-(S)MD-REAP and dp-(S)GGD. We set the step size for the 4 dp-REAP algorithms to be $\eta_k = 8/\sqrt{k}$. The step size for dp-GGD and dp-SGGD is $\eta_k = 1 / 2^{\lfloor k/50 \rfloor}$. 
We use a fixed batch size $B = \max \left(N\sqrt{\frac{\epsilon}{4T}},1\right)$ for both the convex and nonconvex methods \citep{bassily2019private}. 

For both experiments, we randomly generate datasets according to the haystack model, with Gaussian inliers and outliers, described in \citet{lerman2015robust}. Points are scaled to the sphere before running our methods.

In \autoref{fig:comparison}, we plot the median and interquartile range of log-error versus iteration for the six algorithms on 100 randomly generated sets.  The fixed model parameters are $r=2$, $D=20$, $N=2000$ and an inlier ratio 0.5. We set the privacy parameters to be $\epsilon = 0.8$ and $\delta = 1/\sqrt{N}$. All algorithms are run with $T = N$ iterations. We note that dp-(S)GGD converges faster to the underlying subspace than dp-(S)GD-REAP and dp-MD-(S)REAP, since its convergence rate is linear, unlike the sublinear rate for the convex methods. Nevertheless, in the initial 600 iterations of dp-SGD-REAP and dp-SMD-REAP, they converge at a faster rate than dp-SGGD (we also observe this with the initial 100 iterations of the non-stochastic methods). If $D$ is not large, it may be beneficial to initialize dp-(S)GGD with a corresponding dp-REAP method instead of dp-PCA.

\begin{figure}
     \centering
         \includegraphics[width=.24\columnwidth]{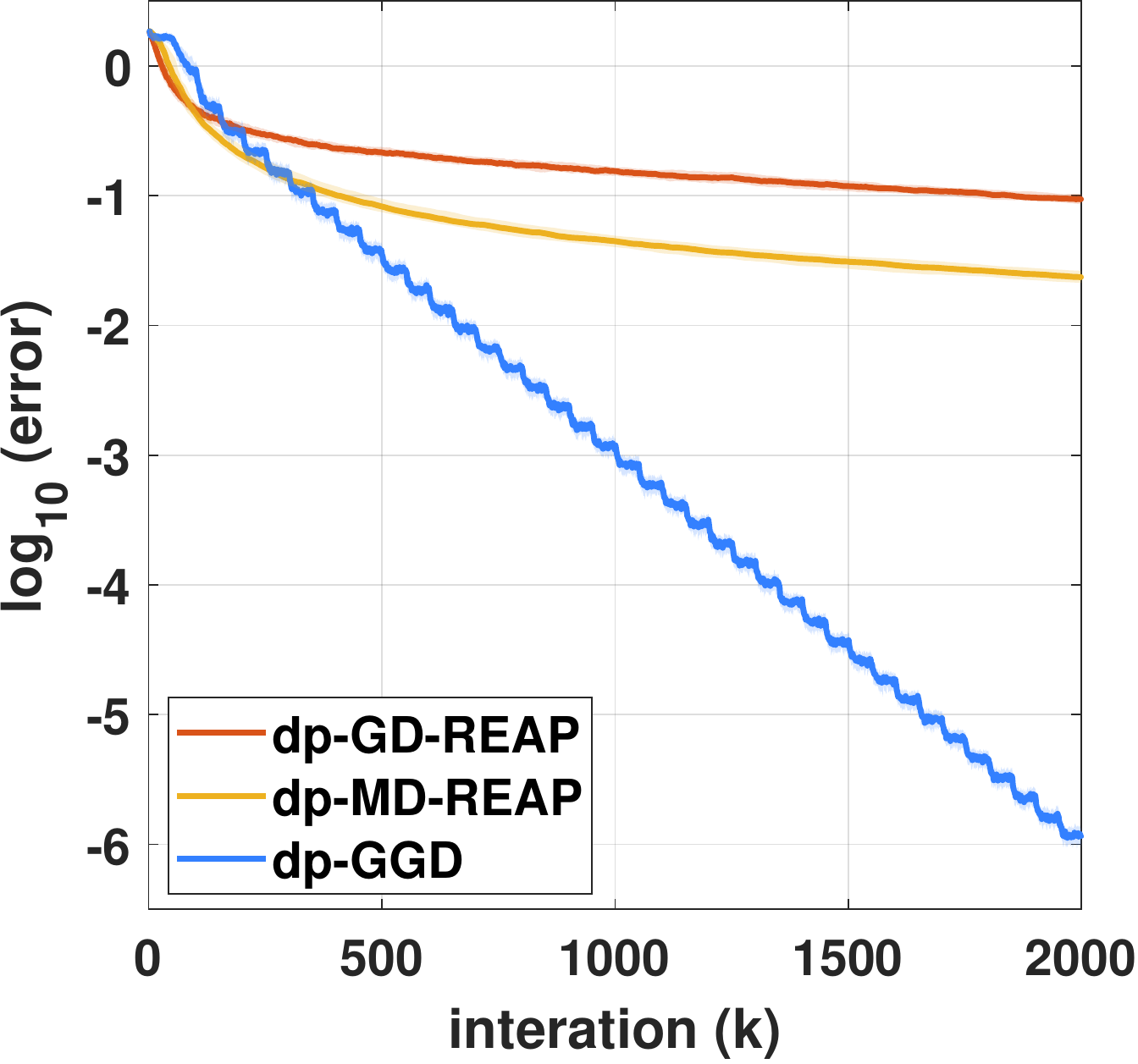}
         \includegraphics[width=.24\columnwidth]{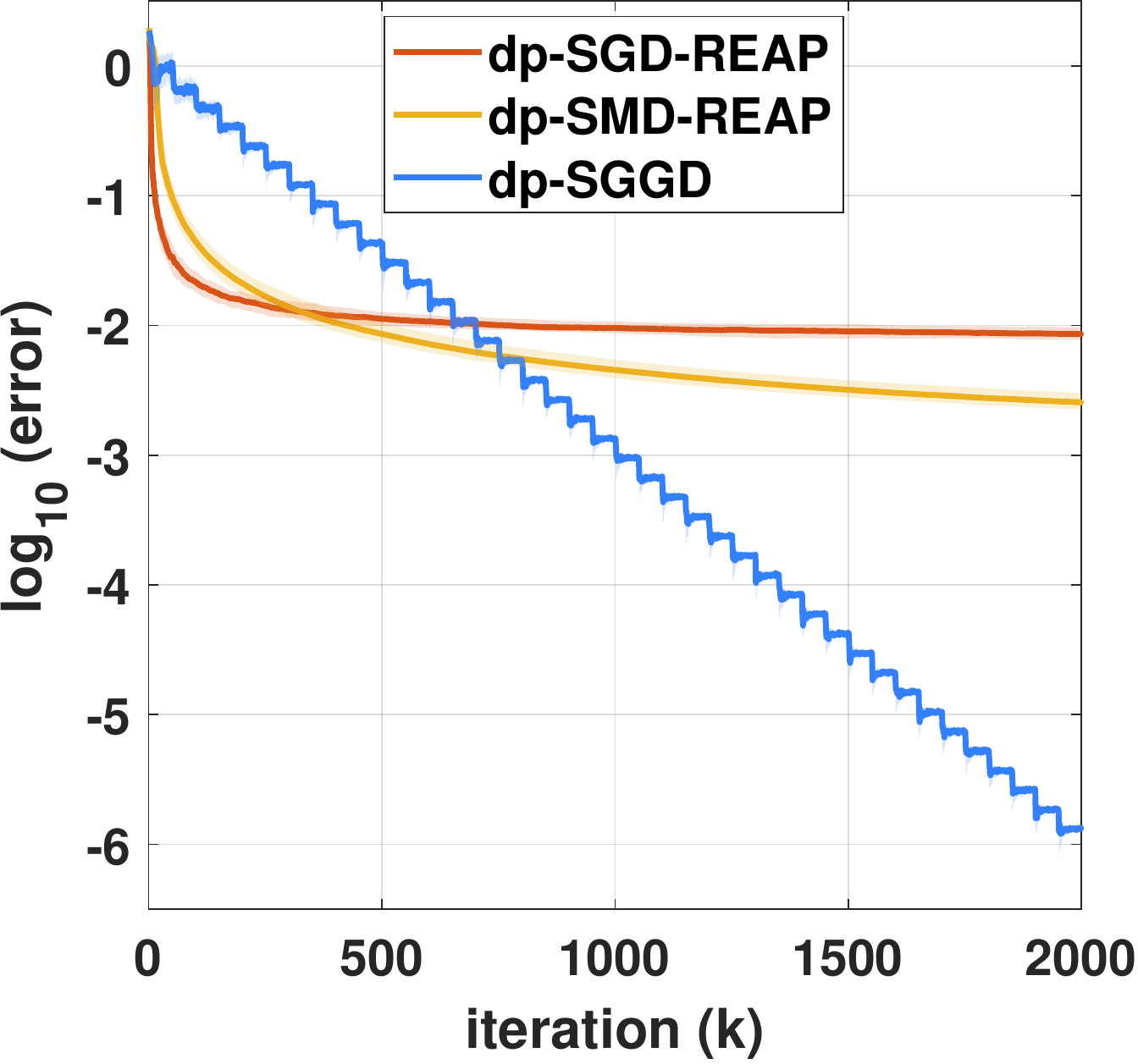}
         \includegraphics[width=.24\columnwidth]{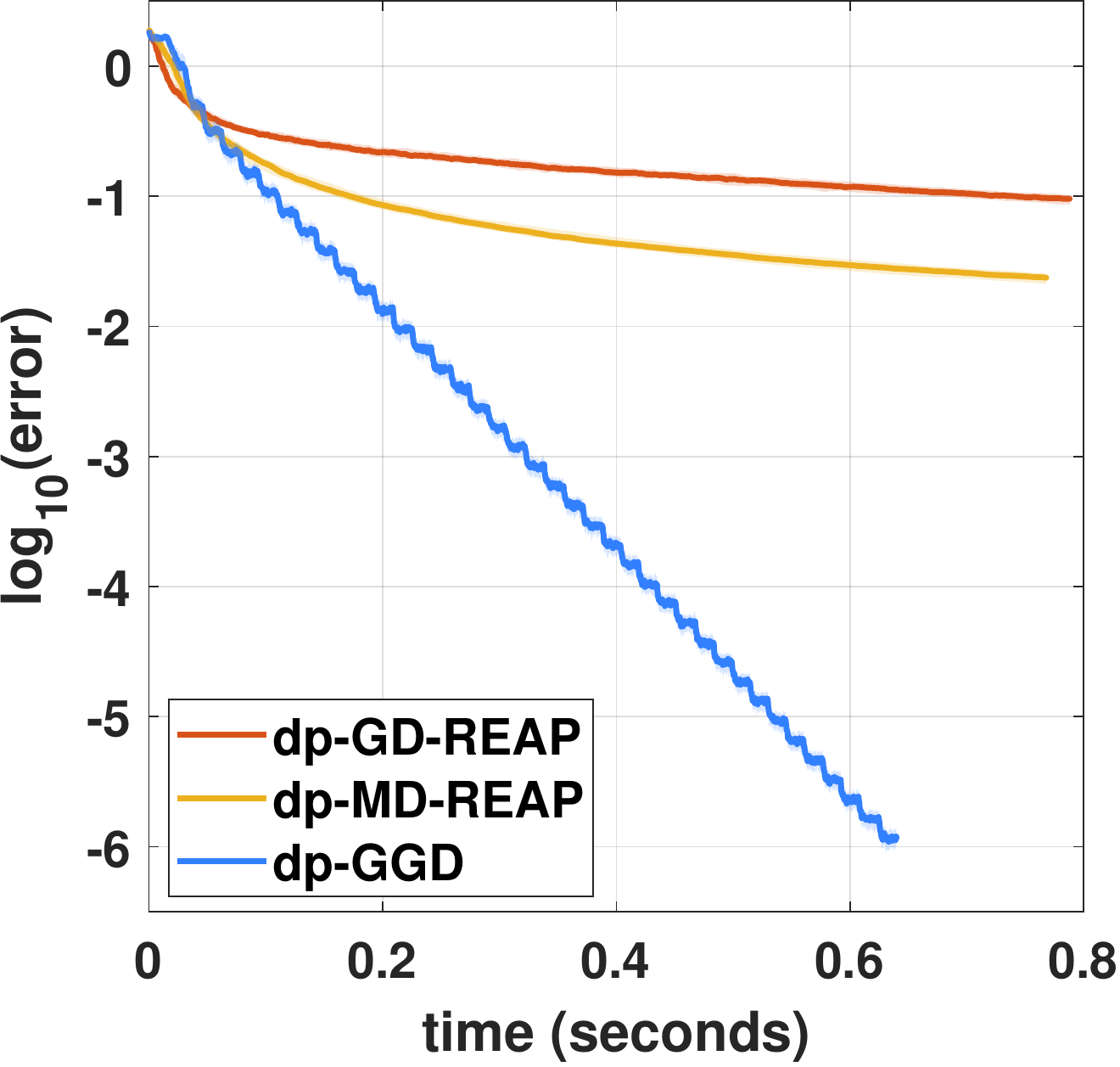}
         \includegraphics[width=.24\columnwidth]{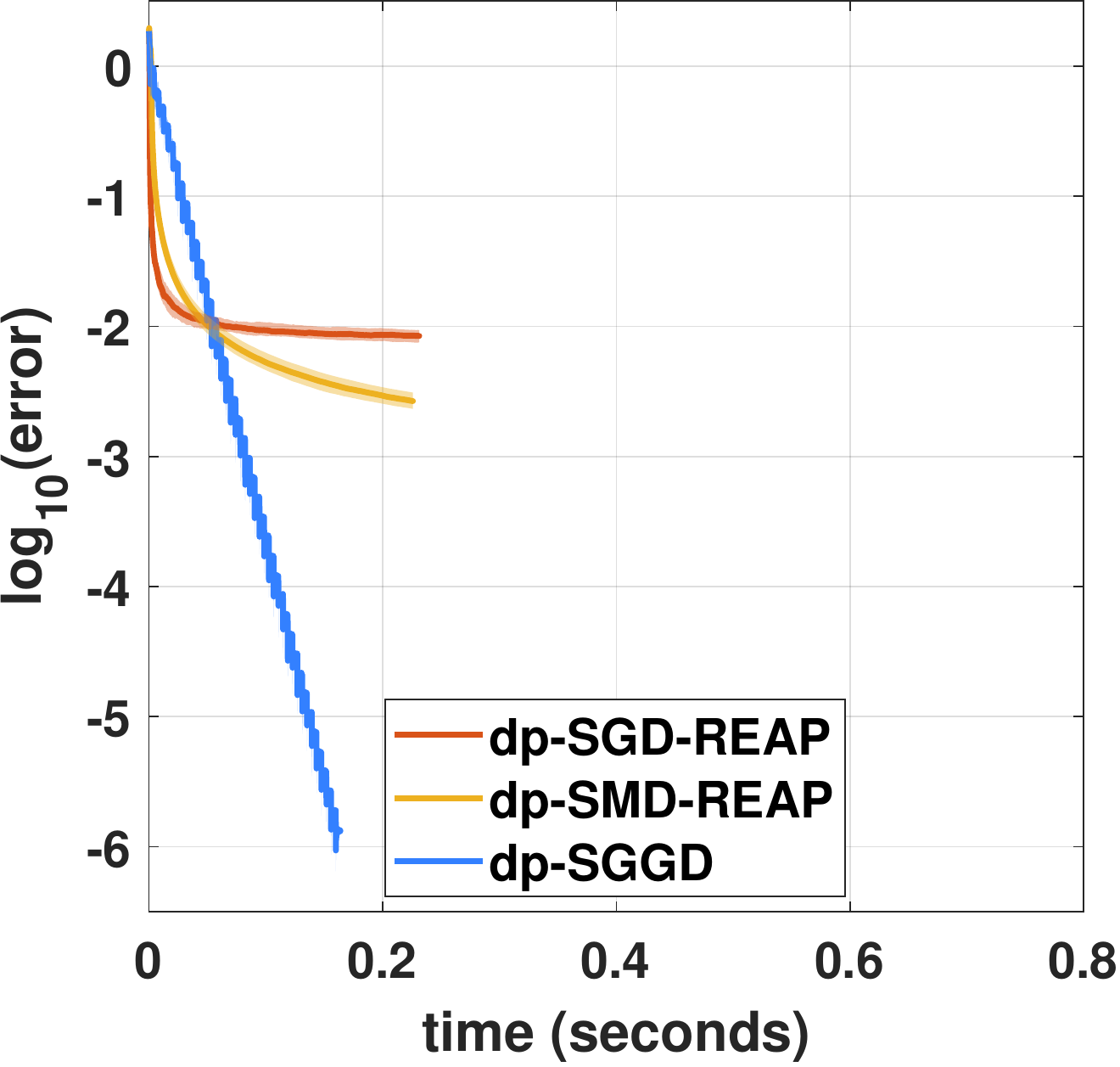}
        \caption{Convergence of the proposed  algorithms with fixed parameters (see main text). Each algorithm is repeated 100 times and the median log error with a shaded region of interquartile range is plotted as a function of iterations (left two figures) and time (right two figures). }
        \label{fig:comparison}
\end{figure}

\begin{figure}[h!]
     \centering
         \includegraphics[width=.24\columnwidth]{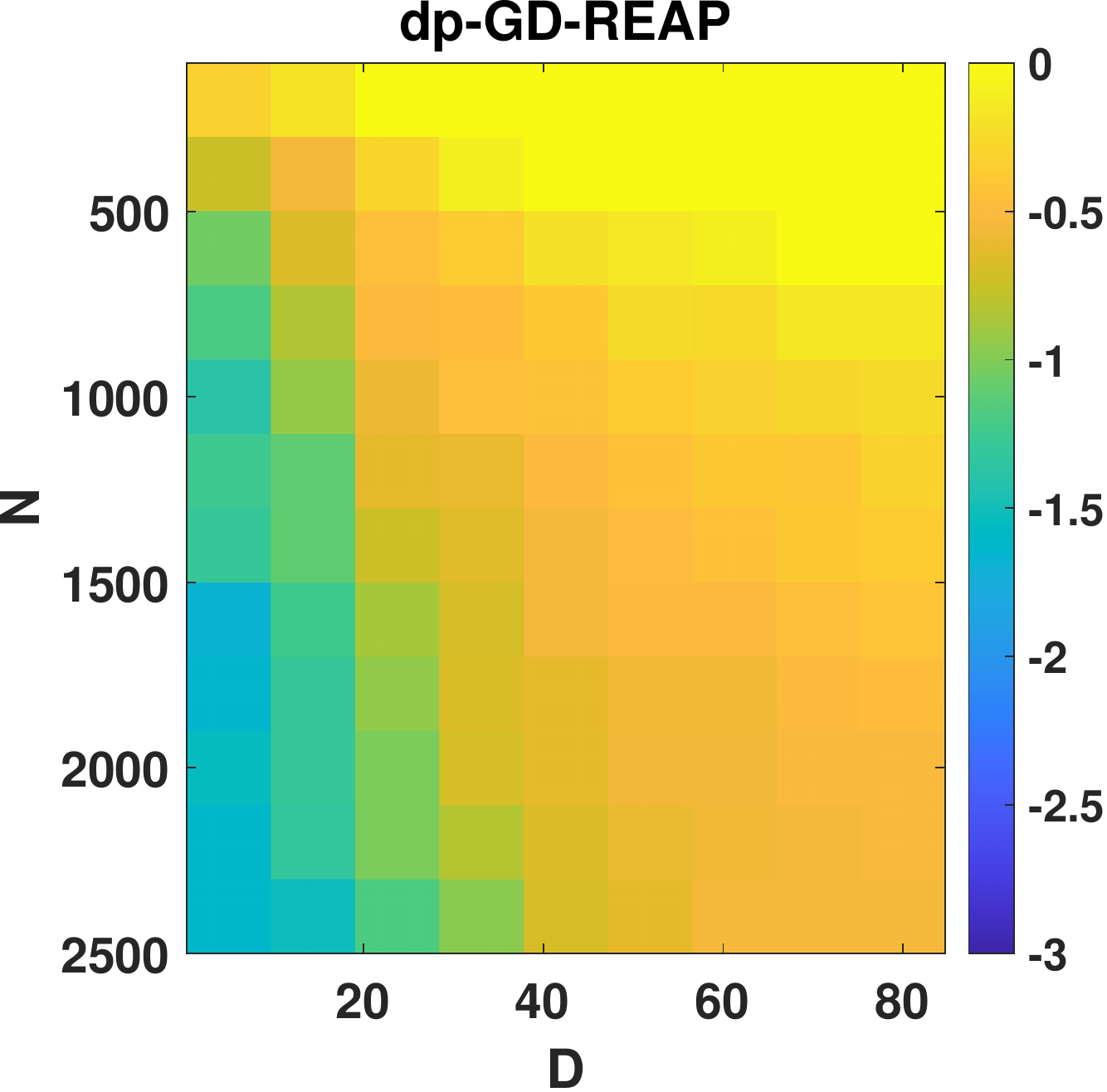}
         \includegraphics[width=.24\columnwidth]{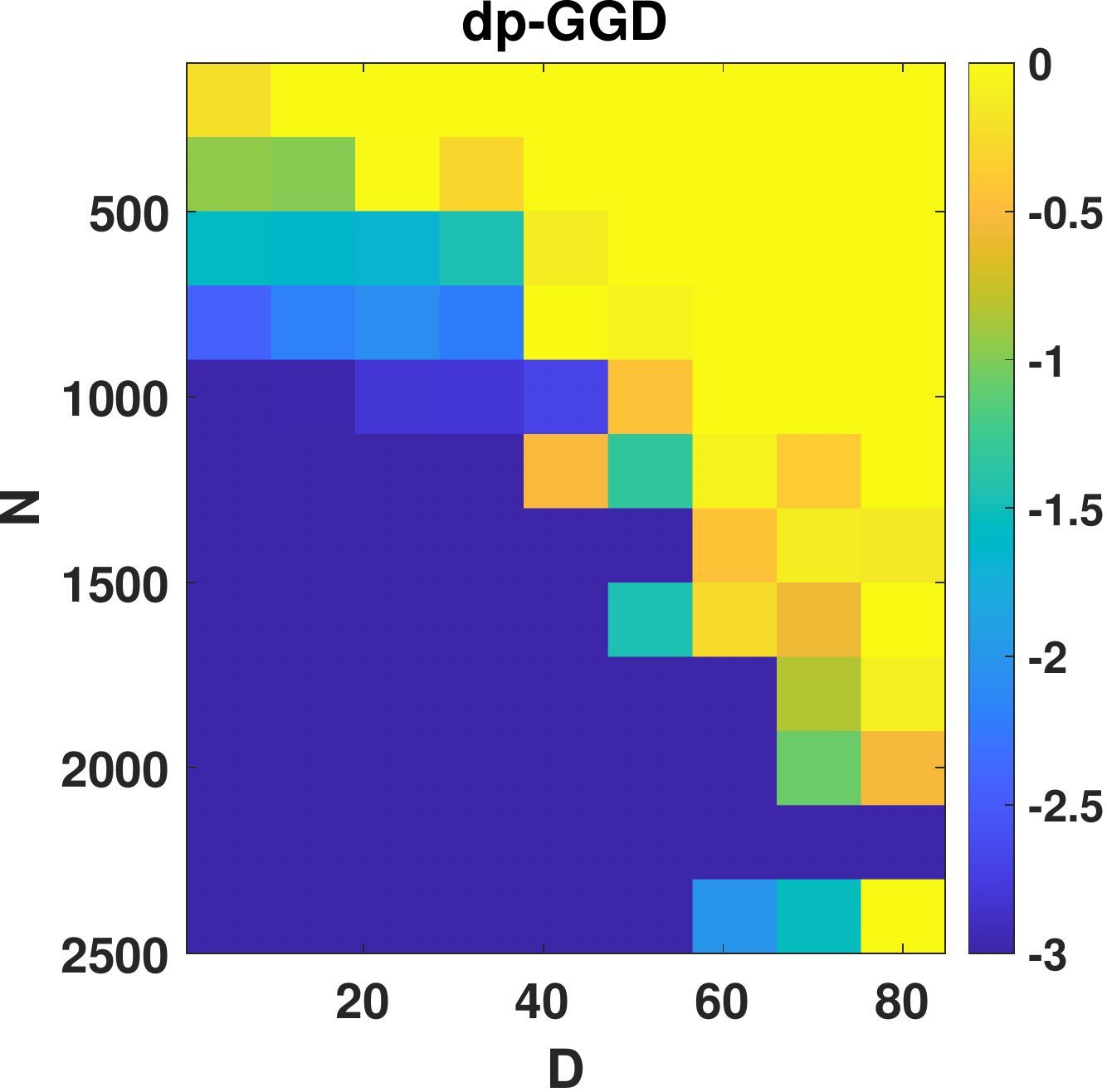}
         \includegraphics[width=.24\columnwidth]{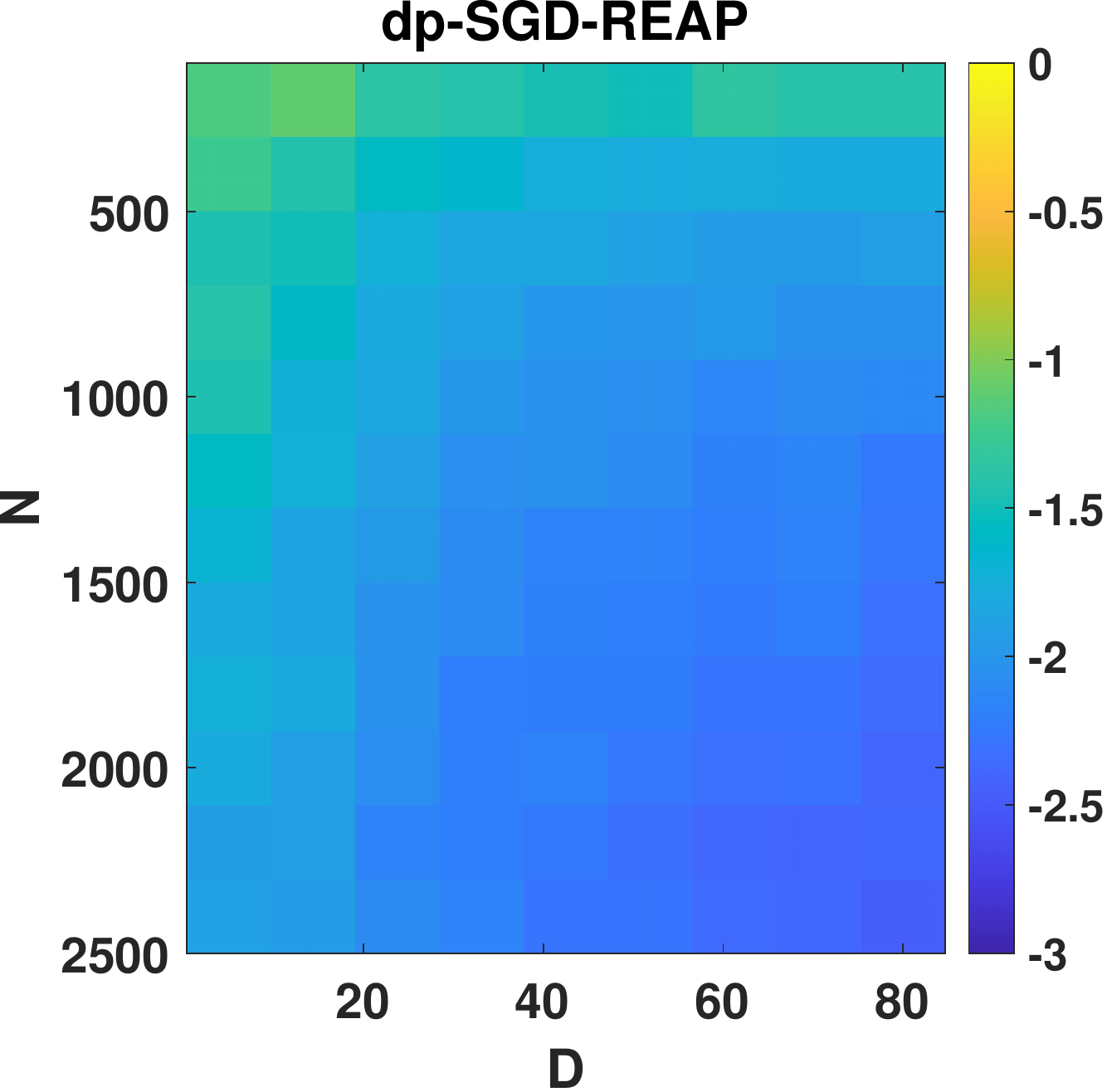}
         \includegraphics[width=.24\columnwidth]{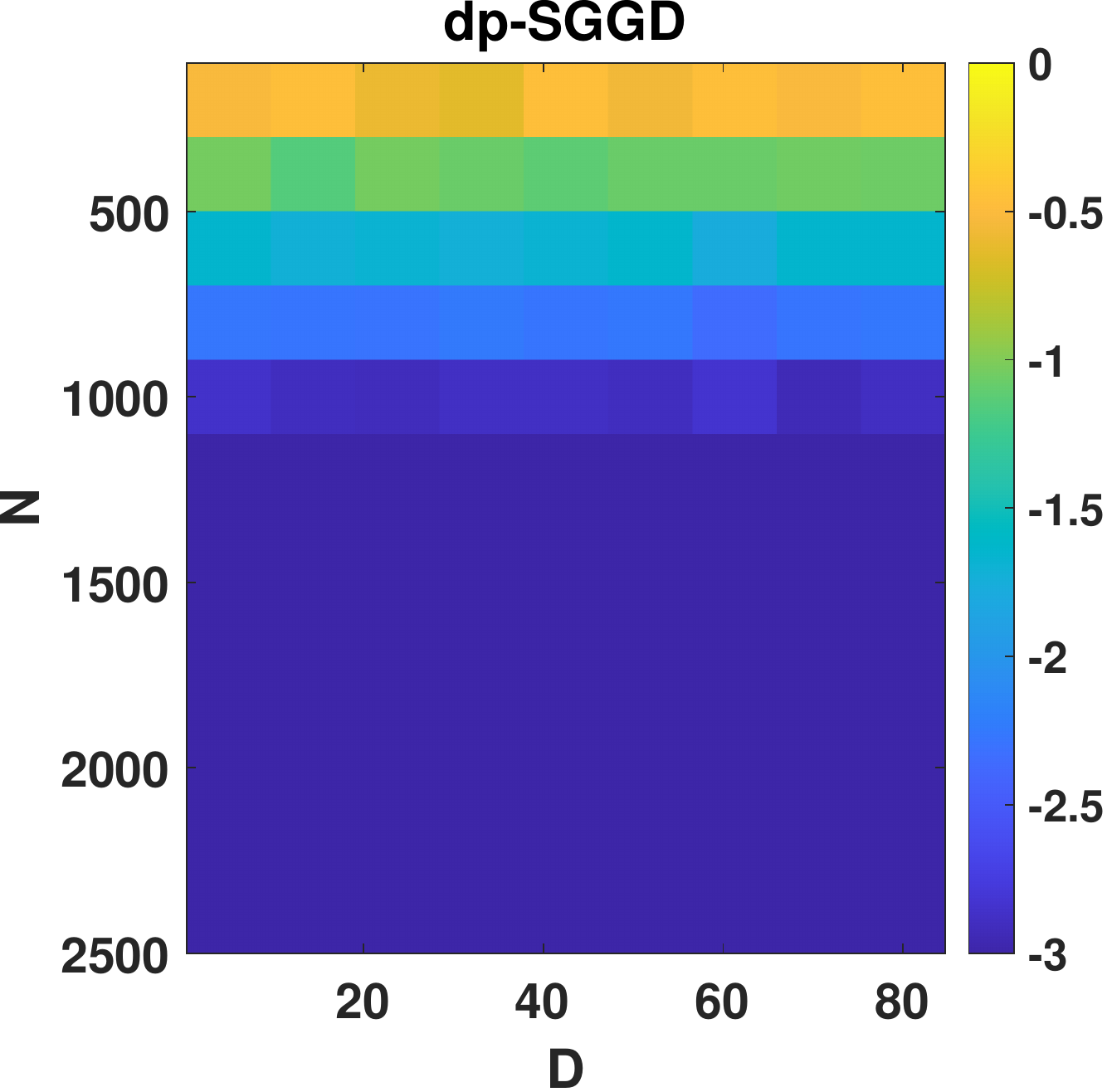}
        \caption{Phase transition plot for the number of points, $N$, versus the dimension, $D$. Each square represents the average log-distance between the final iterate and the true subspace for each algorithm.  The nonconvex methods outperform the convex methods, and the stochastic method is able to perform well for much larger $D$ due to the smaller noise required with stochastic gradients.}\label{fig:NvsD}
\end{figure}

For the second experiment, \autoref{fig:NvsD} gives a phase transition plot of $N$ vs.~$D$. The data parameters are $r=2$, percentage of inliers is $0.5$, and the total number of iterations of each algorithm is $T = 2N$, and we set the privacy parameters to be $\epsilon = 0.8$ and $\delta = 1/\sqrt{N}$. The step size for the dp-(S)GD-REAP algorithms is $\eta_k = 8/\sqrt{k}$, and the step size for dp-(S)GGD is $\eta_k = 1 / 2^{\lfloor k/50 \rfloor}$. Each algorithm is run 50 times and we display the average of the log-errors of the final iterate. In the non-stochastic case, the dp-GGD method outperforms the dp-GD-REAP method. Furthermore, the stochastic versions take a smaller noise, and so the methods are able to better approximate the underlying subspace for much larger $D$s. Finally, as predicted by the theory, the approximations for dp-(S)GGD are much more accurate than those for dp-(S)GD-REAP.

\subsection{Stylized Application: Modified POPRES}

To test on higher-dimensional data with some real characteristics, we create a stylized dataset. It aims to imitate the Population Reference Sample (POPRES) database extracted by~\citet{novembre2008europe}. This highly private database includes 3,192 European individuals with 500,568 alleles at SNP loci.  \citet{novembre2008europe} filtered SNPs and screened individuals to reduce the dataset to a sample of $N=1,387$ individuals and $D=197,146$ SNPs. They applied PCA with $r=2$ to the reduced data and demonstrated that the genetic information of the selected sample correlates with a geographical map of Europe. 

In view of our experience with the POPRES database, we find several issues with directly using the procedure of \citet{novembre2008europe} when addressing the machine learning community. First, POPRES is not publicly available. Second, the suggested preprocessing of \citet{novembre2008europe} raises some questions about the meaningful selection of reduced coordinates and individuals for which a desired correlation with a given map can be demonstrated. Furthermore, the reporting on the selection of individuals (supp.~material of \citet{novembre2008europe}) seems to reveal some private information.

In order to avoid these sensitive issues, we generated a stylized application motivated by the work of \citet{novembre2008europe}. We used the publicly available dataset provided on Github by the authors of \citet{novembre2008europe}. It was obtained by applying (non-private) PCA with $r=20$ to their reduced data, so the provided data matrix $\bX$ has size $1387 \times 20$.
To simulate high-dimensional SNP data and further privatize $\bX$, we transform it as follows: We chose $D=10,000$ and multiplied $\bX$ by a random $20 \times 10,000$ Gaussian orthogonal ensemble (GOE) matrix to get $\bX'$. For outliers, we generated a $1,000 \times 30$ random matrix of uniform i.i.d.~elements in $[-0.5,0.5]$ and multiplied this matrix by a random 
$30 \times 10,000$ GOE matrix. We thresholded the inlier and outlier matrices to obtain the three values -1, 0 and 1 to express alleles, which we then recode as 0, 1, 2 (see details in supplemental material). 
We concatenated the two matrices to form an inlier-outlier $2,387 \times 10,000$ matrix $\bY$ with elements in $\{0,1,2\}$.

\begin{figure}[h!]
     \centering
         \includegraphics[width=.24\columnwidth]{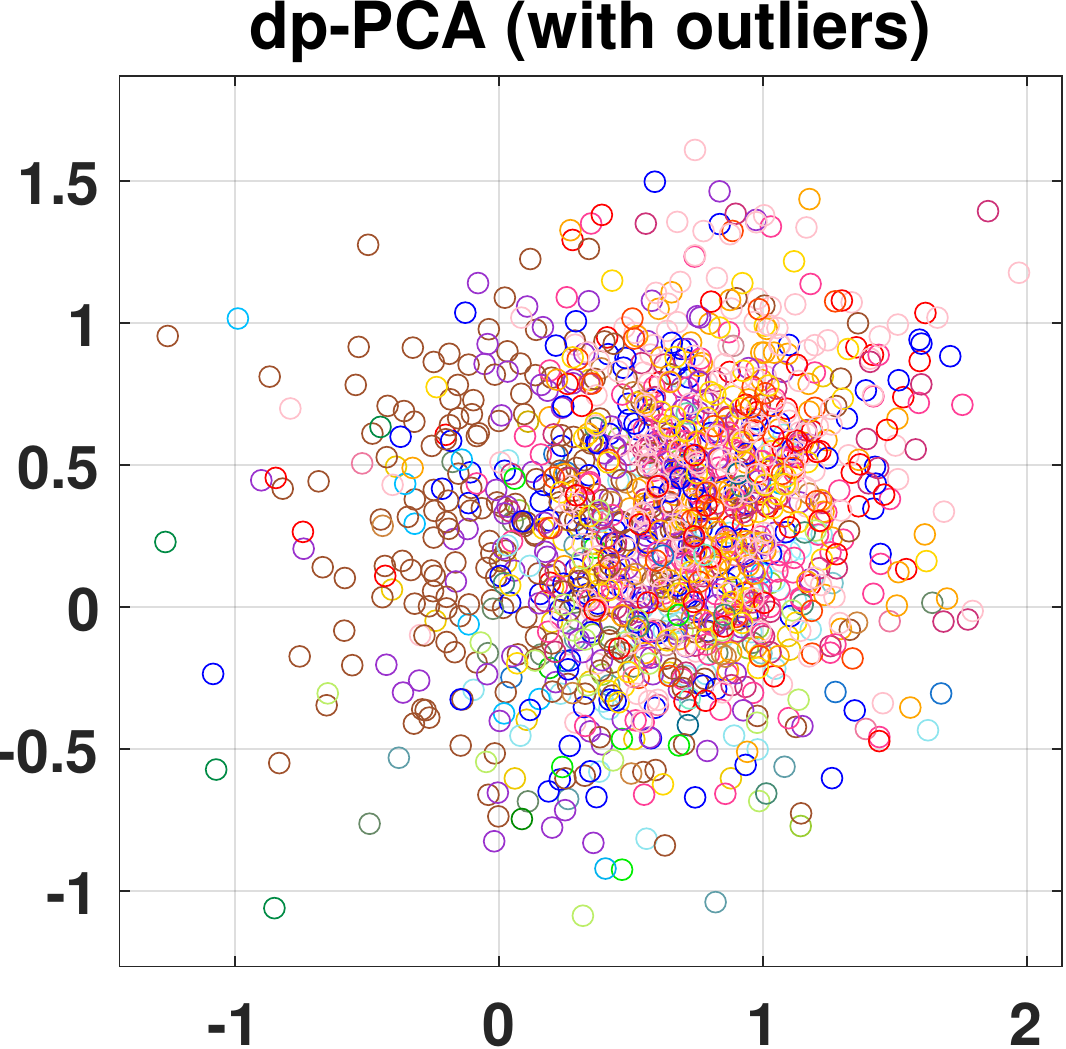}
         \includegraphics[width=.24\columnwidth]{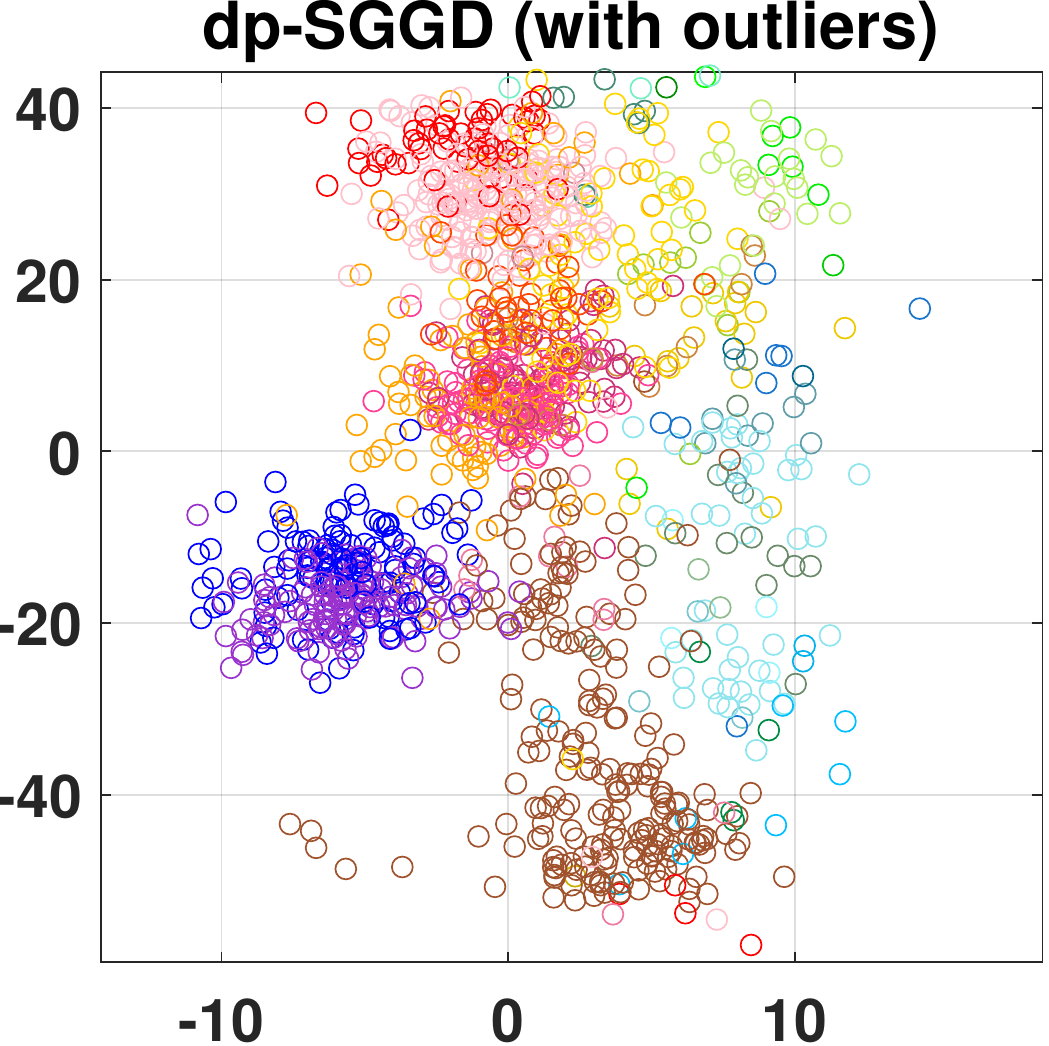}
         \includegraphics[width=.24\columnwidth]{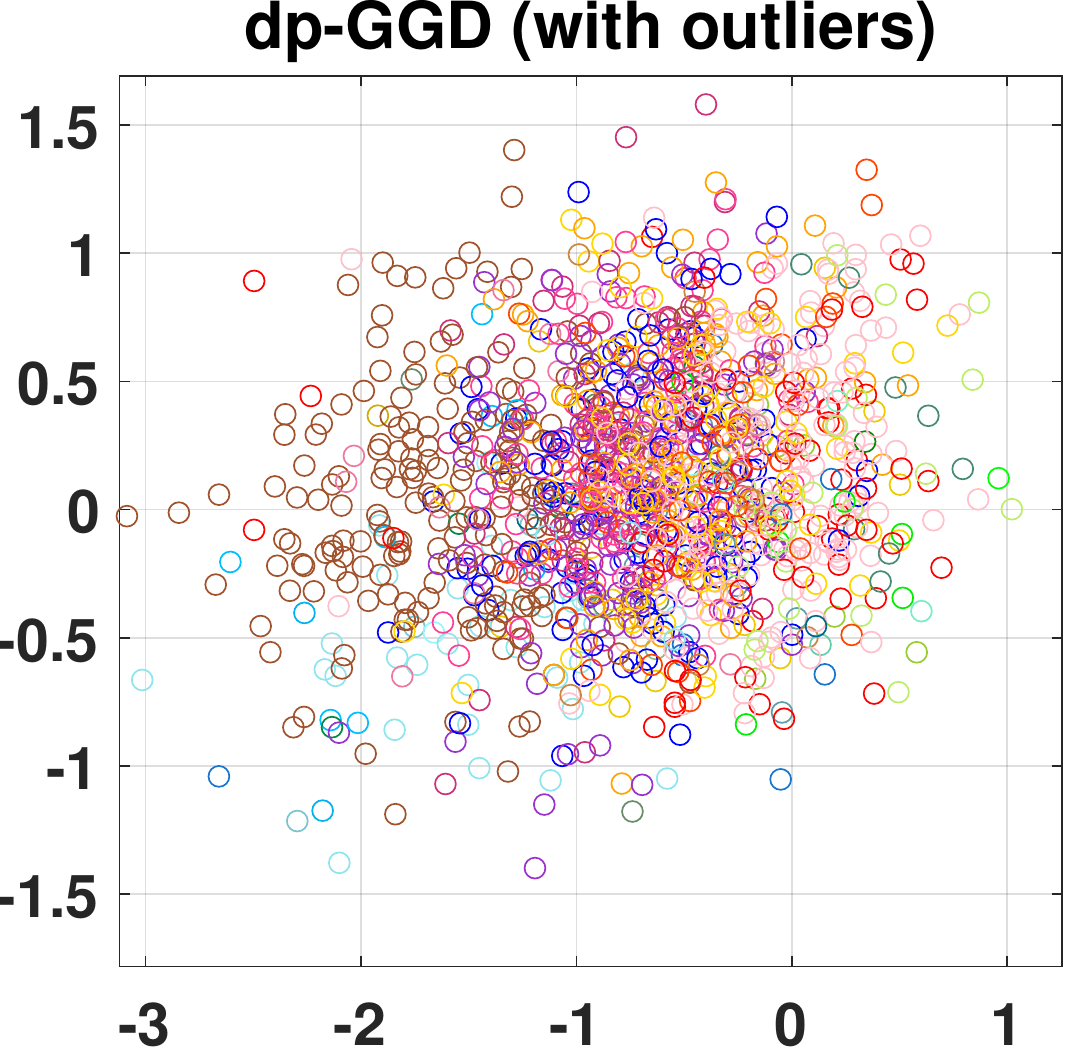}
         \includegraphics[width=.23\columnwidth]{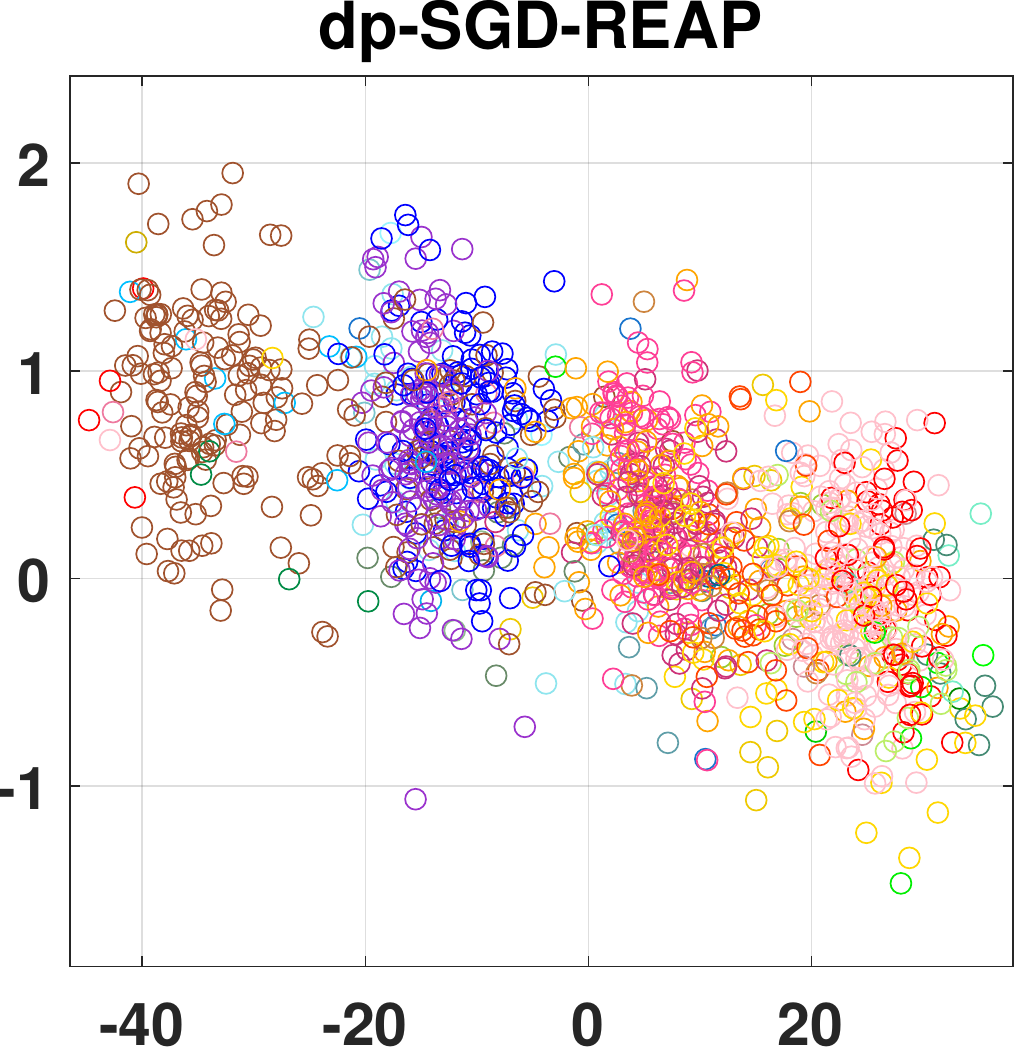}
        \caption{Recovered projections for the stylized POPRES dataset. Each algorithm is run on a synthetically generated gene matrix which mimics the original SNP data with $N = 2387$ and $D = 10000$. Out of these points, $1387$ lie close to the underlying subspace, which recovers the shape of Europe, and 1000 outliers are generated to lie in the ambient space. We note that dp-SGGD is able to recover the correct shape of Europe, unlike 
        other algorithms.}
        \label{fig:stylizedPOPRES}
\end{figure}

\autoref{fig:stylizedPOPRES} demonstrates the application of dp-PCA \citep{chaudhuri2013near}, dp-SGGD, dp-GGD and dp-SGD-REAP to $\bY$, and then plotting the projection of only the inliers (which are also in $\bX'$). We note that dp-PCA, dp-GGD and dp-SGD-REAP are unable to recover the target 2-dimensional subspace which indicates the shape of Europe, whereas the embedding of dp-SGGD is relatively successful in doing this. Additional figures and runtimes are included in the supplemental material.






\section{Conclusion}

In this work, we initiate the first study of differentially private ORPCA algorithms. Our results demonstrate the distinct advantages of taking a nonconvex geometric approach to solving the ORPCA problem privately. In particular, we show that the nonconvex dp-(S)GGD algorithm converges linearly to the underlying subspace under a standard assumption of stability, in contrast to the convex method that only converges sublinearly. The techniques we use to guarantee the nonconvex dp-(S)GGD are interesting in their own right because they are the first proofs of convergence for stochastic methods in nonconvex formulation of least absolute deviations for ORPCA. Furthermore, our experiments confirm our results and demonstrate the advantages of dp-GGD and dp-SGGD. In fact, dp-SGGD seems to be superior to dp-GGD due to its ability to use smaller noise in the Gaussian mechanism.

It would be interesting to further extend the NGGD results to studying the mixing of Langevin dynamics.


There are directions to explore in future  work. The following limitations are of main interest to theoreticians. We only focus here on large $N$, where some constants depend on $D$. It would be interesting to study the high-dimensional regime; however, even current works on dp-PCA do not seem to apply to this regime. There are also limitations due to the theoretical setting of ORPCA. 
First, we consider here the common setting of inliers lying exactly on the subspace and it would also be interesting to consider the interplay of noisy inliers and differentially private subspace recovery. Second, we assume centered data and search for linear subspaces and it will be good to justify differentially private centering approaches or extend this work searching for affine subspaces.
Finally, robustness can enhance privacy~\citep{dwork2009differential}, but we did not explore this in the main text.


In addition, more practical limitations are as follows. First, we lack experiments on real data, though we explained the difficulty of working with and reporting results of the POPRES data. Second, while we theoretically verify privacy, we do not yet have a good test to verify that the algorithms are in fact private. Third, we do not know if our bounds are optimal, and it would be good to tighten these results as well as prove lower bounds. Fourth, the result for the nonconvex case is only local, and so it is unclear how the methods perform in general settings. Fifth, the choice of parameters is not sufficiently clear in the nonconvex case, and even in the convex case the estimates are only approximate. At last, we require the strong assumption of an inlier-outlier model, and it is not clear in general when data may meet this assumption.

\newpage
\bibliographystyle{abbrvnat}
\bibliography{refs}

\newpage

\newpage

\appendix

\begin{center}
    \Large \textbf{Supplementary Material}
\end{center}

\section{REAPER Algorithms}

In Algorithm~\ref{alg:sgd}, we give the dp-SGD-REAPER Algorithm.  The full dp-GD-REAP algorithm follows the same steps, but takes the full dataset as a batch at each iteration, uses the noise variance $\sigma^2 \geq 32 T \log^2(T/\delta)/(\varepsilon^2 N^2)$~\citep{talwar2014private}. The projection step into $\cH$ that is in line 6 of Algorithm~\ref{alg:sgd} can be implemented as the water-filling procedure of~\citet{lerman2015robust}.

We can also use mirror descent to minimize the REAPER objective. This results in the dp-MD-REAP algorithm, which we write in Algorithm~\ref{alg:md}. As in the previous case, the full MD algorithm takes the full dataset as a batch at each iteration and uses the noise variance $\sigma^2 \geq 32 T \log^2(T/\delta)/(\varepsilon^2 N^2)$~\citep{talwar2014private}. The mirror map is the von Neumann entropy, $\Psi(\cdot) = \frac{1}{4}(\Tr(\cdot \log(\cdot)) + \log(D))$, and this approach was used before in~\citet{goes2014robust}. It turns out that the Bregman projection for this choice of mirror map just corresponds to trace renormalization.

\begin{algorithm2e}[ht]
\SetAlgoLined
\KwIn{Dataset $\mathcal{X}$; subspace dimension $r$; step sizes $\eta_k$; max number of step $T$; privacy parameters: $(\epsilon,\delta)$; regularization parameter:$\lambda$;  batch size $B$; noise variance $\sigma^2 = c_2\frac{(B/N)^2 T \log(1/\delta)}{\varepsilon^2 N^2}$.}
$\bP_0 \gets \bm{A}^T\bm{A}$, where $\bm{A} \in \mathbb{R}^{D\times D}$ and $\bm{A}_{ij}\sim N(1,0.01)$\;
\For{$k = 1:T$}{
Sample a batch $\mathcal{B}_k \subset \cX$ of size $B$ uniformly with replacement\;
Sample $\bB_k \in \R^{D \times D}$ such that $\bB_k = \bB_k^\top$ and $\bB_{k,ij} \sim N(0, \sigma^2)$\;
 $\tilde{\nabla}F(\bmP_k;\mathcal{B}_k) = \nabla F(\bmP_k;\mathcal{B}_k) + \bm{B}_k$ (using \eqref{eq:reaper_grad})\;
Update: $\bmP_{k+1} := \argmin_{\bP \in \cH} \Big\| \bP - \Big[\left(\bmP_k - \eta_k\left( \tilde{\nabla}F(\bmP_k;\mathcal{B}_k)\right)\right)\Big] \Big\|$\;
}
\KwOut{$\tilde{\bmP} \gets \frac{1}{T} \sum_{k=1}^T \bmP_k$}
\caption{dp-SGD-REAP}
\label{alg:sgd}
\end{algorithm2e}

\begin{algorithm2e}[ht]
\KwIn{Dataset $\mathcal{X}$; subspace dimension $r$; step sizes $\eta_k$; max number of step $T$; privacy parameters: $(\epsilon,\delta)$; regularization parameter:$\lambda$; batch size $B$; noise variance $\sigma^2 = c_2\frac{(B/N)^2 T \log(1/\delta)}{\varepsilon^2 N^2}$. 
}
$\bP_0 \gets \bm{A}^T\bm{A}$, where $\bm{A} \in \mathbb{R}^{D\times D}$ and $\bm{A}_{ij}\sim N(1,0.01)$\;
\For{$k = 1:T$}{
Sample a batch $\mathcal{B}_k \subset \cX$ of size $B$ uniformly with replacement\;
Sample $\bB_k \in \R^{D \times D}$ such that $\bB_k = \bB_k^\top$ and $\bB_{k,ij} \sim N(0, \sigma^2)$\;
 $\tilde{\nabla}F(\bmP_k;\mathcal{B}_k) = \nabla F(\bmP_k;\mathcal{B}_k) + \bm{B}_k$ (using \eqref{eq:reaper_grad})\;
$ \bmP_{k+1} = \exp\left[\log(\bmP_k) - \eta_k \tilde{\nabla} F (\bmP_k;\mathcal{B}_k )\right]$\;
$\bmP_{k+1} = r \bmP_{k+1} / \Tr(\bmP_{k+1})$\;
}
\KwOut{$\tilde{\bmP} \gets \frac{1}{T} \sum_{k=1}^T \bmP_k$}
\caption{dp-SMD-REAP}
\label{alg:md}
\end{algorithm2e}

\section{Supplemental Theory}

\subsection{Projection and Geodesic Gradient Descent}

While the methods in~\ref{eq:nggd} and~\eqref{eq:sggd} are projected gradient methods, due to \citet{gawlik2018high}, these iteration very well approximate geodesics:
\begin{align}
    \cP_{O(D,r)}({\bV}_k - &\eta_k (\nabla F({\bV}_k;\cX) + \bB_k)) = \\ \nonumber &\Exp_{\bV_k} (- \eta_k (\nabla F(\bV_k;\cX) + \bB_k)) + O(\eta_k^3).
\end{align}

\subsection{Doob's Maximial Inequality}

For the convergence of the dp-GGD, SGGD, and dp-SGGD algorithms, we use the following version of \emph{Doob's maximal inequality}.

\begin{theorem}[Doob's maximal inequality]
    Suppose that $S_t = \sum_{i=1}^t X_i$ is a martingale with respect to the filtered probability space $(\Omega, \cF, \{\cF_t\}, \P)$, then
    \begin{equation}
        \P\left( \max_{t=1, \dots, T} |S_k| \geq x  \right) \leq \frac{\E |S_T|^2}{x^2}.
    \end{equation}
\end{theorem}
In the following proofs, the corresponding filtered probability space should be apparent from context.

\subsection{Differential Privacy of dp-GGD}

\begin{proof}[Proof of Theorem~\ref{thm:priv}]
The result for dp-GGD is just a rehash of the proof of~\citet[Theorem 6.1]{wang2017differentially} using the strong composition theorem.
The result for dp-SGGD is proven in ~\citet{Abadi2016dl_with_privacy} using the moment accounting method.
\end{proof}

\subsection{Convergence of Stochastic Mirror Descent}
\label{app:sgdsmd}

Here we prove Theorem~\ref{thm:talwar}.
\begin{proof}
    
By a classic result that can be found, for example, in~\citet{bubeck2015convex}, if the stochastic oracle is such that $\E \|(\nabla G(\bP;\cB)) + \bB \|^2 \leq R^2$,
$$\E G(\bar{\bP}) - \min_{\bP} G(\bP) \leq R \mathsf{D} \sqrt{\frac{2}{T}}.$$

In this statement, the randomness is taken over the randomness of the minibatches $\cB$, as well as the randomness in the Gaussian noise $\bB$. We have
$$\E \|(\nabla G(\bP;\cB)) + \bB \|^2 \leq \E \|(\nabla G(\bP;\cB)) \|^2 + \E \|\bB \|^2 $$
For the first, we have that $\E \|(\nabla G(\bP;\cB)) \|^2 \leq 1$. For the second,  
\begin{align*}
    \E \|\bB \|^2 &= \sigma^2 \E \|\bZ\|^2 \lesssim \sigma^2 M  \\
    &= c_2\frac{B^2 T \log^2(1/\delta)}{\varepsilon^2 N^2} M.
\end{align*}
Plugging in,
\begin{align*}
    \E G(\bar{\bP}) - \min_{\bP} G(\bP) &\leq \sqrt{1 + c_2\frac{B^2 T \log^2(1/\delta)}{\varepsilon^2 N^2} M} \mathsf{D} \sqrt{\frac{2}{T}}.
\end{align*}
If $T = O(N^2 \epsilon^2)$, then
\begin{align*}
    \E G(\bar{\bP}) - \min_{\bP} G(\bP) &\leq \sqrt{1 + c_2 B^2 \log^2(1/\delta)}\mathsf{D}\frac{1}{\epsilon N}.
\end{align*}

\end{proof}





\subsection{Proof of Theorem~\ref{thm:1overN_factor}}

Theorem~\ref{thm:talwar} only guarantees an approximation to the minimizer of $F$ after $T$ iterations. To turn this then into a result of approximate recovery for REAPER that we see in Theorem~\ref{thm:1overN_factor}, we rely on the rate of ascent for the perturbed objective by Lemma 2.3 of \citet{lerman2015robust}.
\begin{lemma}[\cite{lerman2015robust}, Lemma 2.3]\label{lem:reaperpert}
   $$ G(\bP_\star + \bDelta) - G(\bP_\star) \geq \mathcal S \| \bDelta\|_{*}$$
\end{lemma}

This states that the objective grows quickly when one is far from $\bP_\star$. Therefore, if the excess risk is on the order of $1/N$, then the perturbation can also be bounded on the order of $1/N$. This is stated in the following theorem.
\begin{theorem}\label{thm:reapdist}
Suppose that an algorithm to solve the REAPER problem yields a point $\hat \bP$ such that $F(\hat \bP) - G(\bP_\star) \leq \epsilon$. Let $\hat L$ be the principal $r$-subspace of $\hat \bP$. Then,
\begin{equation}
    d(\hat L, L_\star) \leq \frac{\pi}{2} \frac{\epsilon}{\cS}
\end{equation}
\end{theorem}
\begin{proof}
This follows from combining Lemma~\ref{lem:reaperpert} with the Davis-Kahan $\sin \Theta$ Theorem~\citep{Davis1970,yu2015useful}.
\end{proof}

The proof of Theorem~\ref{thm:1overN_factor} follows by combining Theorem~\ref{thm:reapdist} with Theorem~\ref{thm:talwar}.

\subsection{dp-GGD Proofs}

\subsubsection{Proof of Theorem~\ref{thm:pca}}

Based on~\citet{jiang2016wishart}, we have
\begin{align*}
    \cos(\theta_1(L_{dp-PCA}, L_\star)) &= d_r^2(\bV_{dp-PCA}, \bV_\star) \leq \|\bV_{dp-PCA} \bV_{dp-PCA}^\top - \bV_\star \bV_\star^\top\|_2 \\
     &\leq \|\bV_{dp-PCA} \bV_{dp-PCA}^\top - \bV_{PCA} \bV_{PCA}^\top\|_2 + \|\bV_{PCA} \bV_{PCA}^\top - \bV_\star \bV_\star^\top\|_2 \\
     &\leq 2 \sqrt{d} \|\bW\|_2 + \gamma/4 \\
     &\leq 2 \sqrt{d}O(d\log(d)/(N \epsilon)) + \gamma/4,
\end{align*}
with high probability. For $N$ sufficiently large, we find 
$d_r^2(\bV_{dp-PCA}, \bV_\star) < \gamma$ with high probability.

\subsubsection{Proof of Theorem~\ref{thm:nggd}}
\label{subsec:proofnggd}




\begin{proof}

We can write 
\[
\sigma_r(\bV_\star^{\top} \bV_{k+1}) \geq \sigma_r(\bV_\star^\top \bV_k) + \Big(\bbeta_1^\top  \bV_\star^{\top}( - s(\sfgrad F(\bV_k;\cX) + \bB_k) \bbeta_2 \Big) - c(\cX)s^2.
\]
Taking one minus both sides yields
\[
d_r^2(\bV_\star, \bV_{k+1}) \leq d_r^2(\bV_\star^\top \bV_k) + s\Big(\bbeta_1^\top  \bV_\star^{\top}( (\sfgrad F(\bV_k;\cX) + \bB_k) \bbeta_2 \Big) + c(\cX)s^2.
\]
Using the fact that $\sin(\arccos(x)) = \sqrt{1-x^2}\geq \sqrt{1-x}$, $x \geq 0$, stability implies that
\[
    s\Big(\bbeta_1^\top  \bV_\star^{\top}( (\sfgrad F(\bV_k;\cX) + \bB_k) \bbeta_2 \Big) \leq -s \mathcal S_\gamma(\cX) d_r(\bV_\star, \bV_k) + s\bbeta_1^\top  \bV_\star^{\top}\bB_k\bbeta_2 
\]
Thus
\begin{equation}\label{eq:distdecr}
    d_r^2(\bV_\star, \bV_{k+1}) \leq \Big(1-\frac{s\cS_\gamma(\cX)}{d_r(\bV_\star^\top \bV_k)}\Big) d_r^2(\bV_\star,\bV_k) + s\bbeta_1^\top  \bV_\star^{\top}\bB_k\bbeta_2 + c(\cX)s^2.
\end{equation}
Let $m_T = \max_{j=1,\dots,T} d_r(\bV_\star, \bV_j)$. We can iteratively apply this to yield
\begin{align*}
    d_r^2(\bV_\star, \bV_{T}) \leq (1-s\cS_\gamma(\cX)/m_T)^T d_r^2(\bV_\star, \bV_T) + &\sum_{j=1}^{T} s(1-s\cS_\gamma(\cX)/m_T)^{T-j}\bbeta_1^\top  \bV_\star^{\top}\bB_k\bbeta_2 + \\
    &\sum_{j=1}^T (1-s\cS_\gamma(\cX)/m_T)^{T-j} c(\cX) s^2.
\end{align*}

\noindent
\textbf{Bounding the maximum:}
We proceed by first bounding $m_T$ in terms of $d_r^2(\bV_\star, \bV_0)$ and other quantities. To do this, first notice that this amounts to bounding $d_r^2(\bV_\star, \bV_k)$ for $k=1,\dots, T$. We can use~\eqref{eq:distdecr} along with $\sin(\arccos(x)) \geq 1-x$ to write
\[
d_r^2(\bV_\star, \bV_{k+1}) \leq \Big(1-s\cS_\gamma(\cX)\Big) d_r^2(\bV_\star,\bV_k) + s\bbeta_1^\top  \bV_\star^{\top}\bB_k\bbeta_2 + c(\cX)s^2.
\] 
We proceed by applying Doob's maximal inequality
\begin{multline*}
    \Pr\Big(\max_{k=1,\dots,T} \Big| \sum_{j=1}^k s (1-s \cS_\gamma(\cX))^{k-j} \bbeta_1^{j\top} \bV_\star^\top \bB_j \bbeta_2^j \Big| > \epsilon \Big) \\ \leq \Pr\Big(\max_{k=1,\dots,T} \Big| \sum_{j=1}^k s \bbeta_1^{j\top} \bV_\star^\top \bB_j \bbeta_2^j \Big| > \epsilon \Big) 
    \leq \frac{s^2 T \sigma^2(\sqrt{D} + \sqrt{d})^2}{\epsilon^2}.    
\end{multline*}
Choosing $\epsilon = \frac{s \sqrt{T} \sigma (\sqrt{D} + \sqrt{d})}{\sqrt{\lambda}}$ yields
\begin{align*}
    \Pr\Big(\max_{k=1,\dots,T} \Big| \sum_{j=1}^k s (1-s \cS_\gamma(\cX))^{k-j} \bbeta_1^{j\top} \bV_\star^\top \bB_j \bbeta_2^j \Big| > \epsilon \Big) &\leq \lambda.
\end{align*}
With probability at least $1-\lambda$, we thus have that for all $k=1,\dots,T$,
\[
d_r^2(\bV_\star, \bV_{k}) \leq \Big(1-s\cS_\gamma(\cX)\Big)^k d_r^2(\bV_\star,\bV_0) + \frac{s \sqrt{k} \sigma (\sqrt{D} + \sqrt{d})}{\sqrt{\lambda}}  + \frac{c(\cX)s k}{\cS_\gamma(\cX)}.
\] 
Thus, if
    \[
         \frac{s\sqrt{T}  \sigma (\sqrt{D} + \sqrt{d})}{\sqrt{\lambda}}  < \frac{d_r^2(\bV_\star, \bV_0)}{2}
    \]
    and
    \[
         \frac{s^2 T c(\cX)}{\cS_{\gamma}(\cX)} < \frac{d_r^2(\bV_\star, \bV_0)}{2},
    \]
    then $m_T < 2 d_r^2(\bV_\star, \bV_0)$. In particular, for $s = c_1 a T^{-\nu}$, these are satisfied if
    \[
        T > \max \Big(\Big[\frac{ 2c_1 a^2 c(\cX)}{\cS_{\gamma}(\cX) d_r^2(\bV_\star, \bV_0)}\Big]^{1/(2\nu-1)}, \Big[\frac{ 2a \sigma (\sqrt{D} + \sqrt{d})}{\sqrt{\lambda} d_r^2(\bV_\star, \bV_0)}\Big]^{2/(2\nu-1)} \Big).
    \]
    Since $a > a^2$, a sufficient condition is 
    \begin{equation}\label{eq:nggdTcond1}
        T > \max \Big(\Big[\frac{ 2c_1 a c(\cX)}{\cS_{\gamma}(\cX) d_r^2(\bV_\star, \bV_0)}\Big]^{1/(2\nu-1)}, \Big[\frac{ 2a \sigma (\sqrt{D} + \sqrt{d})}{\sqrt{\lambda} d_r^2(\bV_\star, \bV_0)}\Big]^{2/(2\nu-1)} \Big) =: \cF_1(a/d_r^2(\bV_{\star}, \bV_0), \lambda).
    \end{equation}
    Notice that this is a function of $a/d_r^2(\bV_{\star}, \bV_0)$ and $\lambda$.
    
Choosing $s = \frac{C a }{T^{\nu}} \leq \frac{C d_r^2(\bV_\star, \bV_0)}{2 \sqrt{T}}$ then yields that $m_T < 2 d_r^2(\bV_\star, \bV_0)$.
In particular, as long as we initialize in $B_{d_r^2}(\bV_\star, \gamma/2)$, we see that stability holds throughout all iterations with probability at least $1-\lambda$.

\noindent
\textbf{Bounding the $T$th iterate:}
Let $m_0 = 2 d_r^2(\bV_\star, \bV_0)$.
From here the proof is straightforward: the first term geometrically decreases.  The second can be bounded with Doob's maximal inequality with high probability and uses the fact that 
\[
\sum_j (1-s \cS_\gamma(\cX)/m_0)^{2(T-j)} \leq \frac{1}{1 - (1-s \cS_\gamma(\cX)/m_0)^{2}} = \frac{m_0^2}{2s \cS_\gamma(\cX) m_0 - (s \cS_\gamma(\cX))^2},
\]
which is independent of $T$. 
More specifically, Doob's maximal inequality yields
\begin{multline*}
    \Pr \Big( \inf_{1 \leq k \leq T} \Big|\sum_{j=1}^k s(1-s \cS_\gamma(\cX)/m_0)^{T-j} \bbeta_1^{j\top}\bV_\star^\top \bB_j \bbeta_2^j\Big| > \epsilon \Big) \leq \\
    \frac{\E \Big(\sum_{j=1}^T s^2(1-s \cS_\gamma(\cX)/m_0)^{2(T-j)} \Big[ \bbeta_1^{j\top}\bV_\star^\top \bB_j \bbeta_2^j\Big]^2\Big)}{\epsilon^2}.    
\end{multline*}

We can upper bound 
\begin{align*}
    \E (\bbeta_1^{j\top}\bV_\star^\top \bB_j \bbeta_2^j)^2 &= \E (\balpha_j^\top \bB_j \bbeta_2^j  )^2 \leq \sigma^2 \E \sigma_1 ( \bZ_j )^2 \\
    &\lesssim \sigma^2 (\sqrt{D} + \sqrt{d})^2
\end{align*}
In any case, this implies that
\begin{multline*}
    \Pr \Big( \inf_{1 \leq k \leq T} \Big|\sum_{j=1}^k s(1-s \cS_\gamma(\cX)/m_0)^{T-j} \bbeta_1^{j\top}\bV_\star^\top \bB_j \bbeta_2^j \Big|\leq \\
    C \frac{\sigma (\sqrt{D} + \sqrt{d})\sqrt{\sum_{j=1}^T s^2(1-s \cS_\gamma(\cX)/m_0)^{2(T-j)}} }{\sqrt{\lambda}}\Big) \leq \lambda,
\end{multline*}
or
\begin{multline*}
    \Pr \Big( \inf_{1 \leq k \leq T} \Big|\sum_{j=1}^k s(1-s \cS_\gamma(\cX)/m_0)^{T-j} \bbeta_1^{j\top}\bV_\star^\top \bB_j \bbeta_2^j \Big|\leq \\
    Cs \frac{\sigma (\sqrt{D} + \sqrt{d}) }{\sqrt{\lambda}}\sqrt{\frac{m_0^2}{2s \cS_\gamma(\cX) m_0 - (s \cS_\gamma(\cX))^2}}\Big) \leq \lambda.
\end{multline*}

The last term uses the fact that 
\[
\sum_{j=1}^T (1-s\cS_\gamma(\cX)/m_0)^{T-j} \leq \frac{m_0}{s\cS_\gamma(\cX)}
\]

Putting these together, we find with probability at least $1-2\lambda$,
\begin{multline*}
d_r^2(\bV_\star, \bV_{T}) \leq (1-s\cS_\gamma(\cX)/m_0)^T d_r^2(\bV_\star^\top \bV_0) \\
+ s\Big[ C \frac{\sigma (\sqrt{D} + \sqrt{d}) }{\sqrt{\lambda}}\frac{m_0}{\sqrt{2s \cS_\gamma(\cX) m_0 - (s \cS_\gamma(\cX))^2}}  +  \frac{m_0 c(\cX)}{\cS_\gamma(\cX)}\Big].    
\end{multline*}

Now, if $T$ is sufficiently large so that $s = \frac{C a }{T^{\nu}}$ satisfies
\begin{align*}
    s &< \frac{m_0}{\cS_\gamma(\cX)}, \\
    s &< \frac{a}{2 m_0} \Big[ C \frac{\sigma (\sqrt{D} + \sqrt{d}) }{\sqrt{\lambda}}\frac{1}{\sqrt{2s \cS_\gamma(\cX) m_0- (s \cS_\gamma(\cX))^2}}  +  \frac{c(\cX)}{\cS_\gamma(\cX)} \Big]^{-1}.
\end{align*}
The second condition can be satisfied for $T$ greater than a constant $C'$ with respect to $a$. Indeed, 
\begin{multline*}
   \Big[ \frac{1}{\sqrt{sm_0}} C \frac{\sigma (\sqrt{D} + \sqrt{d}) }{\sqrt{\lambda}}\frac{1}{\sqrt{2 \cS_\gamma(\cX) - s  (\cS_\gamma(\cX))^2/ m_0}}  +  \frac{c(\cX)}{\cS_\gamma(\cX)} \Big]
   \\
   \leq \Big[ \frac{1}{\sqrt{sm_0}} C \frac{\sigma (\sqrt{D} + \sqrt{d}) }{\sqrt{\lambda}}\sqrt{\frac{1}{ \cS_\gamma(\cX)}}  +  \frac{c(\cX)}{\cS_\gamma(\cX)} \Big]
   = O(1/\sqrt{s m_0}).
\end{multline*}
Thus, to satisfy the second condition, we need 
\[
s = \frac{C a}{T^{\nu}} < C' \frac{a }{2 m_0} \sqrt{m_0} \sqrt{\frac{C a}{T^{\nu}}},
\]
or
\begin{equation}\label{eq:nggdTcond2}
    T = \Omega \Big(\Big[ \frac{m_0}{a} \Big]^{1/\nu}\Big).
\end{equation}

Choosing $T$ to be sufficiently large so that it also satisfies
\begin{equation}\label{eq:nggdTcond3}
    T > \Big[\mathcal{C} \frac{m_0  \log(a/2 d_r^2(\bV_\star^\top \bV_0))}{ a \cS_\gamma(\cX)}\Big]^{1/(1-\nu)} > \frac{ \log(a/2 d_r^2(\bV_\star^\top \bV_0))}{\log(1-s\cS_\gamma(\cX)/m_0)},
\end{equation}
then, with probability at least $1-2\lambda$,
\begin{equation}
    d_r^2(\bV_\star, \bV_T) < a.
\end{equation}

Notice that the constraints on $T$ given by~\eqref{eq:nggdTcond1},~\eqref{eq:nggdTcond2}, and~\eqref{eq:nggdTcond3} depend only on $a$ and $d_r^2(\bV_\star, \bV_0)$ the ratio $a/d_r^2(\bV_\star, \bV_0)$.



\end{proof}

\subsubsection{Proof of Theorem~\ref{thm:sggd}}

Since $\cS(\cX^k)$ is bounded between $[-\max_i \|\bx_i\|^2, \max_i \|\bx_i\|^2]$, we also have the uniform bound
\begin{equation}\label{eq:squarebd}
	\E [\cS(\cX^k) - \cS_\E]^2 \leq \max_i \|\bx_i\|^2,
\end{equation}
although in general we expect this to be much smaller. In particular, if the data is spherized, then this is bounded by 1.

\begin{proof}
    The proof of the theorem follows from the same reasoning as dp-GGD after splitting the sequence of errors as
    \begin{align*}
        d_r^2(\bV_\star, \bV_{k+1}) &\leq d_r^2(\bV_\star^\top \bV_k) + s\Big(\bbeta_1^\top  \bV_\star^{\top}( (\sfgrad F(\bV_k;\cX)) \bbeta_2 \Big) + c(\cX)s^2\\
        &\leq \Big(1-\frac{s\cS_{\E}}{d_r(\bV_\star^\top \bV_k)}\Big) d_r^2(\bV_\star,\bV_k) + (\cS_{\gamma}(\cX^k) - \cS_{\gamma,\E}) s d_r(\bV_\star,\bV_k) + c(\cX)s^2 \\
        &\leq \Big(1-\frac{s\cS_{\E}}{m_T}\Big) d_r^2(\bV_\star,\bV_k) + (\cS_{\gamma}(\cX^k) - \cS_{\gamma,\E}) s d_r(\bV_\star,\bV_k) + c(\cX)s^2,
    \end{align*}
    and then controlling $(\cS_{\gamma}(\cX^k) - \cS_{\gamma,\E}) s d_r(\bV_\star,\bV_k)$.
    Here, again, $m_T = \max_{j=1,\dots,T} d_r(\bV_\star, \bV_j)$.
    
    \noindent
    \textbf{Bounding $m_T$:} As before, be begin by bounding $m_T$ by first looking at the looser bound
    \[
    d_r^2(\bV_\star, \bV_{k+1}) \leq \Big(1-s\cS_{\E}\Big) d_r^2(\bV_\star,\bV_k) + (\cS_{\gamma}(\cX^k) - \cS_{\gamma,\E}) s d_r(\bV_\star,\bV_k) + c(\cX)s^2.
    \]
    Telescoping yields
    \begin{multline*}
        d_r^2(\bV_\star, \bV_{k}) \leq \Big(1-s\cS_{\E}\Big)^k d_r^2(\bV_\star,\bV_0) \\
        + \sum_{j=1}^k \Big(1-s\cS_{\E}\Big)^{k-j} (\cS_{\gamma}(\cX^j) - \cS_{\gamma,\E}) s d_r(\bV_\star,\bV_j) + \sum_{j=1}^k \Big(1-s\cS_{\E}\Big)^{k-j}c(\cX)s^2.        
    \end{multline*}

    The last term is bounded by
    \[
        \sum_{j=1}^k \Big(1-s\cS_{\E}\Big)^{k-j}c(\cX)s^2 \leq   \frac{s c(\cX) }{ \cS_{\gamma,\E}}.
    \]
    The other term can be bounded by Doob's maximal inequality
    \begin{multline*}
        \Pr\Big(\max_{k=1,\dots,T} \Big|\sum_{j=1}^k \Big(1-s\cS_{\E}\Big)^{k-j} (\cS_{\gamma}(\cX^j) - \cS_{\gamma,\E}) s d_r(\bV_\star,\bV_j) \Big| > \epsilon \Big) \\
        \leq \Pr\Big(\max_{k=1,\dots,T} \Big| \sum_{j=1}^k (\cS_{\gamma}(\cX^j) - \cS_{\gamma,\E}) s \Big| > \epsilon \Big) 
    \leq \frac{s^2 \sum_{j=1}^T \E (\cS_{\gamma}(\cX^j) - \cS_{\gamma,\E}) }{\epsilon^2}
    \leq \frac{s^2 T }{\epsilon^2}.
    \end{multline*}
    Setting $\epsilon = \frac{s \sqrt{T}}{\sqrt{\lambda}}$ yields
    \[
    \Pr\Big(\max_{k=1,\dots,T} \Big|\sum_{j=1}^k \Big(1-s\cS_{\E}\Big)^{k-j} (\cS_{\gamma}(\cX^j) - \cS_{\gamma,\E}) s d_r(\bV_\star,\bV_j) \Big| > \frac{s \sqrt{T}}{\sqrt{\lambda}} \Big) \leq \lambda.
    \]
    Thus, if
    \[
         \frac{s\sqrt{T}}{\sqrt{\lambda}} < \frac{d_r^2(\bV_\star, \bV_0)}{2}
    \]
    and
    \[
         \frac{s^2 T c(\cX)}{\cS_{\gamma,\E}} < \frac{d_r^2(\bV_\star, \bV_0)}{2},
    \]
    then $m_T < 2 d_r^2(\bV_\star, \bV_0)$. In particular, for $s = c_1 a T^{-\nu}$, these are satisfied if
    \begin{equation}\label{eq:sggdTcond1}
        T > \max \Big(\Big[\frac{ 2c_1 a^2 c(\cX)}{\cS_{\gamma,\E} d_r^2(\bV_\star, \bV_0)}\Big]^{1/(2\nu-1)}, \Big[\frac{ 2a }{\sqrt{\lambda} d_r^2(\bV_\star, \bV_0)}\Big]^{2/(2\nu-1)} \Big)=: \cF_2(a/d_r^2(\bV_{\star}, \bV_0), \lambda).
    \end{equation}
    
    \noindent
\textbf{Bounding the $T$th iterate:}
Let $m_0 = 2 d_r^2(\bV_\star, \bV_0)$.
From here the proof is straightforward: the first term geometrically decreases.  The second can be bounded with Doob's maximal inequality with high probability and uses the fact that 
\[
\sum_j (1-s \cS_{\gamma,\E}/m_0)^{2(T-j)} \leq \frac{1}{1 - (1-s \cS_{\gamma,\E}/m_0)^{2}} = \frac{m_0^2}{2s \cS_{\gamma,\E} m_0 - (s \cS_{\gamma,\E})^2},
\]
which is independent of $T$. 
More specifically, Doob's maximal inequality yields
\begin{align*}
    \Pr \Big( &\inf_{1 \leq k \leq T} \Big|\sum_{j=1}^k s(1-s \cS_\gamma(\cX)/m_0)^{T-j} (\cS_{\gamma}(\cX^j) - \cS_{\gamma,\E}) d_r(\bV_\star,\bV_j) \Big| > \epsilon \Big) \\
    &\leq \Pr \Big( \inf_{1 \leq k \leq T} \Big|\sum_{j=1}^k s(1-s \cS_\gamma(\cX)/m_0)^{T-j} (\cS_{\gamma}(\cX^j) - \cS_{\gamma,\E})  \Big| > \epsilon \Big) \\
    &\leq \frac{\E \Big(\sum_{j=1}^T s^2(1-s \cS_\gamma(\cX)/m_0)^{2(T-j)} \Big[ (\cS_{\gamma}(\cX^j) - \cS_{\gamma,\E}) \Big]^2\Big)}{\epsilon^2}.
\end{align*}
We can upper bound 
\begin{align*}
    \E \Big[ (\cS_{\gamma}(\cX^j) - \cS_{\gamma,\E}) \Big]^2 &\leq 1.
\end{align*}
In any case, this implies that
\begin{multline*}
    \Pr \Big( \inf_{1 \leq k \leq T} \Big|\sum_{j=1}^k s(1-s \cS_\gamma(\cX)/m_0)^{T-j} (\cS_{\gamma}(\cX^j) - \cS_{\gamma,\E}) d_r(\bV_\star,\bV_j) \Big| > \epsilon \Big)\\
    \leq \frac{m_0^2}{2s \cS_{\gamma,\E} m_0 - (s \cS_{\gamma,\E})^2}\frac{s^2}{\epsilon^2},
\end{multline*}
or
\begin{multline*}
\Pr \Big( \inf_{1 \leq k \leq T} \Big|\sum_{j=1}^k s(1-s \cS_\gamma(\cX)/m_0)^{T-j} (\cS_{\gamma}(\cX^j) - \cS_{\gamma,\E}) d_r(\bV_\star,\bV_j) \Big|\\
> s \frac{m_0}{\sqrt{2s \cS_{\gamma,\E} m_0 - (s \cS_{\gamma,\E})^2}\sqrt{\lambda}} \Big)
\leq \lambda.
\end{multline*}

Putting these together, we find with probability at least $1-2\lambda$,
\[
d_r^2(\bV_\star, \bV_{T}) \leq (1-s\cS_\gamma(\cX)/m_0)^T d_r^2(\bV_\star^\top \bV_0) + s\Big[ \frac{m_0}{\sqrt{2s \cS_{\gamma,\E} m_0 - (s \cS_{\gamma,\E})^2}\sqrt{\lambda}}  +  \frac{m_0 c(\cX)}{\cS_{\gamma,\E}}\Big].
\]

Now, if $T$ is sufficiently large so that $s = \frac{C a }{T^{\nu}}$ satisfies
\begin{align*}
    s &< \frac{m_0}{\cS_\gamma(\cX)} \\
    s &< \frac{a}{2 m_0} \Big[ \frac{1}{\sqrt{2s \cS_{\gamma,\E} m_0 - (s \cS_{\gamma,\E})^2}\sqrt{\lambda}} +  \frac{c(\cX)}{\cS_{\gamma,\E}} \Big]^{-1}.
\end{align*}
To satisfy the second condition, we again need 
\begin{equation}\label{eq:sggdTcond2}
    T = \Omega \Big(\Big[ \frac{m_0}{a} \Big]^{1/\nu}\Big).
\end{equation}

Choosing $T$ to be sufficiently large so that it also satisfies
\begin{equation}\label{eq:sggdTcond3}
    T > \Big[\mathcal{C} \frac{m_0  \log(a/2 d_r^2(\bV_\star^\top \bV_0))}{ a \cS_{\gamma,\E}}\Big]^{1/(1-\nu)} > \frac{ \log(a/2 d_r^2(\bV_\star^\top \bV_0))}{\log(1-s\cS_{\gamma,\E} /m_0)},
\end{equation}
then, with probability at least $1-2\lambda$,
\begin{equation}
    d_r^2(\bV_\star, \bV_T) < a.
\end{equation}

Notice again that the constraints on $T$ given by~\eqref{eq:nggdTcond1},~\eqref{eq:nggdTcond2}, and~\eqref{eq:nggdTcond3} depend only on $a$ and $d_r^2(\bV_\star, \bV_0)$ the ratio $a/d_r^2(\bV_\star, \bV_0)$.
\end{proof}

Finally, we demonstrate how one might hope to have $\cS_{\gamma, \E} > 0$ with a simple model. If the inlier and outlier percentages are $\alpha_1$ and $\alpha_0$ and the inlier and outlier sample covariances are $\bSigma_{\mathrm{in}}$ and $\bSigma_{\mathrm{out}}$, respectively, then 
\begin{equation}
	\E \cS(\cX^k) \geq \E \Big( \gamma \lambda_r\Big(\frac{1}{m} \bXin^k \bXin^{k\top}\Big) - \lambda_1 \Big(\frac{1}{m} \bXout^k \bXout^{k\top} \Big) \Big) = \gamma \alpha_0 \lambda_r(\bSigma_{\mathrm{in}}) - \alpha_1 \lambda_1(\bSigma_{\mathrm{out}}) = \cS_{\gamma, \E}
\end{equation}
Therefore, one could assume that $\gamma \alpha_0 \lambda_r(\bSigma_{\mathrm{in}}) - \alpha_1 \lambda_1(\bSigma_{\mathrm{out}}) > 0$.

\subsubsection{Proof of Theorem~\ref{thm:nsggd}}

	
\begin{proof}

Combining~the results of the previous two theorems yields the result by simultaneously controlling both martingales. Notice that $\cF_3$ will be defined similarly to $\cF_1$ and $\cF_2$.

\end{proof}


\subsubsection{Proof of Theorem~\ref{thm:linconv}}

\begin{proof}
Set $s = \frac{c_1 a}{T^{\nu}}$. For $T_1$ sufficiently large so that the conditions within the theorem hold.
Then, with probability $1-2\lambda$ (or $1-4\lambda$ for dp-SGGD), in $T_1$ iterations, $d_r^2(\bV_\star, \bV_{T_1}) < a$.

Now suppose that we restart with $s' = s/2$ and $\bV_0 = \bV_{T_1}$. Notice that this is equivalent to taking $a' = a/2$ and starting distance $d_r^2(\bV_\star, \bV_0') = d_r^2(\bV_\star, \bV_T) < a$. In particular, this takes $T'>\cF_{\bullet}(1/2)$ iterations to reach $\bV_{T + T'}$ such that $d_r^2(\bV_\star, \bV_{T + T'}) < a/2$. Repeating this procedure for $r$ restarts every $\cF_{\bullet}(1/2)$ iterations yields the desired result.

\end{proof}

This section gives additional plots demonstrating the performance of the various differentially private methods we discuss.

\begin{figure}
     \centering
         \includegraphics[width=.28\textwidth]{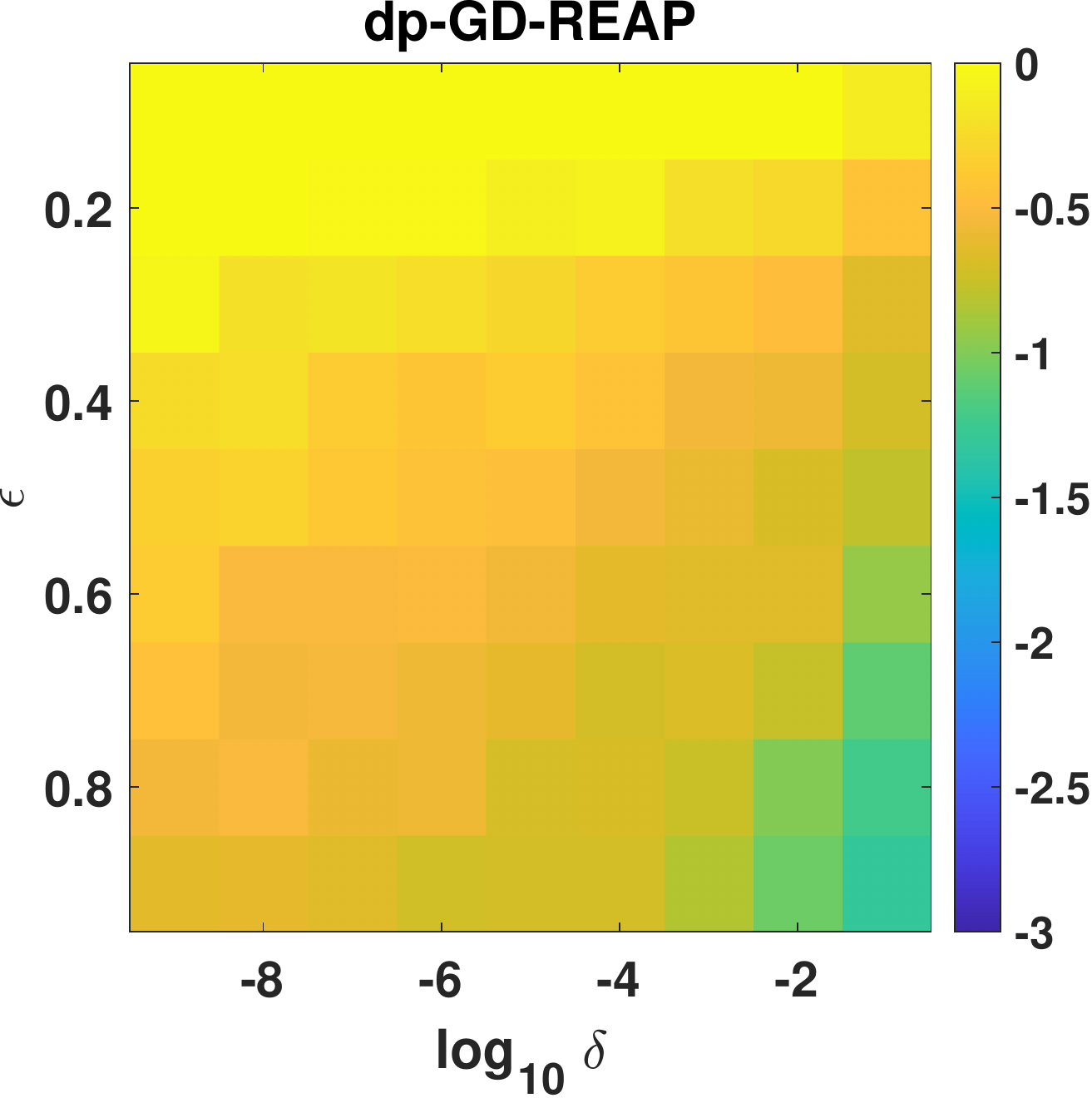}
         \includegraphics[width=.28\textwidth]{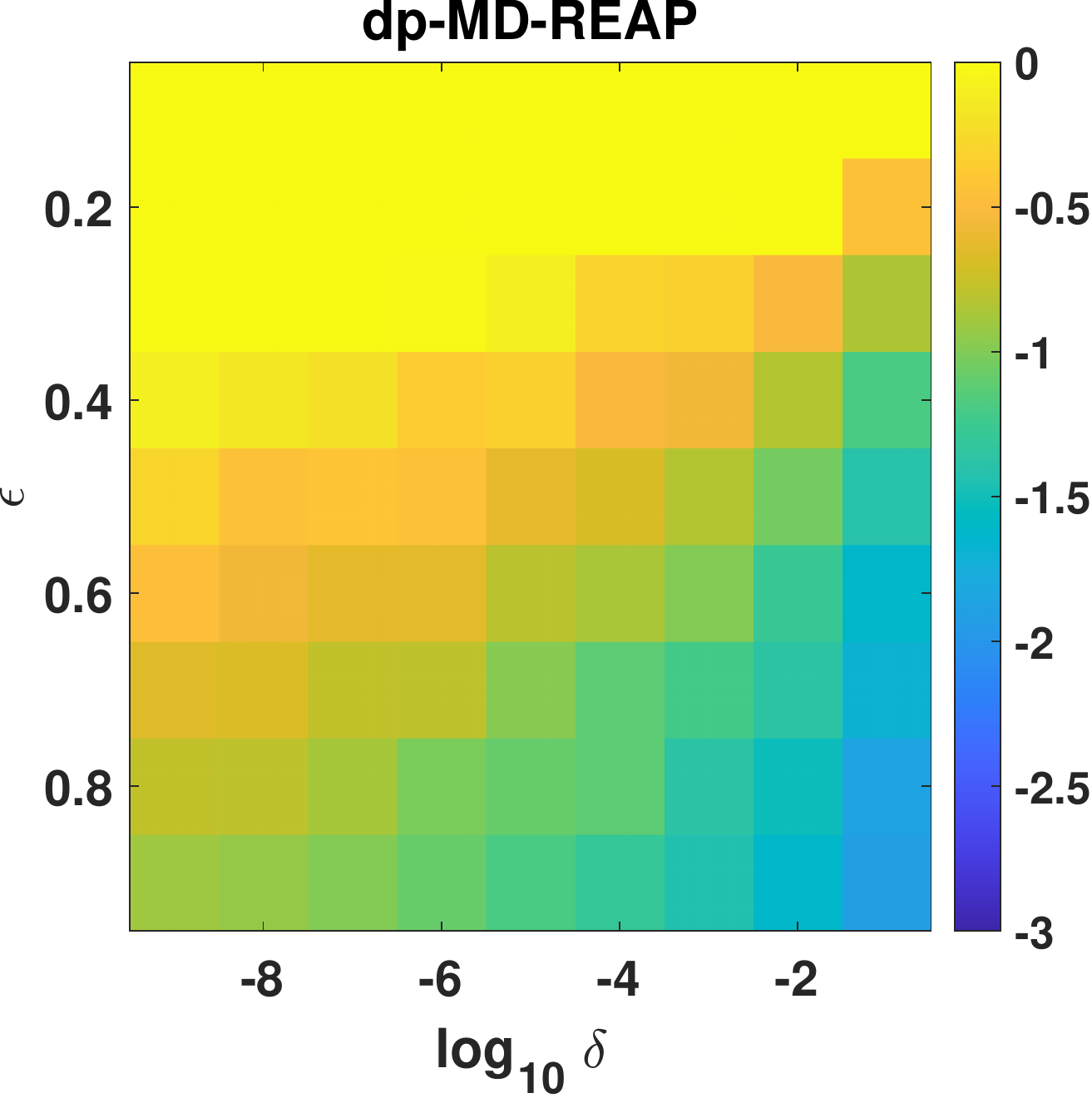}
         \includegraphics[width=.28\textwidth]{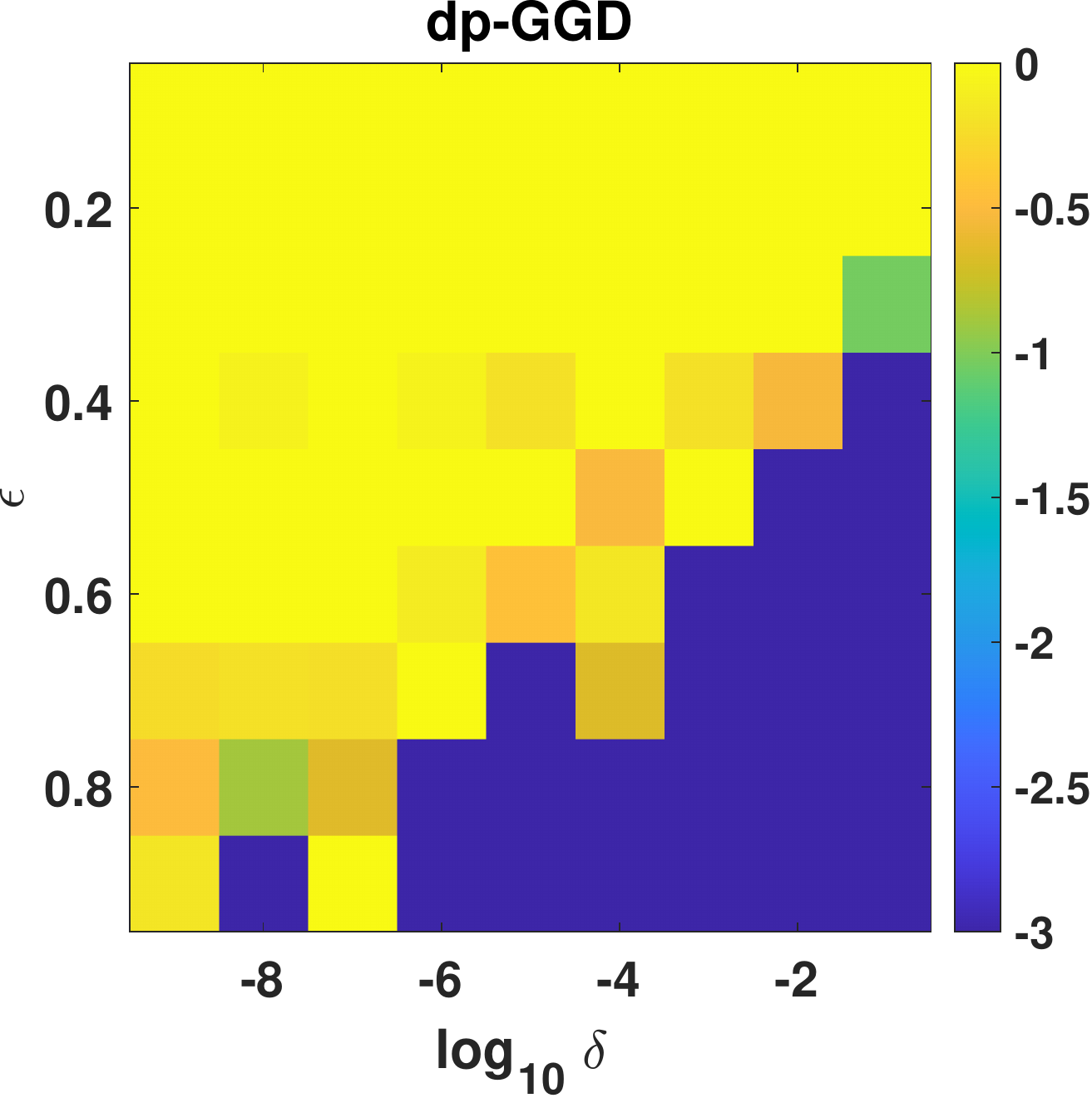}
         \includegraphics[width=.28\textwidth]{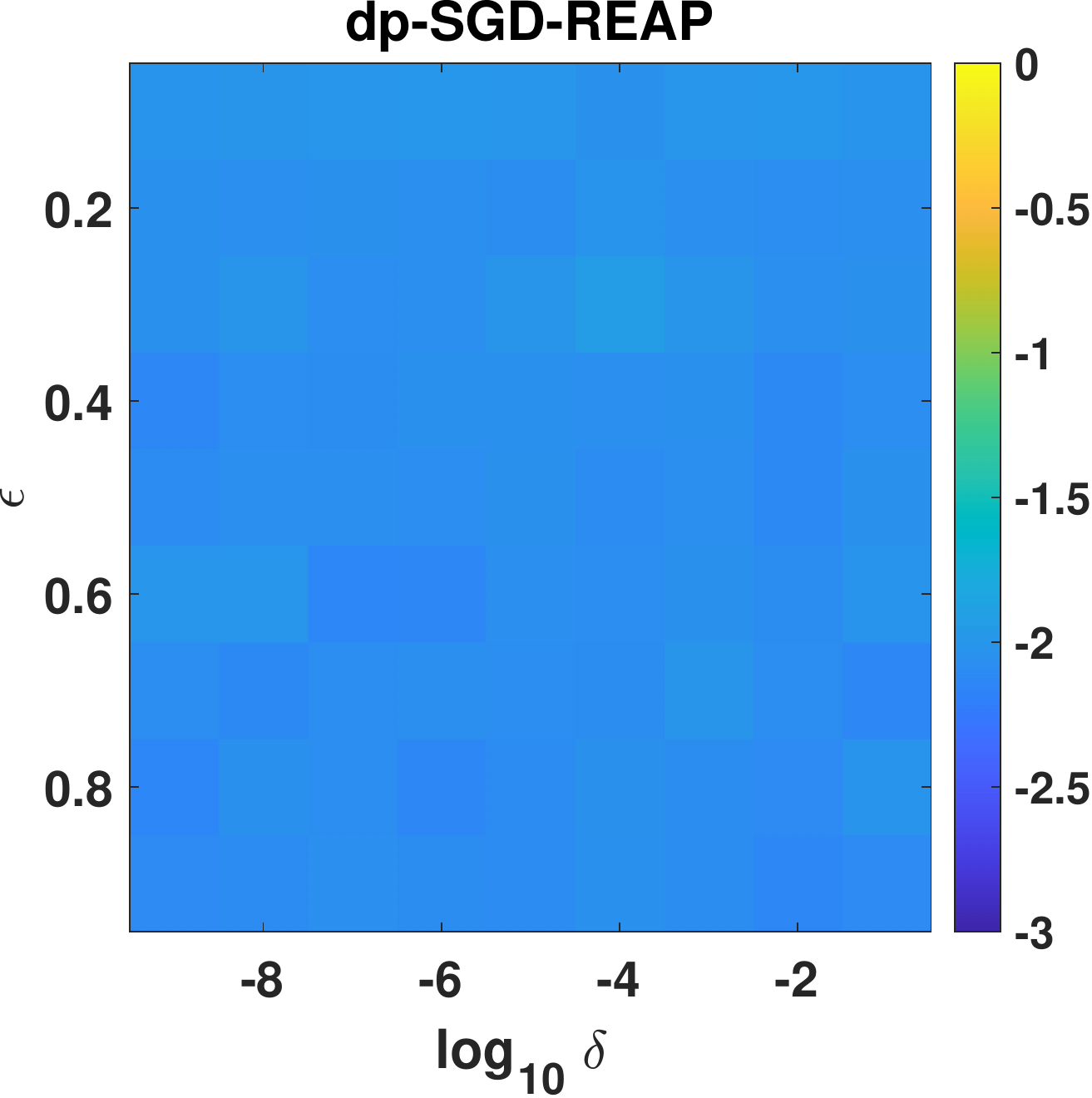}
         \includegraphics[width=.28\textwidth]{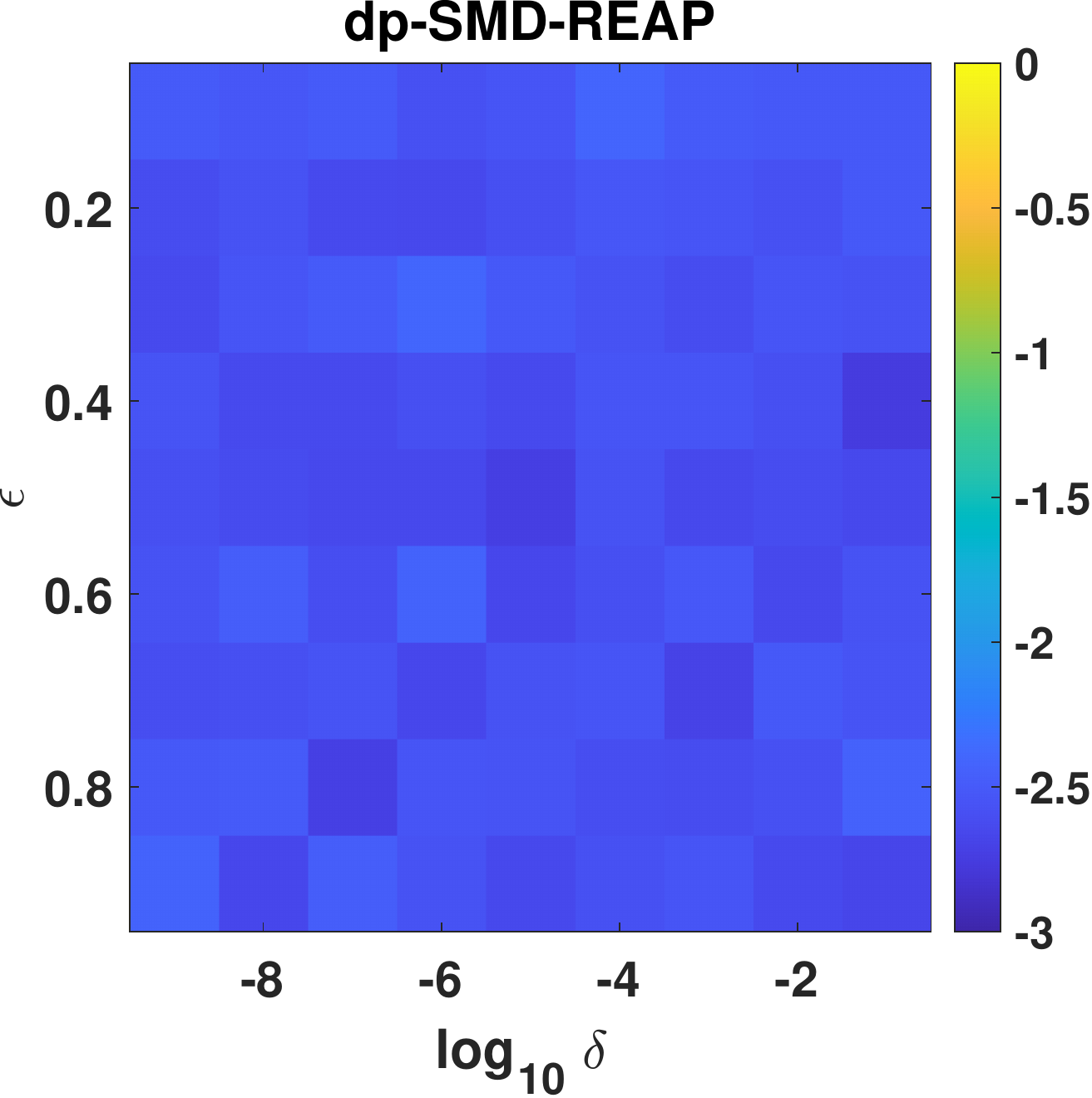}
         \includegraphics[width=.28\textwidth]{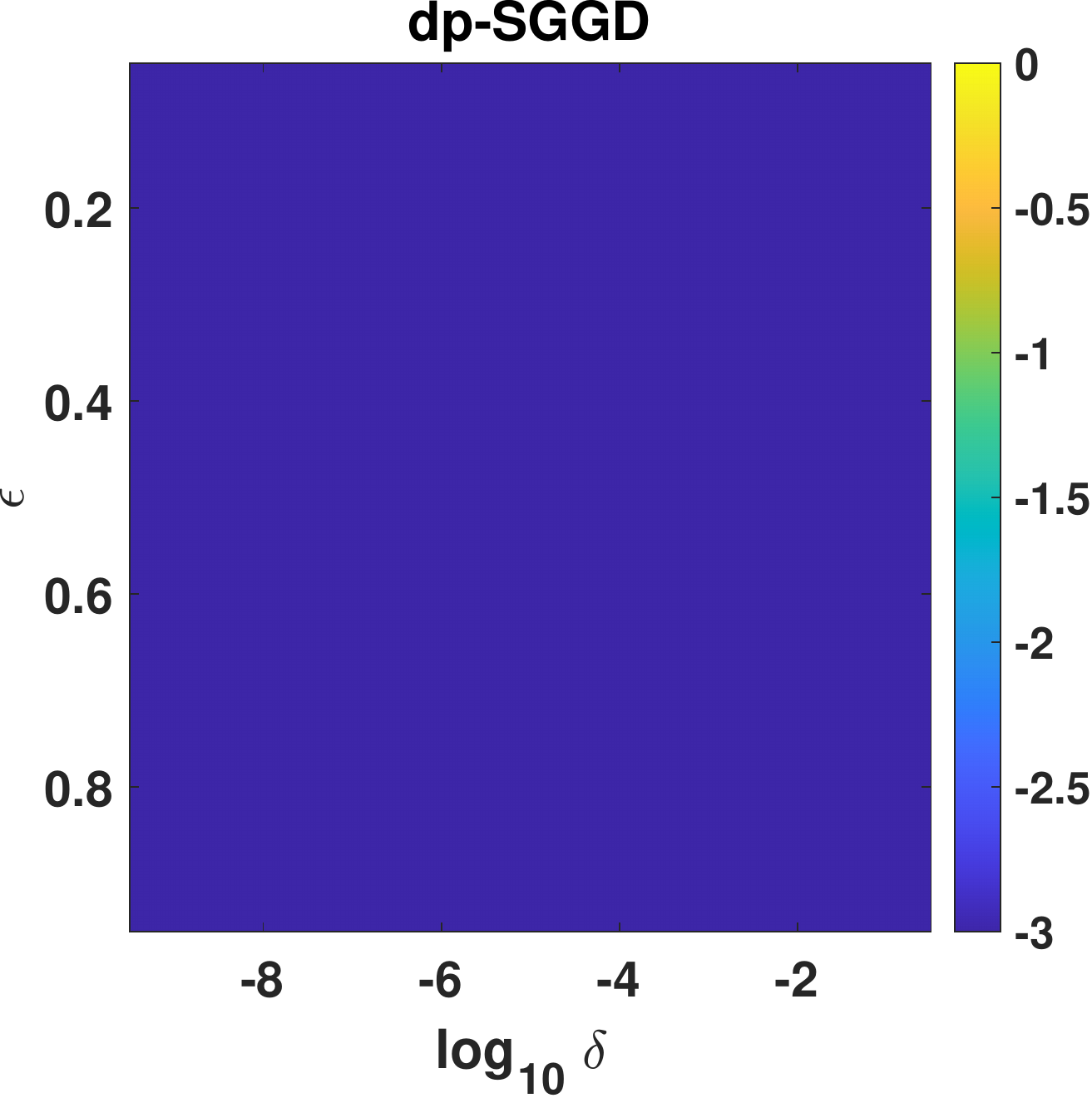}
        \caption{$\delta$ and $\epsilon$ phase transition plot. Each square represents $\log_{10}$ final error.}\label{fig:delta_eps}
\end{figure}
    
First, we give a phase transition plot with respect to $\epsilon$ and $\delta$. \autoref{fig:delta_eps} shows that dp-GGD's transition from small $\epsilon$ and $\delta$ is more abrupt than that of dp-GD-REAP and dp-MD-REAP. The data parameters are as follows, the inlier dimension $r=2$, total dimension $D=20$, number of points $N=2,000$ and inlier ratio $= 0.5$. The algorithms' parameters are as follows, total number of iterations of each algorithm are the same as the number of data points $N$, the step size for the four dp-REAP algorithms to be $\eta_k = 8/\sqrt{k}$, the step size for dp-GGD and dp-SGGD is $\eta_k = 1 / 2^{\lfloor k/50 \rfloor}$. The experiment is repeated 50 times and the medium error is plotted.
    
Next, we plot the ratio of successful attempts to converge to tolerance $10^{-2}$ in 50 repetitions of each algorithm as a function of inlier ratio and batch size in  \autoref{fig:phase_transition_inlier_ratio_batchsize}. (plot inlier ratio vs N) The data parameters are as follows, number of points $N=2,000$, total dimension $D = 20$, and inlier dimension $r=2$. The algorithms' parameters are as follows, total number of iterations of each algorithm are the same as the number of points $N$, the step size for the four dp-REAP algorithms to be $\eta_k = 8/\sqrt{k}$, the step size for dp-GGD and dp-SGGD is $\eta_k = 1 / 2^{\lfloor k/50 \rfloor}$. The plot shows that dp-SGD-REAP converges in the regime where inlier ratio is $\geq 0.6$, dp-SMD-REAP converges in the regime where inlier ratio is $\geq 0.2$, and dp-SGGD converges almost for all inlier ratios when the batch size is greater than 2.
    
\begin{figure}
     \centering
         \includegraphics[width=.28\textwidth]{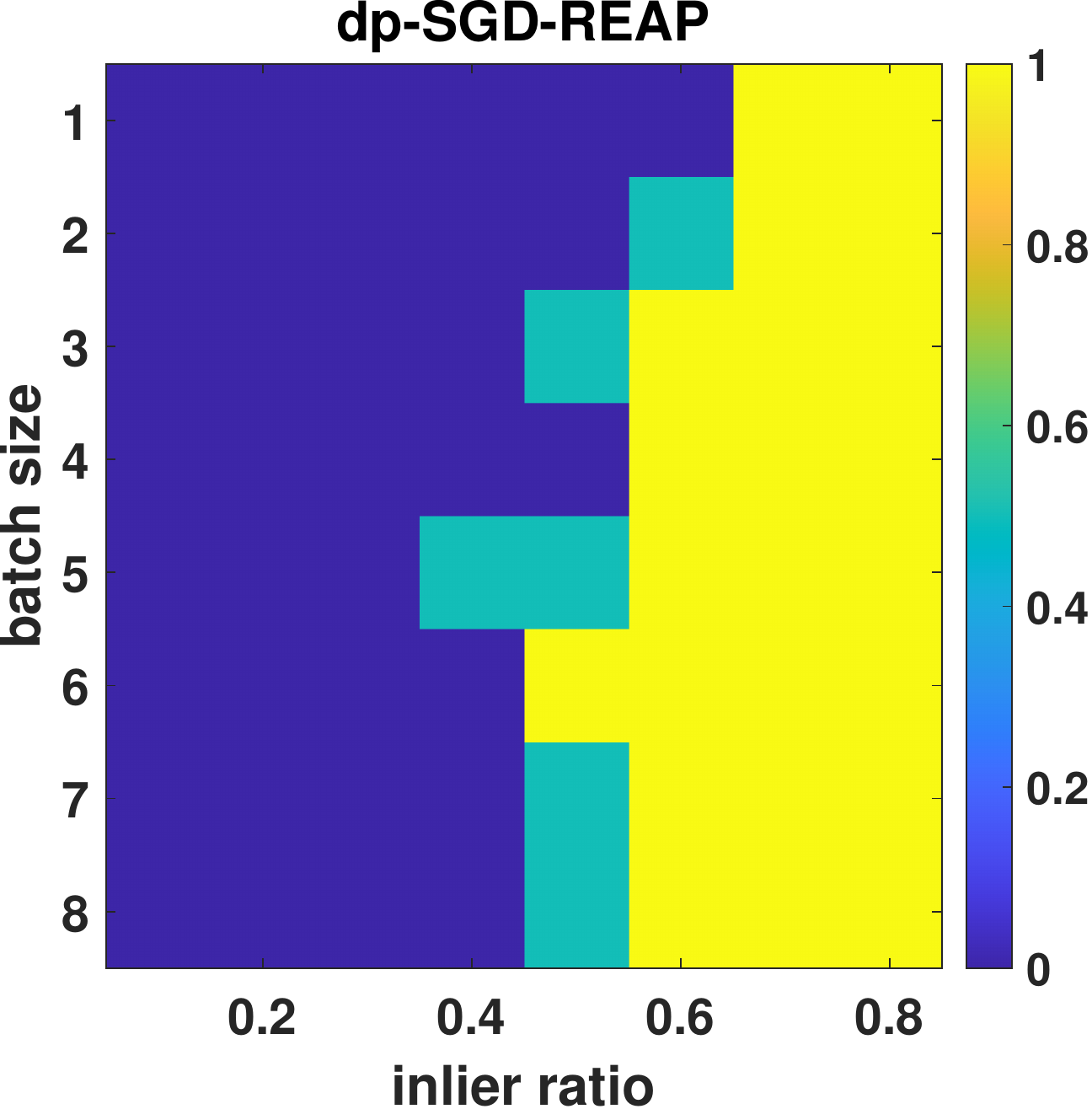}
         \includegraphics[width=.28\textwidth]{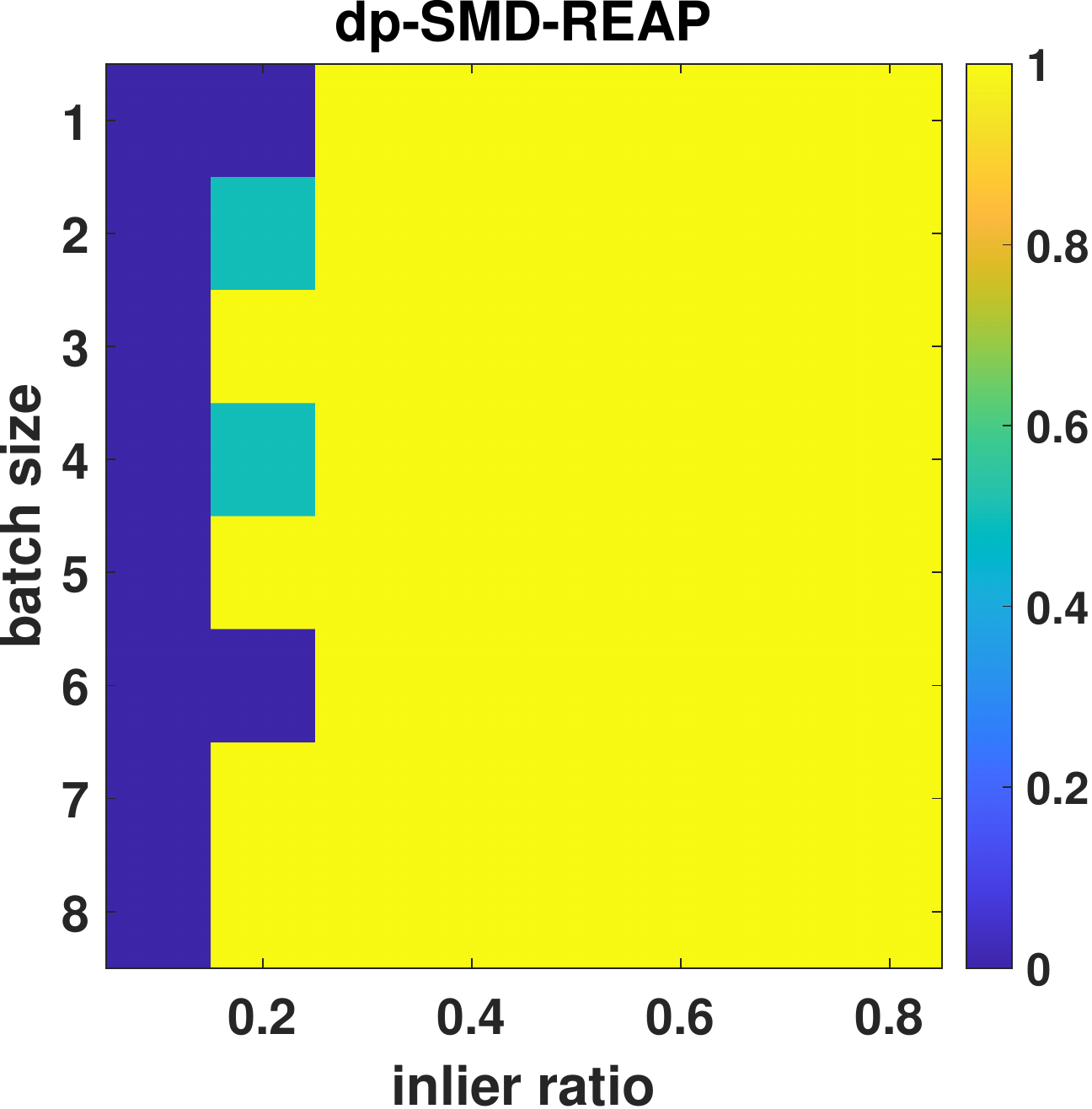}
         \includegraphics[width=.28\textwidth]{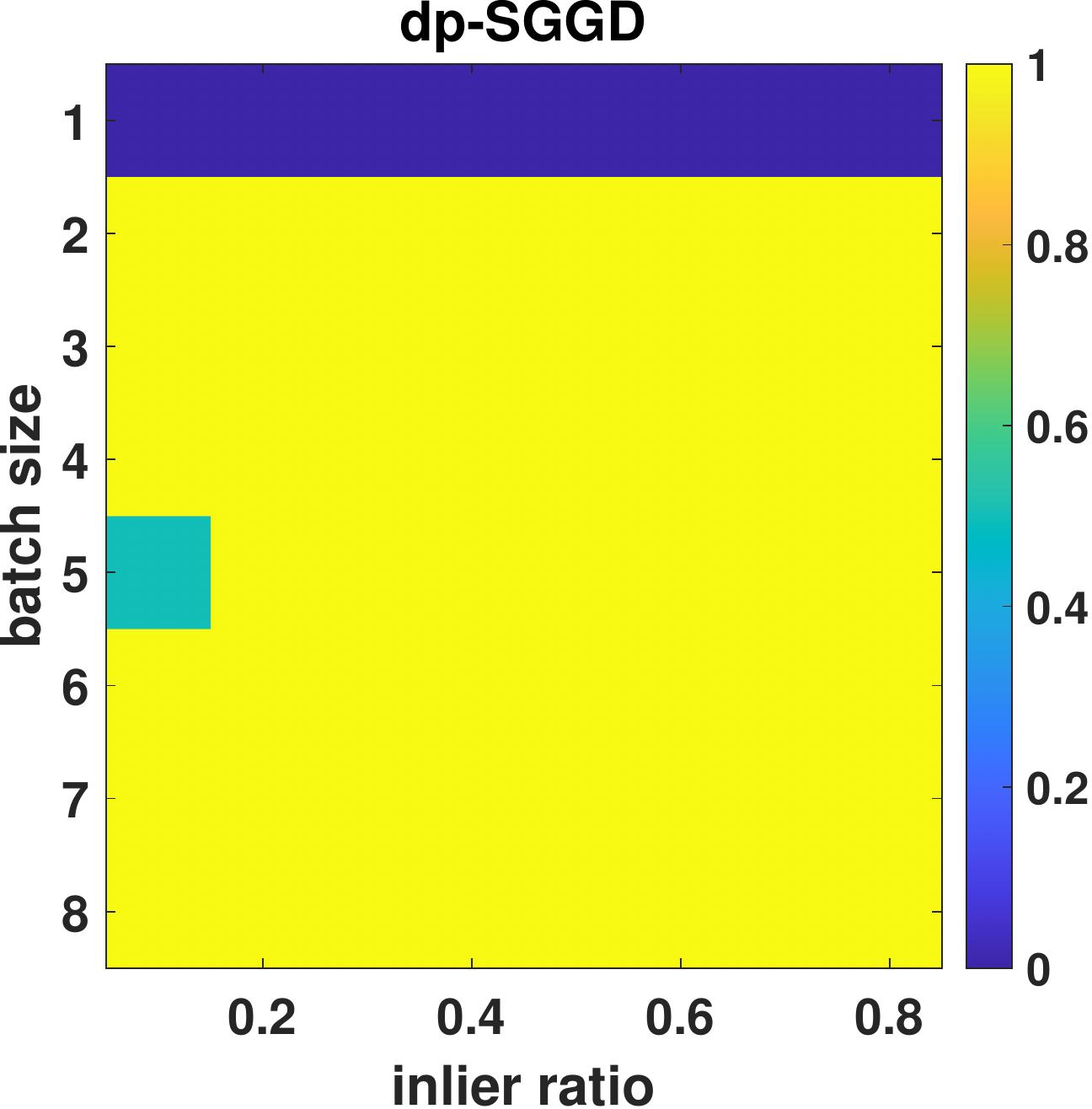}
        \caption{Inlier ratio and batch size transition plot. Each square represents the percentage of time that the algorithm achieves the given tolerance for a given combination of inlier ratio and batch size.}
        \label{fig:phase_transition_inlier_ratio_batchsize}
\end{figure}

We plot the percentage of repetitions that algorithms converge to tolerance $10^{-2}$ in \autoref{fig:phase_transition_inlier_ratio_batchsize}.

 We also give a phase transition plot of $D$ versus $r$, where the value is the final $\log_{10}$ error. \autoref{fig:D_d} shows that dp-GD-REAP, dp-MD-REAP and dp-GGD work well in the regime where both $D$ and $r$ are small. The data parameters are as follows, number of points $N=2,000$ and inlier ratio $= 0.5$. The algorithms' parameters are as follows, total number of iterations of each algorithm are the same as the number of points $N$, the step size for the four dp-REAP algorithms to be $\eta_k = 8/\sqrt{k}$, the step size for dp-GGD and dp-SGGD is $\eta_k = 1 / 2^{\lfloor k/50 \rfloor}$. The experiment is repeated 50 times and the medium error of all repetitions is plotted. 
    
\begin{figure}[h!]
     \centering
     \includegraphics[width=.28\textwidth]{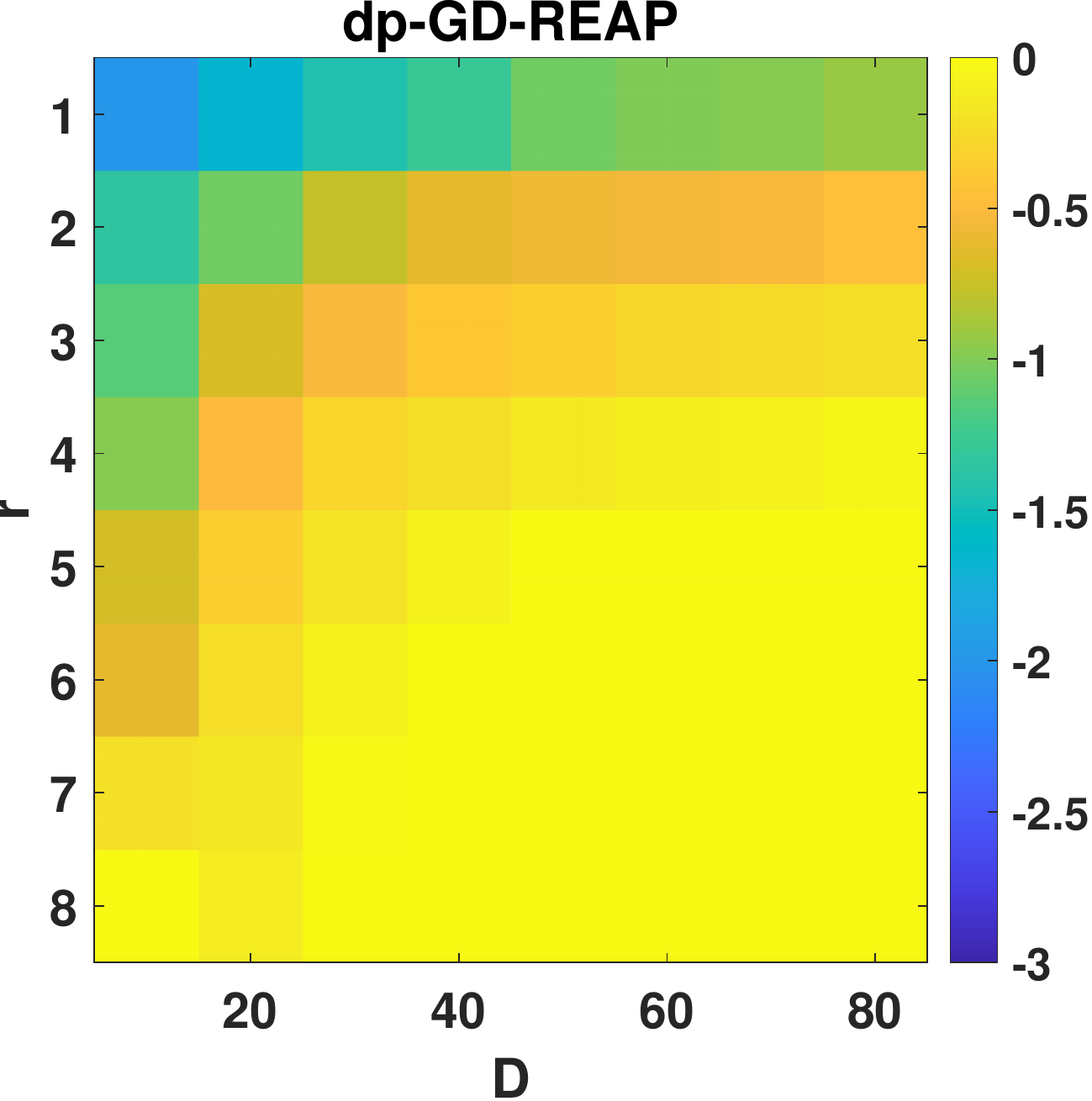}
     \includegraphics[width=.28\textwidth]{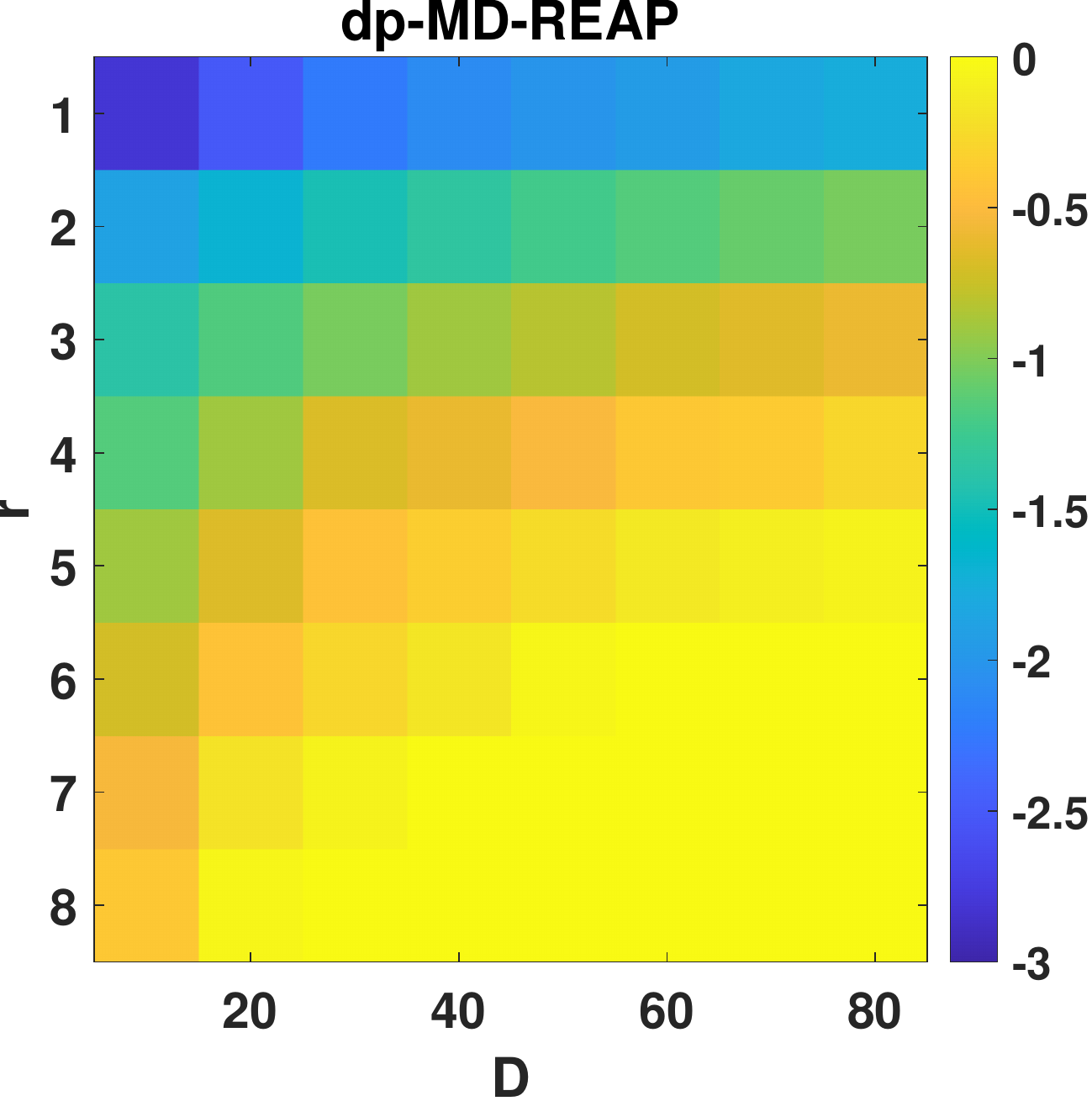}
    \includegraphics[width=.28\textwidth]{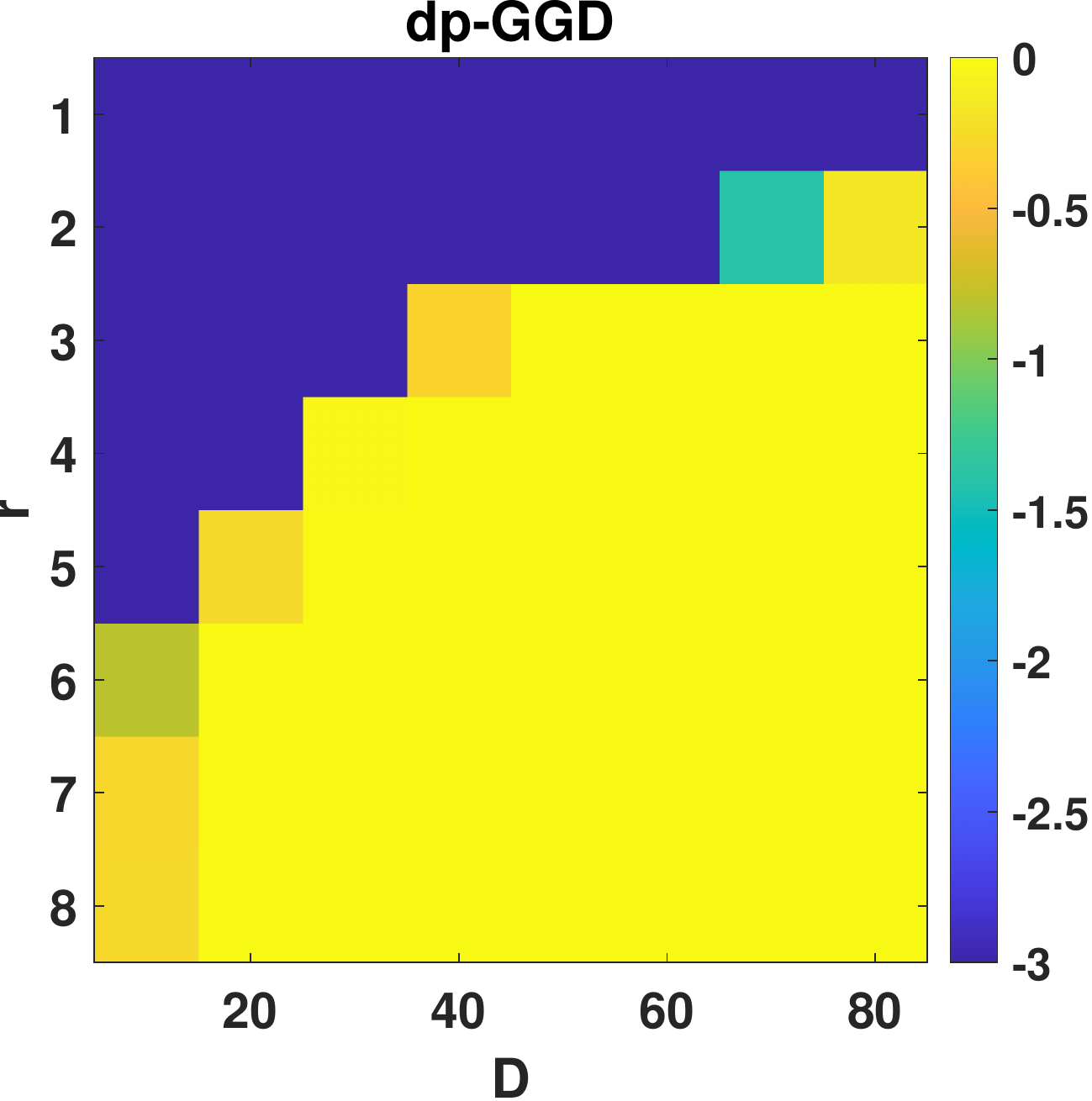}
    \includegraphics[width=.28\textwidth]{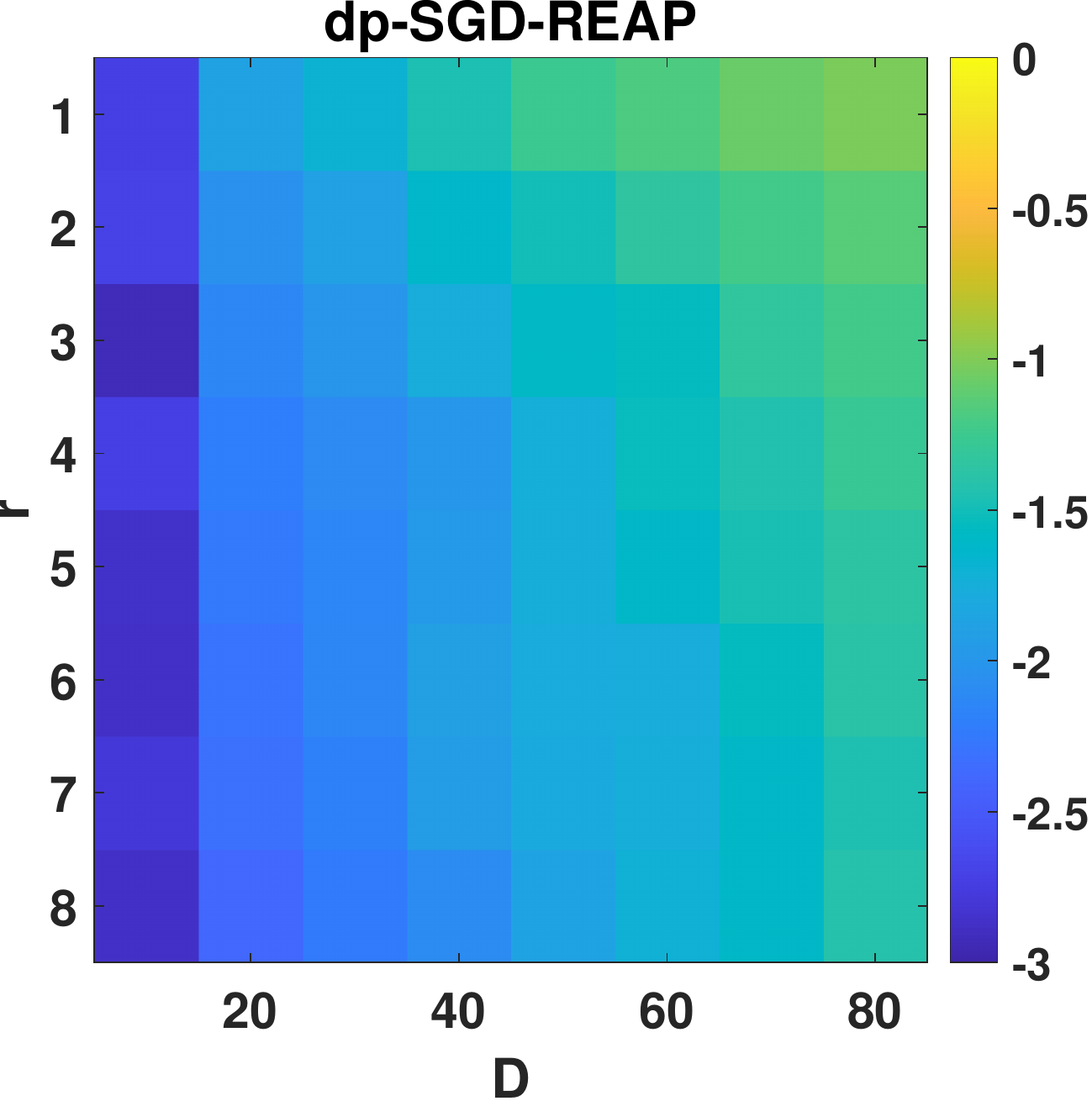}
     \includegraphics[width=.28\textwidth]{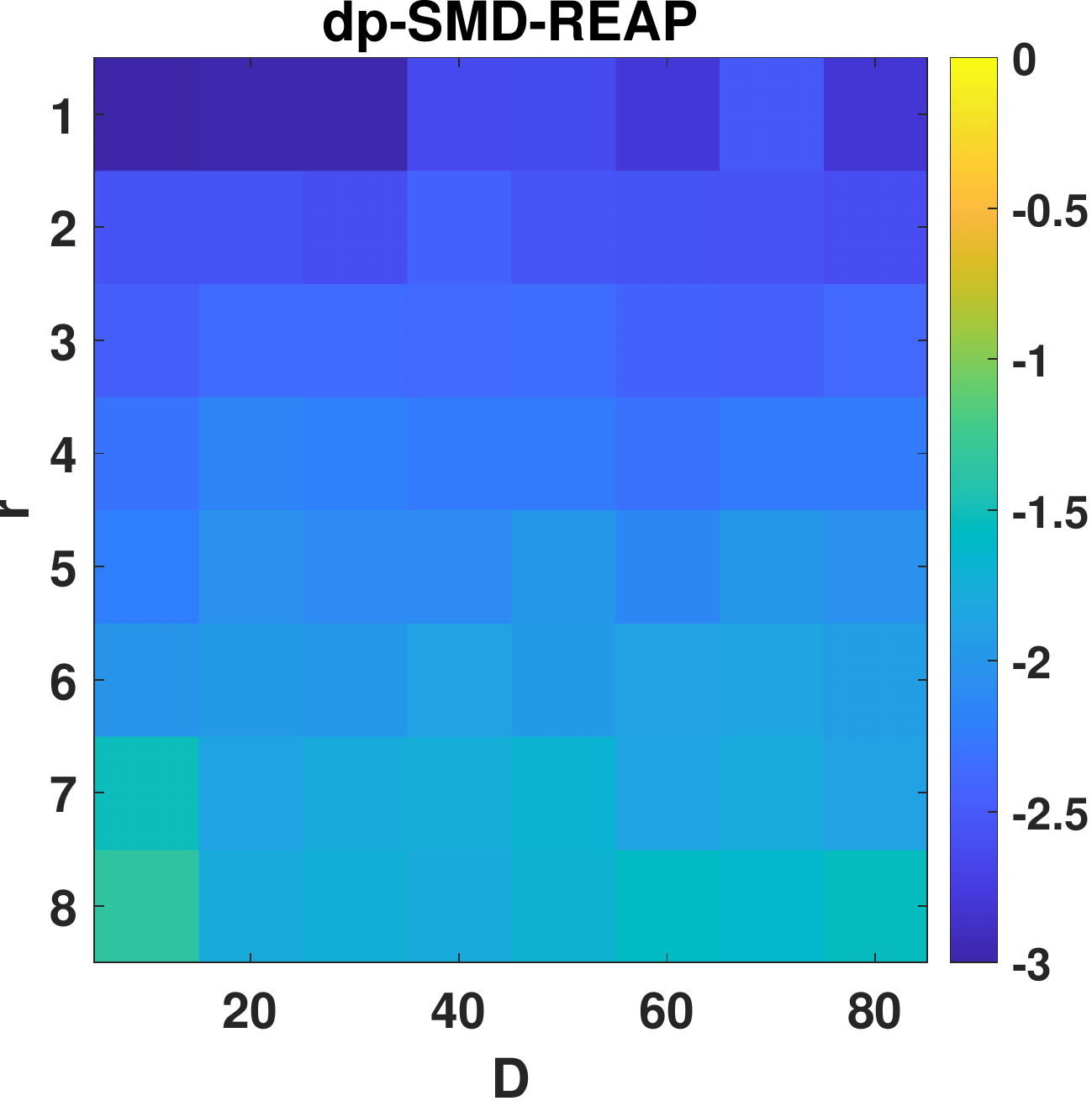}
    \includegraphics[width=.28\textwidth]{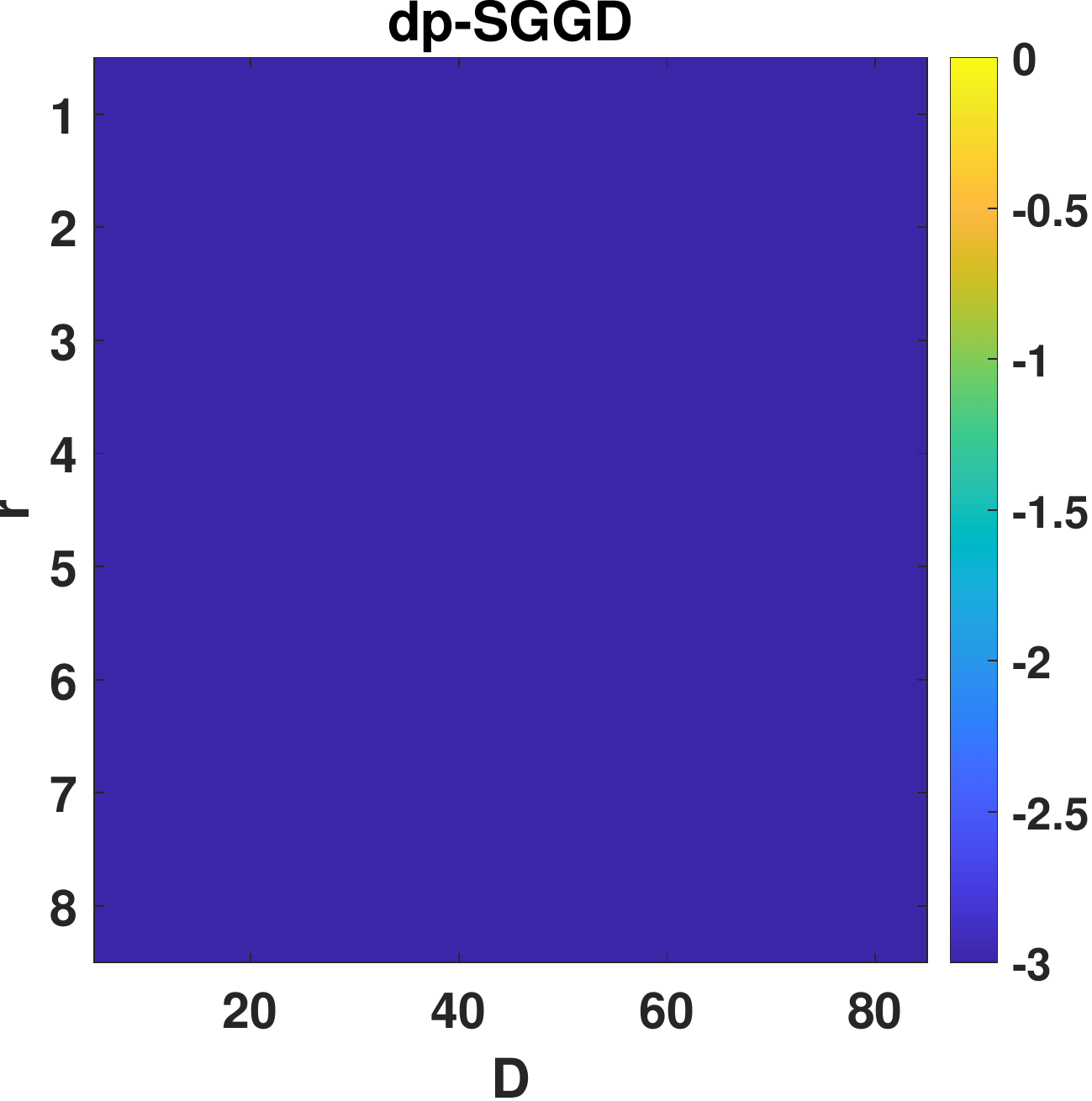}
    \caption{$D$ and $r$ phase transition plot. Each square represents $\log_{10}$ final error. The $x$-axis correspond to different values of $D$ and the $y$-axis correspond to different values of $r$.}\label{fig:D_d}
\end{figure}

Next, we give a phase transition plot of $N$ versus $D$ with $\log_{10}$ error in \autoref{fig:N_D}, and time to error = $10^{-2}$ in \autoref{fig:N_D_time}. We plot $\log_{10}$ final error, and time to converge to tolerance (\autoref{fig:N_D_time}), and ratio of failed attempt to reach tolerance in 50 repetitions (\autoref{fig:N_D_conv}) as a function of $N$ and $D$. In the event that none of the repetitions successfully reaches tolerance, the square shows yellow in \autoref{fig:N_D_time}, this corresponds with number of failed attempts in \autoref{fig:N_D_conv}. \autoref{fig:N_D} shows that the dp-GD-REAP, dp-MD-REAP and dp-GD work well in the regime where $D$ is small and $N$ is large. The data parameters are as follows, the inlier dimension $r=2$, and the inlier ratio $= 0.5$. The algorithms' parameters are as follows, total number of iterations of each algorithm are the same twice the number of points $2N$, the step size for the four dp-REAP algorithms to be $\eta_k = 8/\sqrt{k}$, the step size for dp-GGD and dp-SGGD is $\eta_k = 1 / 2^{\lfloor k/50 \rfloor}$. Each algorithm is run repetitions of 50 times.

\begin{figure}[h!]
     \centering
         \includegraphics[width=.28\textwidth]{experiment_figures/N_D1.pdf}
         \includegraphics[width=.28\textwidth]{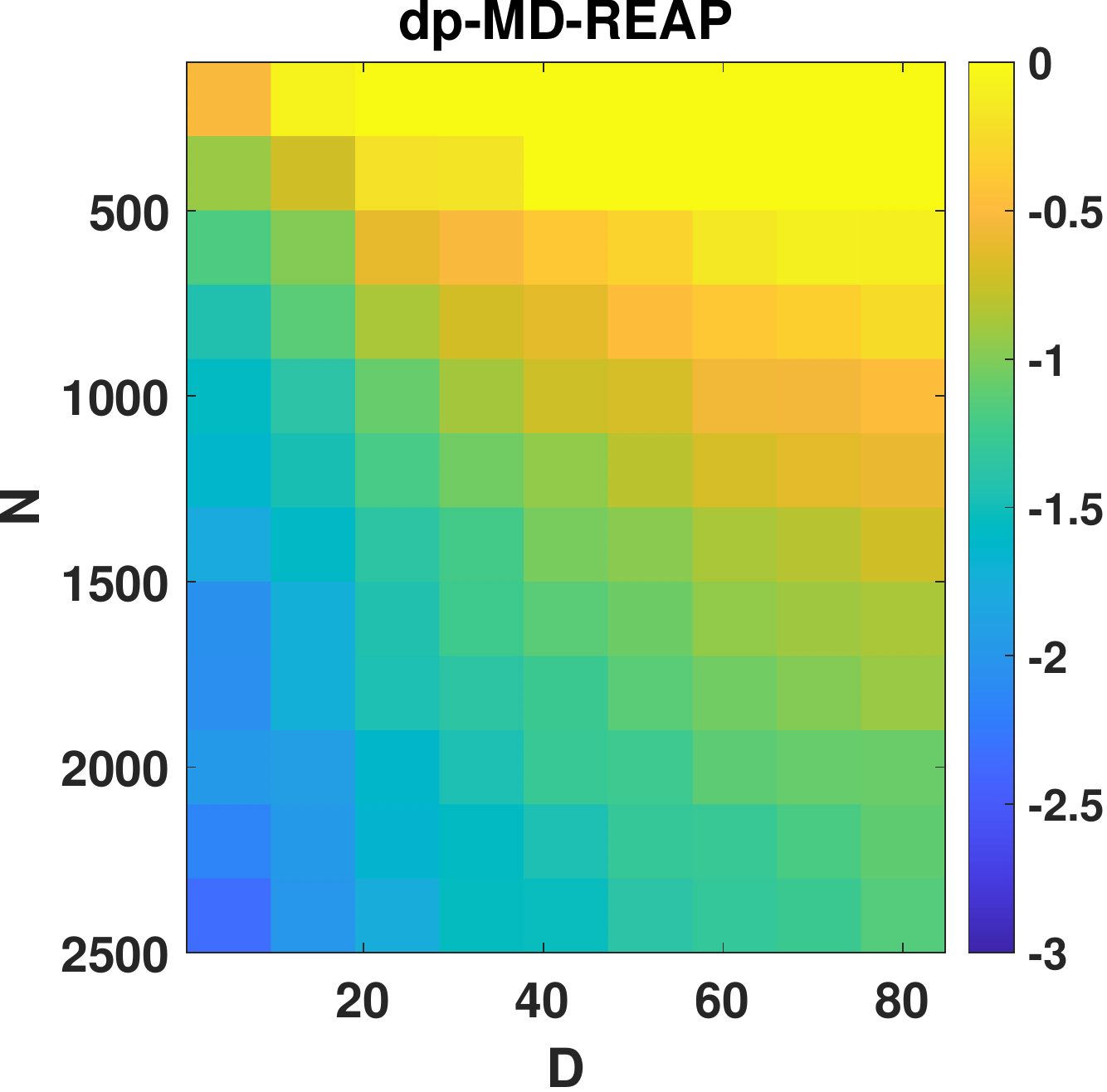}
         \includegraphics[width=.28\textwidth]{experiment_figures/N_D3.pdf}
         \includegraphics[width=.28\textwidth]{experiment_figures/N_D4.pdf}
         \includegraphics[width=.28\textwidth]{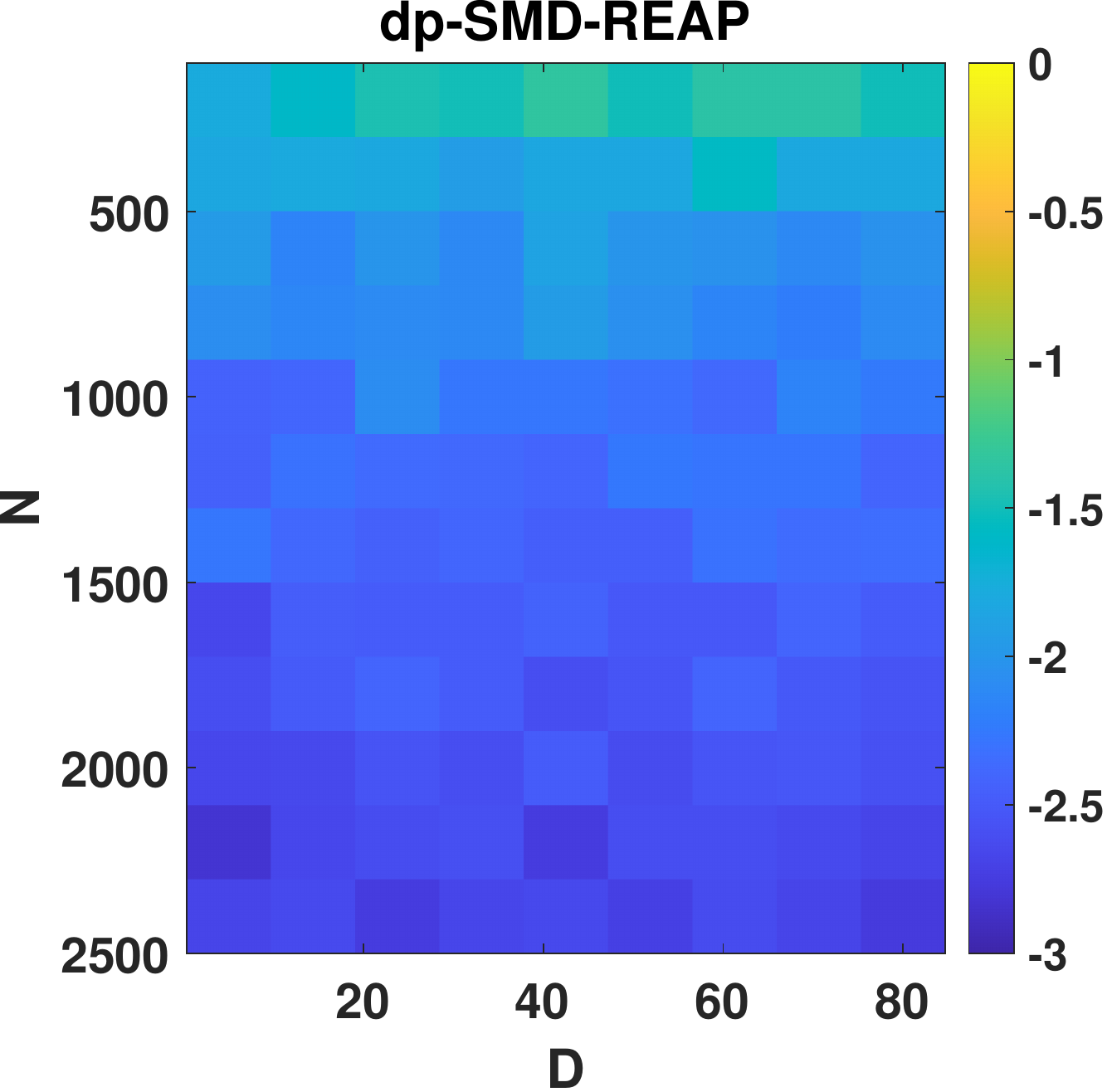}
         \includegraphics[width=.28\textwidth]{experiment_figures/N_D6.pdf}
        \caption{$D$ and $N$ phase transition plot. Each square represents $\log_{10}$ final error.}\label{fig:N_D}
\end{figure}

\begin{figure}[h!]
     \centering
         \includegraphics[width=.28\textwidth]{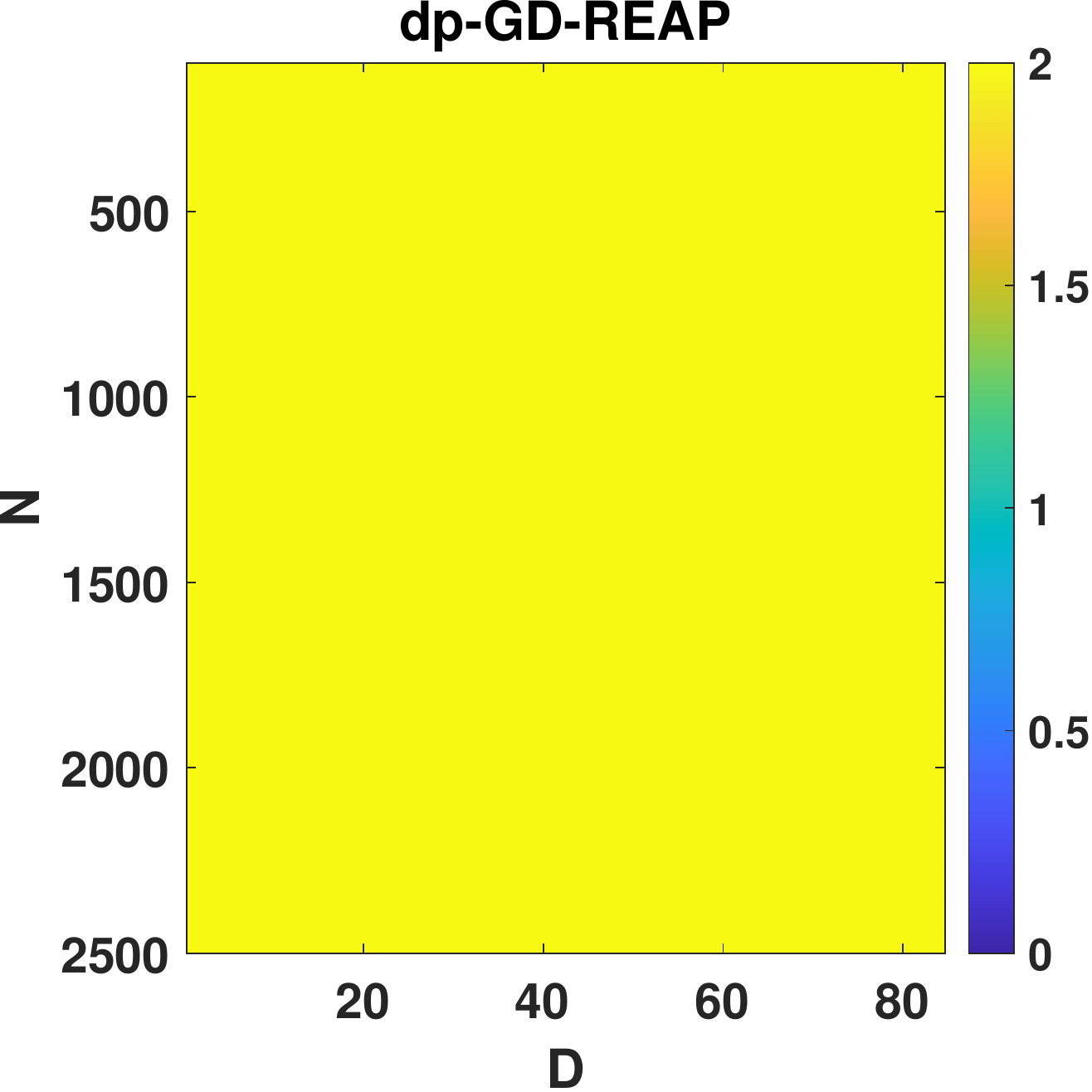}
         \includegraphics[width=.28\textwidth]{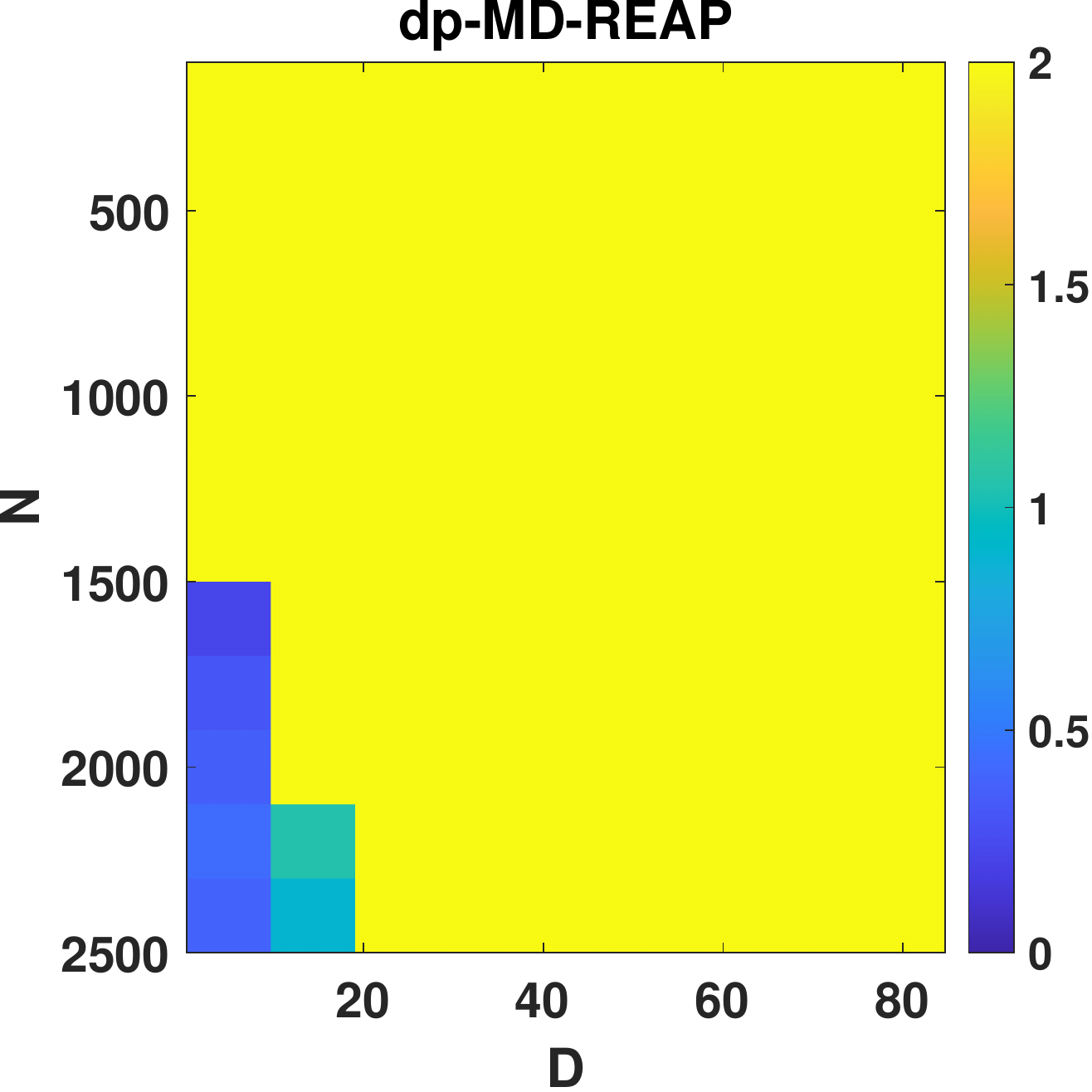}
         \includegraphics[width=.28\textwidth]{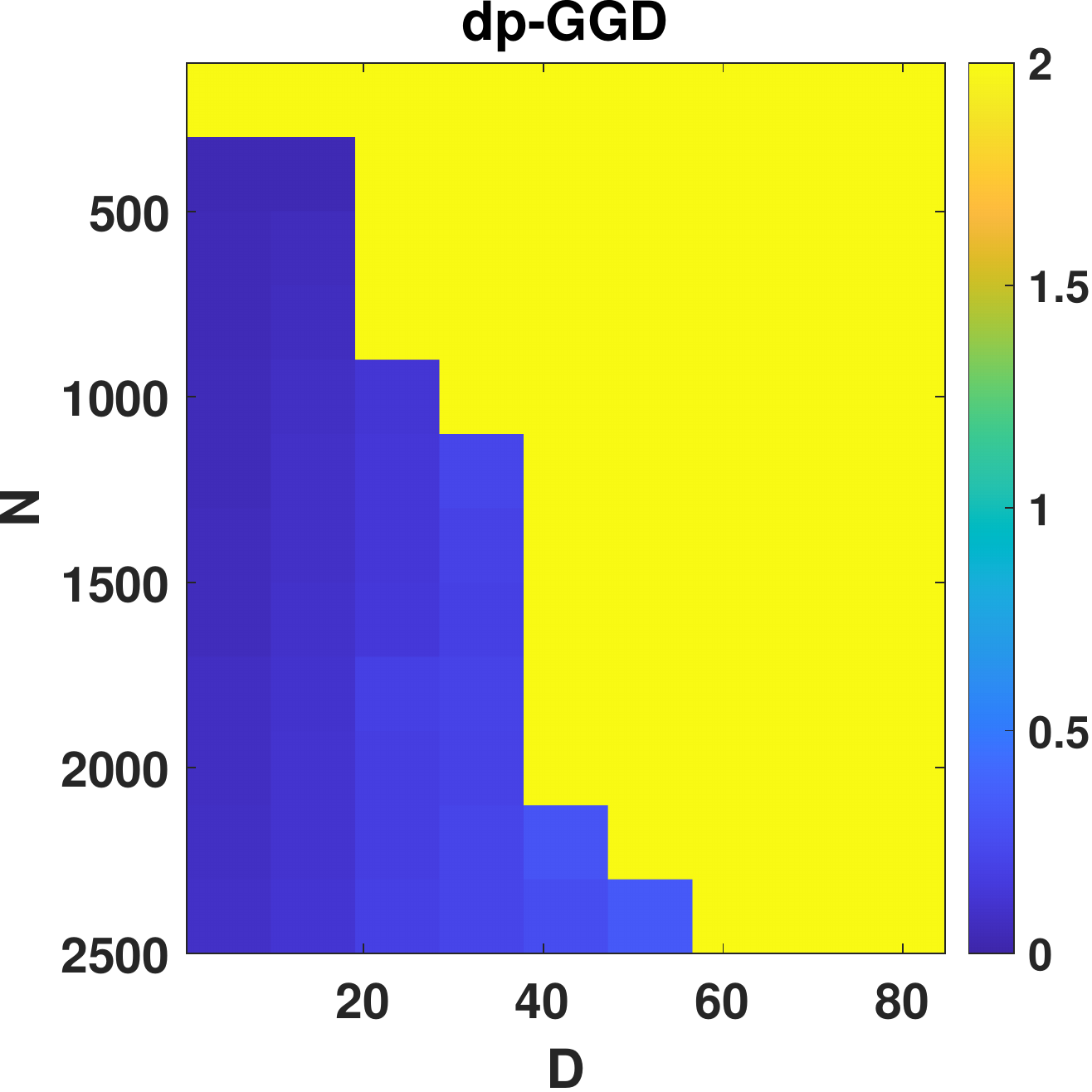}
         \includegraphics[width=.28\textwidth]{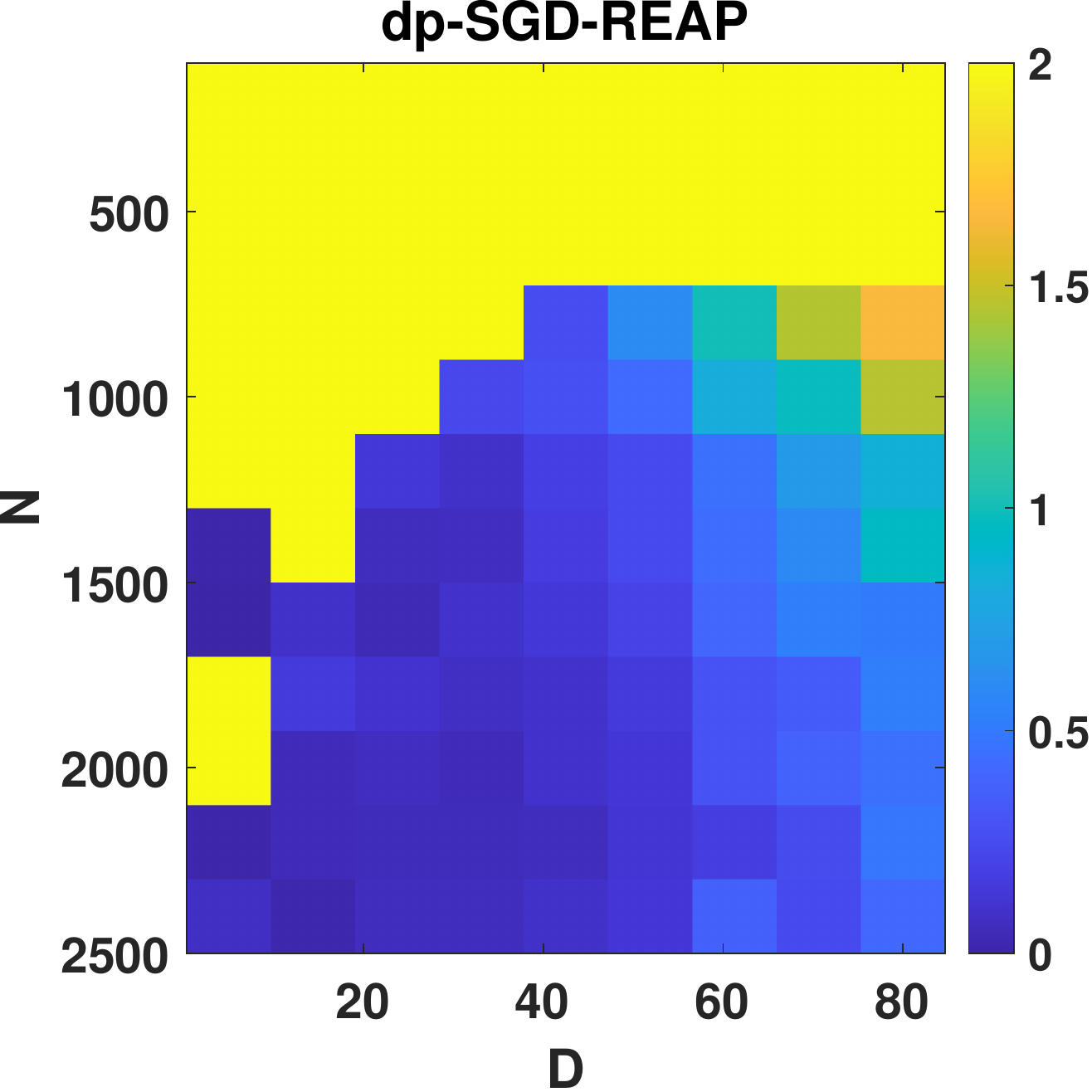}
         \includegraphics[width=.28\textwidth]{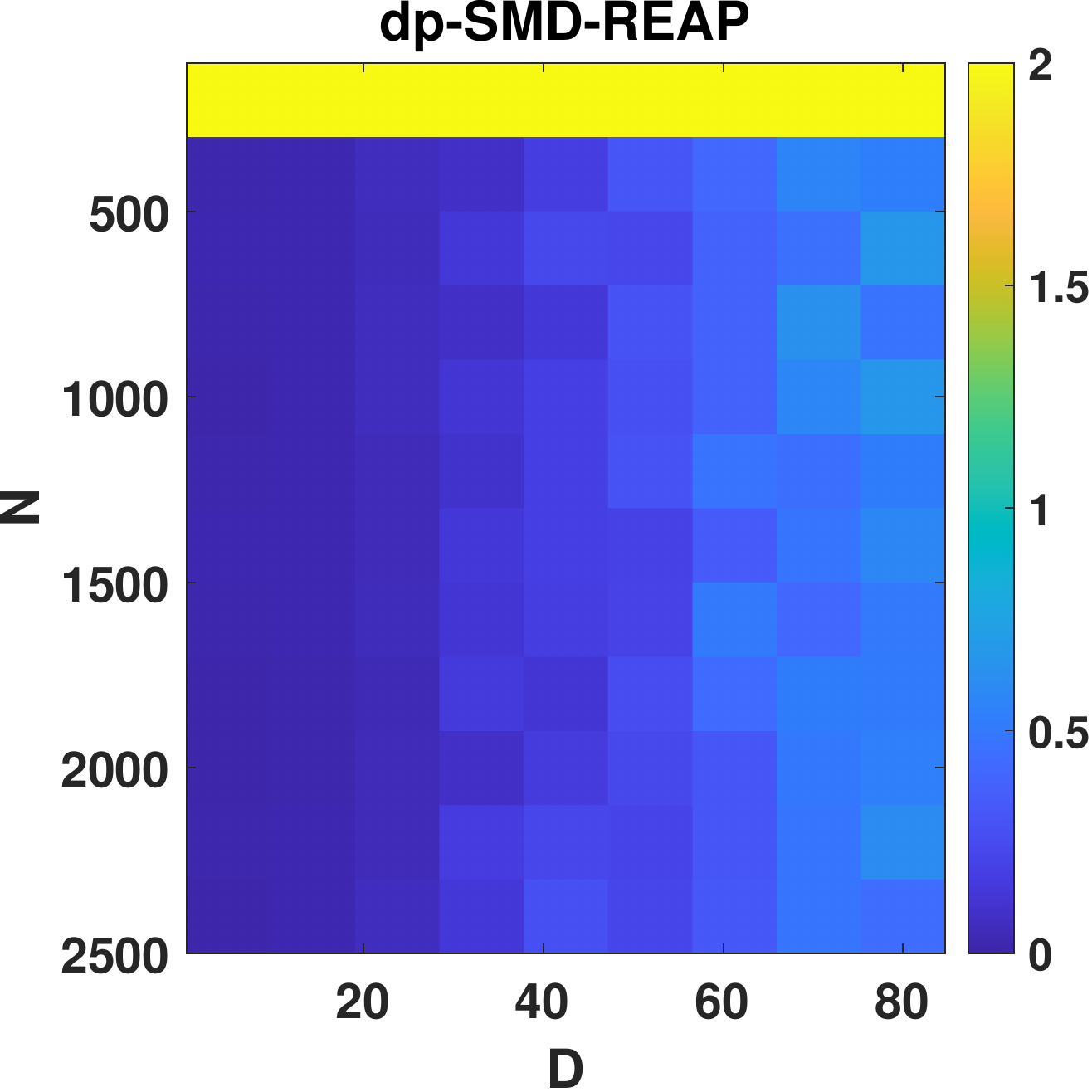}
         \includegraphics[width=.28\textwidth]{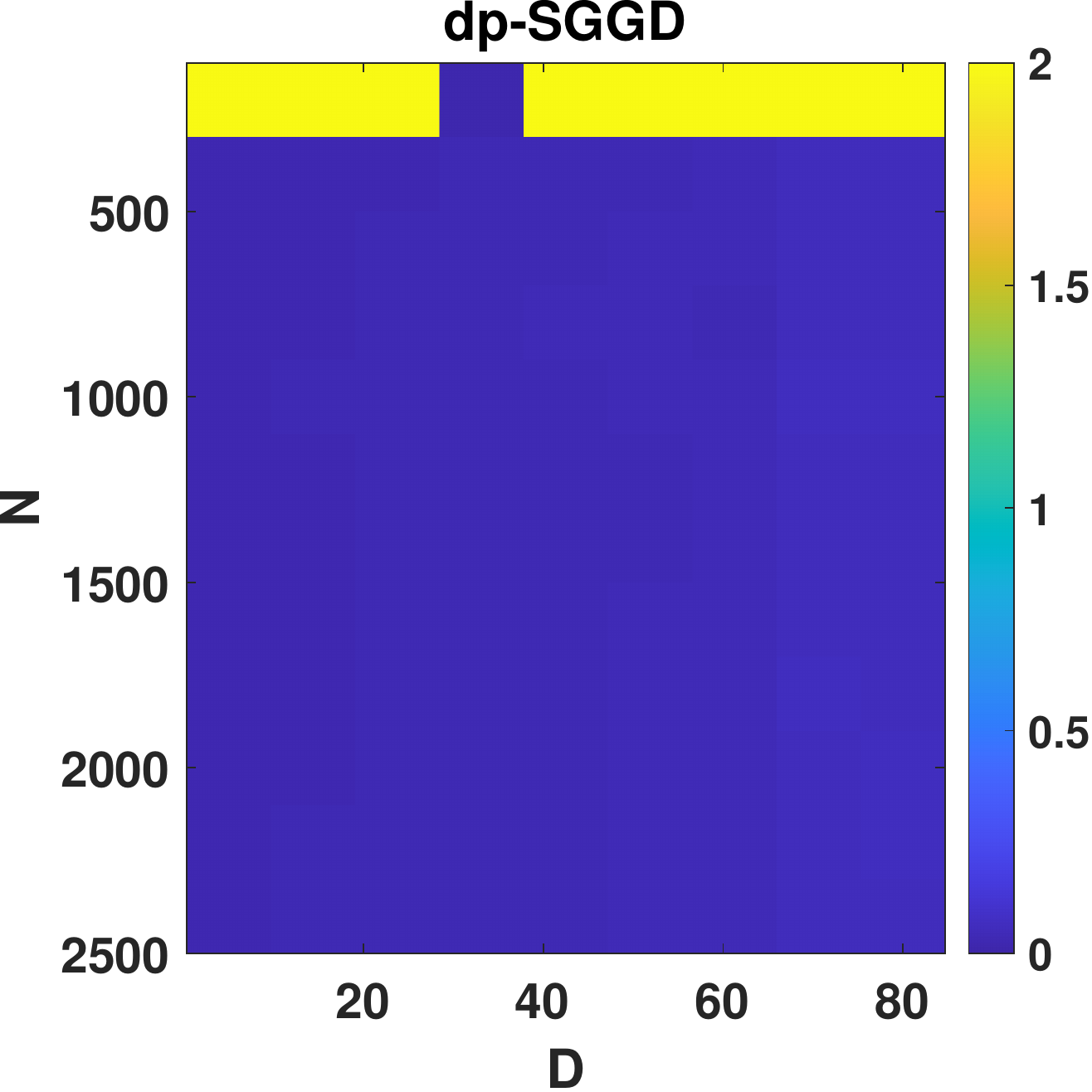}
        \caption{$D$ and $N$ phase transition plot. Each square represents the time to reach tolerance $10^{-2}$. If in all repetitions the algorithms failed to converge, the square is shown in yellow (I imputed a large number, 2 in this case). }\label{fig:N_D_time}
\end{figure}

\begin{figure}[h!]
     \centering
         \includegraphics[width=.28\textwidth]{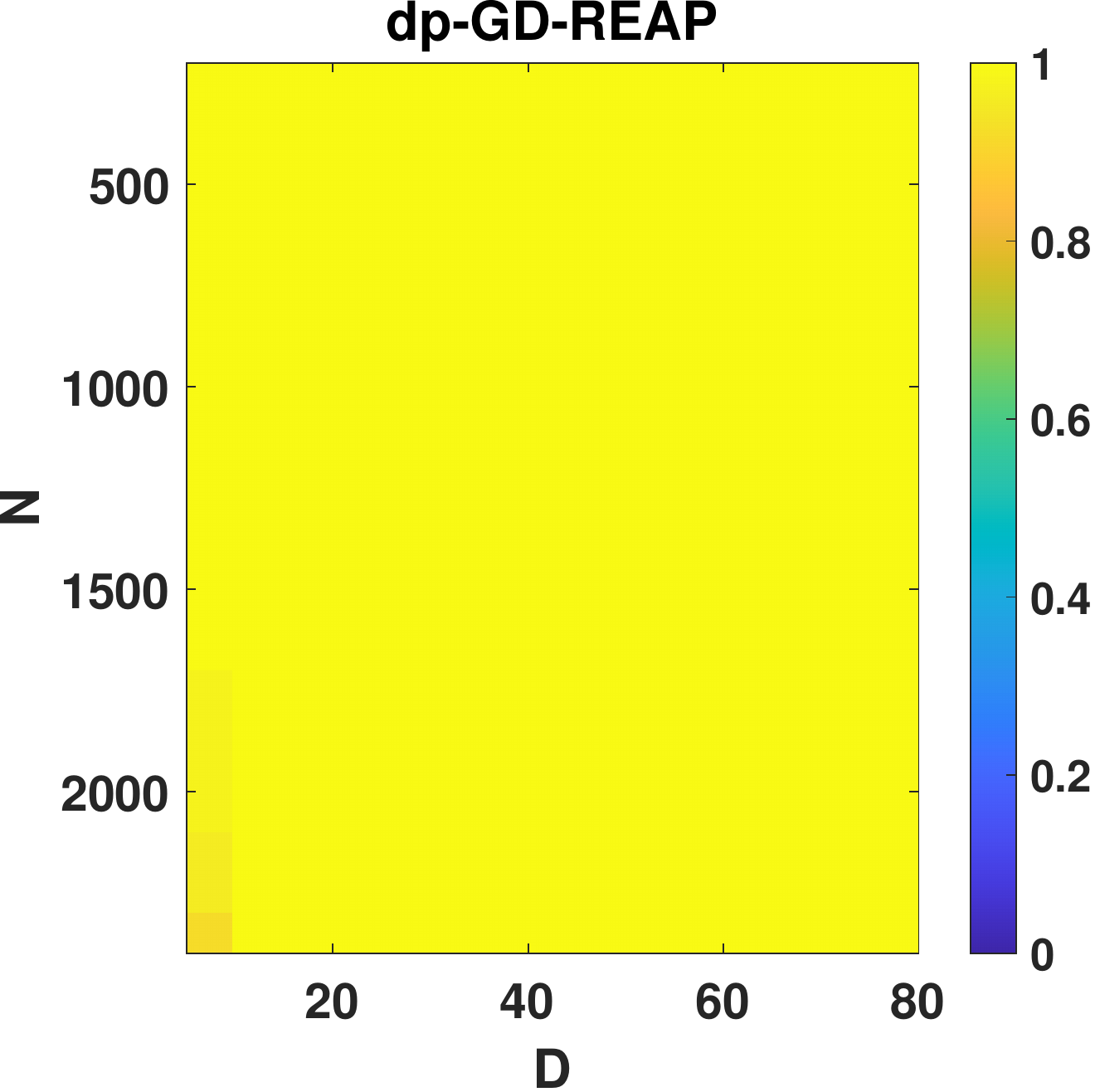}
         \includegraphics[width=.28\textwidth]{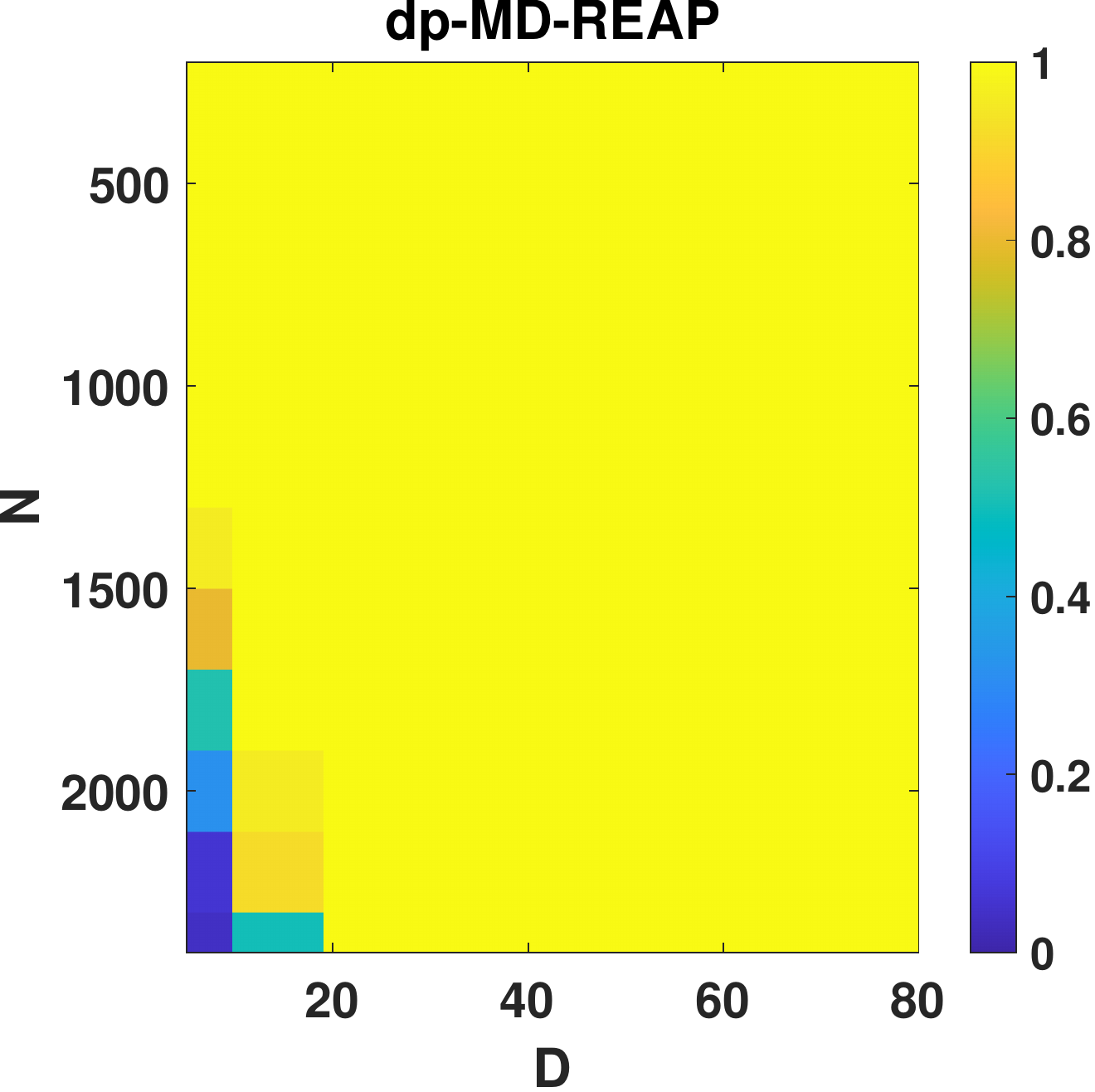}
         \includegraphics[width=.28\textwidth]{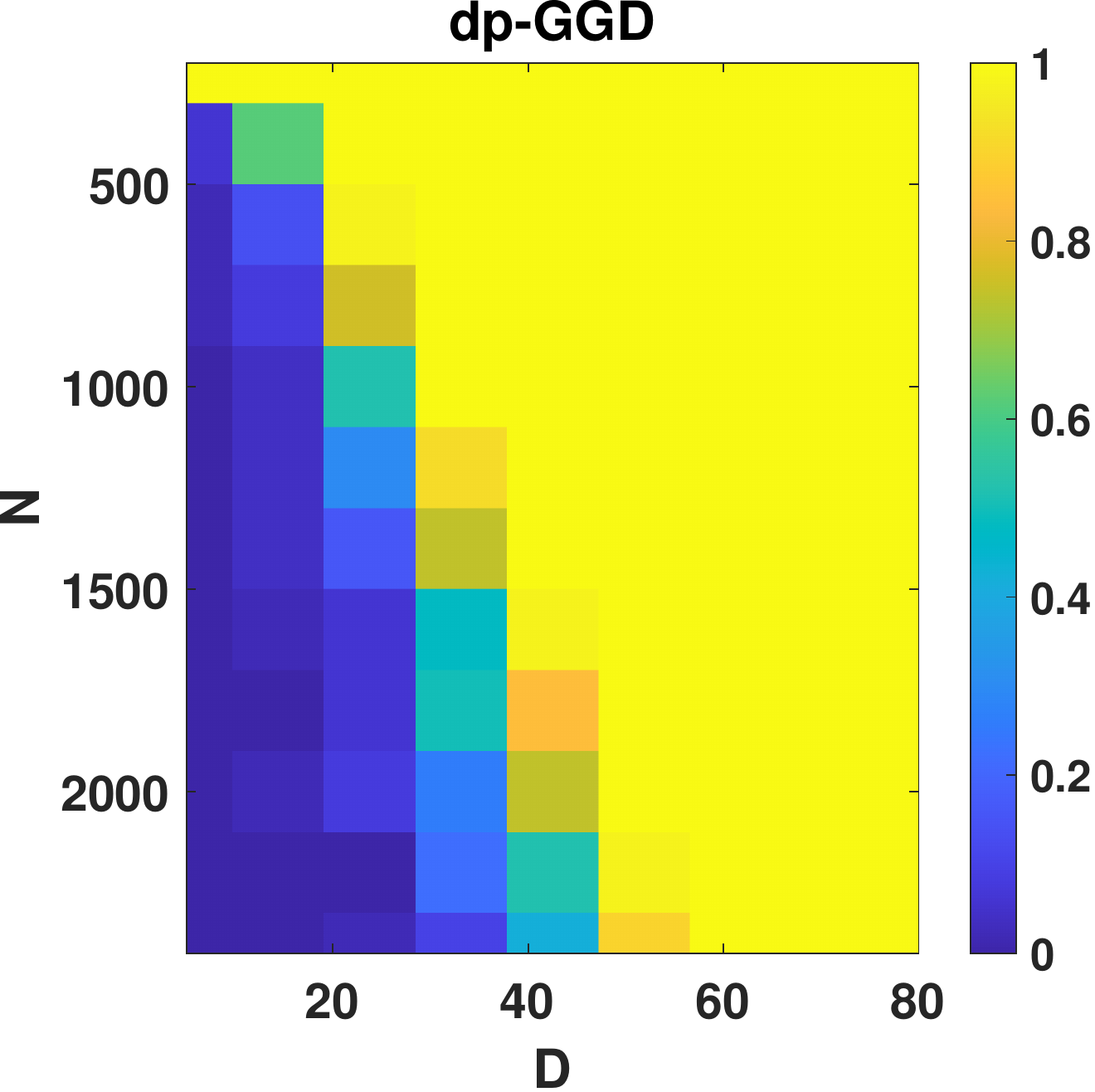}
         \includegraphics[width=.28\textwidth]{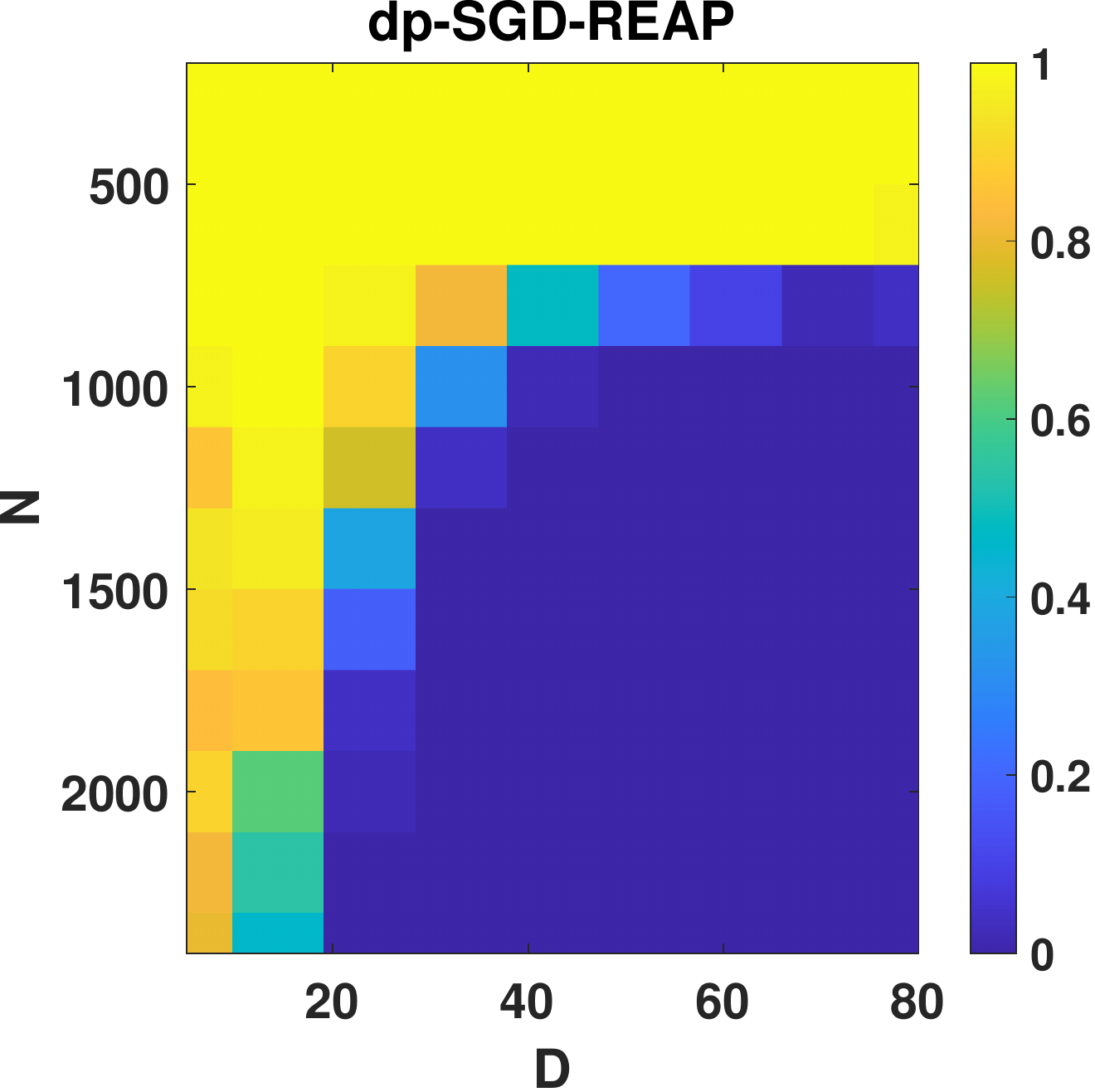}
         \includegraphics[width=.28\textwidth]{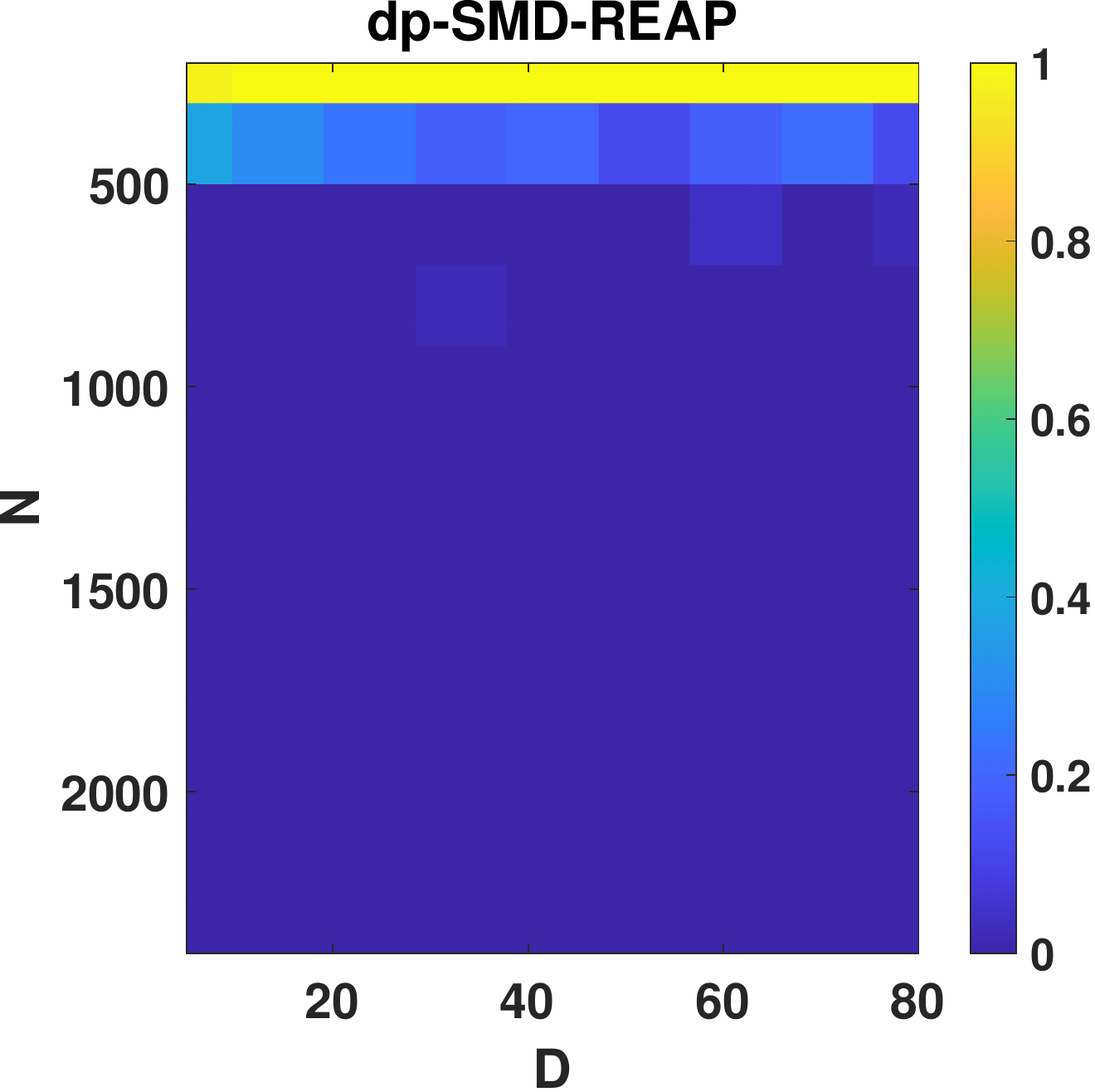}
         \includegraphics[width=.28\textwidth]{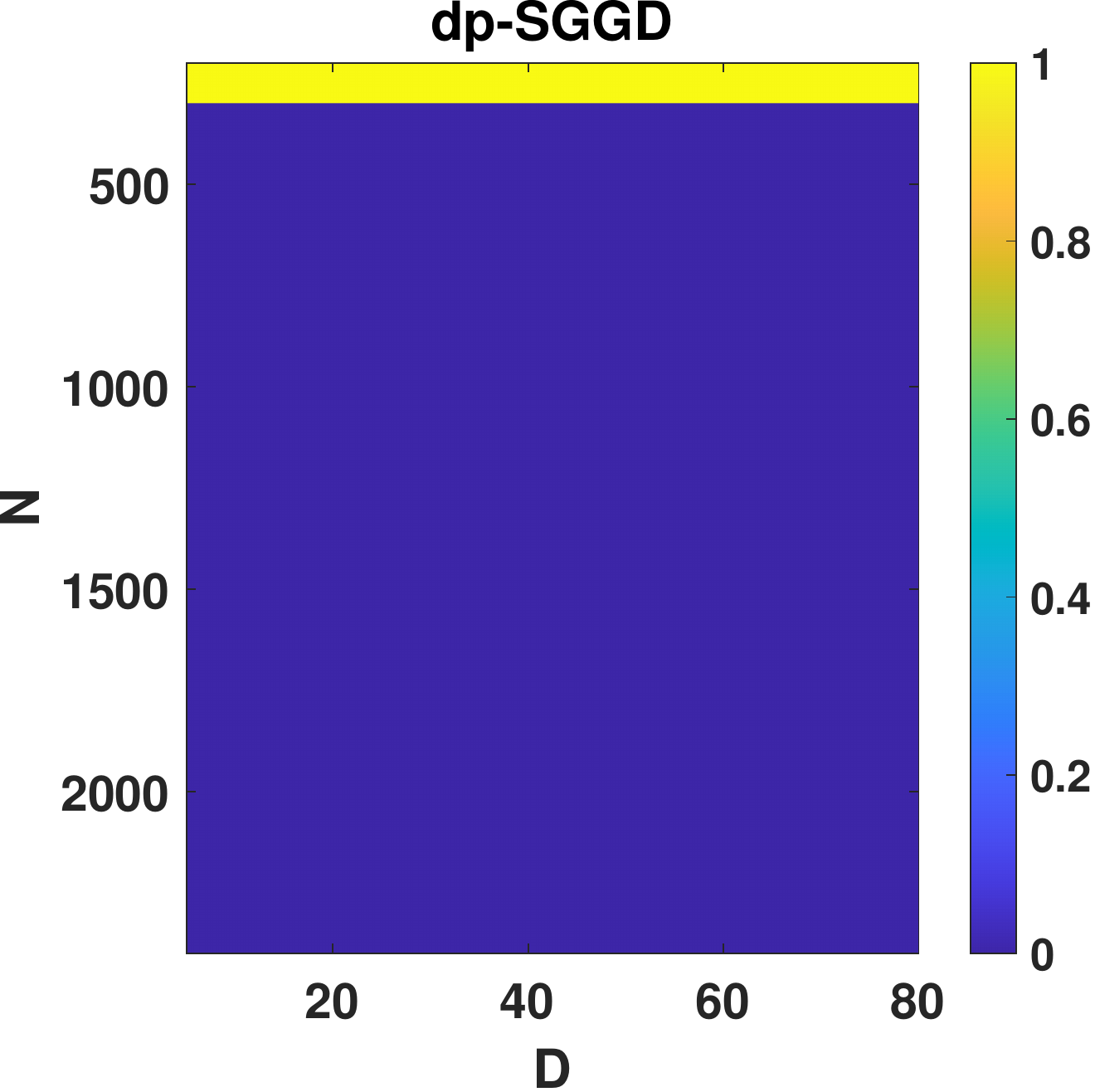}
        \caption{$D$ and $N$ phase transition plot. Each square represents the percentage of repetitions each algorithm fails to reach tolerance $10^{-2}$. }\label{fig:N_D_conv}
\end{figure}



\end{document}